\documentclass{article}

\usepackage[final, nonatbib]{neurips_2022}

\usepackage[utf8]{inputenc} 
\usepackage[T1]{fontenc}    
\usepackage{hyperref}       
\usepackage{url}            
\usepackage{booktabs}       
\usepackage{amsfonts}       
\usepackage{nicefrac}       
\usepackage{microtype}      
\usepackage{xcolor}         

\usepackage[numbers]{natbib}
\usepackage{amsmath}
\usepackage{amsthm}
\newtheorem{theorem}{Theorem}[section]
\newtheorem{definition}{Definition}
\usepackage{bbm}
\usepackage{multirow}
\usepackage{graphicx}
\usepackage{subcaption}
\usepackage{enumitem}  

\usepackage{algorithm}
\usepackage{algorithmic}

\usepackage{tabularx, ragged2e}
\usepackage{makecell}
\usepackage{multibib}
\newcites{Appendix}{Appendix References}

\usepackage[toc,page,header]{appendix}
\usepackage{minitoc}

\def\h2{\hspace*{-2pt}}
\definecolor{comment}{rgb}{0.2,0.4,0.2}
\def\comment#1{\textcolor{comment}{\textit{// #1 }}}

\newcommand{\x}{{\textbf{x}}}
\newcommand{\y}{{\textbf{y}}}
\newcommand{\MS}{MS}
\newcommand{\ourmethod}{{\textrm{TreeFARMS}}}
\def\fix{\textrm{\rm fix}}
\def\splitrm{\textrm{\rm split}}
\def\leftrm{\textrm{\rm left}}
\def\rightrm{\textrm{\rm right}}

\def\h2{\hspace*{-2pt}}
\def\x{\mathbf{x}}
\def\y{\mathbf{y}}

\def\refrm{\textrm{\rm ref}}
\def\set{\textrm{\rm set}}

\title{Exploring the Whole Rashomon Set of Sparse Decision Trees}

\author{%
  Rui Xin\textsuperscript{\rm 1}\thanks{Equal Contribution} \quad Chudi Zhong\textsuperscript{\rm 1}\footnotemark[1] \quad Zhi Chen\textsuperscript{\rm 1}\footnotemark[1] \\
  \And Takuya Takagi\textsuperscript{\rm 2} \quad Margo Seltzer\textsuperscript{\rm 3} \quad Cynthia Rudin\textsuperscript{\rm 1} \\
  \\
  \textsuperscript{1} Duke University \textsuperscript{2} Fujitsu Laboratories Ltd. \textsuperscript{3} The University of British Columbia\\
  \texttt{\{rui.xin926, chudi.zhong, zhi.chen1\}@duke.edu}\\
  \texttt{takagi.takuya@fujitsu.com, mseltzer@cs.ubc.ca, cynthia@cs.duke.edu}
}

\begin{document}

\maketitle

\doparttoc 
\faketableofcontents 

\begin{abstract}
In any given machine learning problem, there might be many models that explain the data almost equally well. However, most learning algorithms return only one of these models, leaving practitioners with no practical way to explore alternative models that might have desirable properties beyond what could be expressed by a loss function. 
The \textit{Rashomon set} is the set of these all almost-optimal models. Rashomon sets can be large in size and complicated in structure, particularly for highly nonlinear function classes that allow complex interaction terms, such as decision trees. We provide the first technique for completely enumerating the Rashomon set for sparse decision trees; in fact, our work provides the first complete enumeration of any Rashomon set for a non-trivial problem with a highly nonlinear discrete function class. This allows the user an unprecedented level of control over model choice among all models that are approximately equally good.
We represent the Rashomon set in a specialized data structure that supports efficient querying and sampling. We show three applications of the Rashomon set: 1) it can be used to study variable importance for the set of almost-optimal trees (as opposed to a single tree), 2) the Rashomon set for accuracy enables enumeration of the Rashomon sets for balanced accuracy and F1-score, and 3) the Rashomon set for a full dataset can be used to produce Rashomon sets constructed with only subsets of the data set.
Thus, we are able to examine Rashomon sets across problems with a new lens, enabling users to choose models rather than be at the mercy of an algorithm that produces only a single model.

\end{abstract}

\section{Introduction}
The \textit{Rashomon set} is the set of almost-equally-optimal models \cite{semenova2019study, fisher2019all}. Rashomon sets were named for the \textit{Rashomon effect} of Leo Breiman, whereby many equally good models could explain the data well \cite{breiman2001statistical}. The Rashomon effect helps us understand that there is not just one ``best'' explanation for the data, but many diverse equally predictive models. The existence of Rashomon sets has important practical implications: practitioners may not want to use the single model from the Rashomon set that was output by a machine learning algorithm. Instead, they may want to explore the Rashomon set to find models with important properties, such as interpretability, fairness, or use of specific variables, or they may want to choose a model that agrees with possible causal hypotheses, monotonicity trends, ease of calculation, or simply domain intuition. In light of viewing the learning problem this way, the whole paradigm that machine learning should provide only one ``optimal'' model makes little sense: perhaps we should turn the \textit{optimization} problem into a \textit{feasibility} problem, and insist that we find all approximately-equally-good models and allow the user to choose among them. Perhaps the tiny sacrifice of a small amount of empirical risk can make the difference between a model that can be used and one that cannot.

While wishing for full Rashomon sets is easy, actually finding them can be extremely difficult for nonlinear function classes. Consider sparse decision trees, which is the focus of this work. Trees have wildly nonlinear relationships between features, where every leaf of depth $d$ represents an interaction between features of order $d$. Even for trees of depth at most 4 with only 10 binary features, the number of possible trees (the size of the hypothesis space) is more than $9.338\times 10^{20}$ models \citep{hu2019optimal}.
This complexity explains why no previous work has been able to provide full Rashomon sets for optimal sparse trees, nor for any non-trivial function class with a large number of interactions.

While the hypothesis space of sparse trees is huge, the Rashomon set of good sparse trees is a small subset of it. In fact, the Rashomon set is often small enough to store, even if the function class itself is too large to be enumerated. Thus, rather than enumerate all sparse trees and store just those in the Rashomon set, we use \textit{analytical bounds to prove that large portions of the search space do not contain any members of the Rashomon set} and can safely be excluded. This allows us to home in on just the portion of the space containing the Rashomon set. We \textit{store the Rashomon set in a specialized data structure}, which permits memory-efficient storage and easy indexing of the Rashomon set's members. 
The combination of these strong bounds and efficient representation enable us to provide \textbf{the first complete enumeration of the Rashomon set for sparse decision trees on real-world datasets}.  Our method is called \ourmethod{}, for ``Trees FAst RashoMon Sets.''

The Rashomon sets we compute give us an unprecedented lens through which to examine learning problems. We can now directly answer questions such as: \textbf{Size:} How large is the Rashomon set? Does its size vary between datasets? \textbf{Variable importance:} How does variable importance change among trees within the Rashomon set? Perhaps this will give us a better sense of how important a variable is to a dataset, rather than just to one model. \textbf{Variability in predictions:} 
Do the models in the Rashomon set all predict similarly on the data? 
\textbf{Robustness:} How does the Rashomon set change if we remove a subset of data?
\textbf{Rashomon sets for other losses:} 
How can we construct Rashomon sets for balanced accuracy and F1-score?
In this work, we show how the answers to these questions can be computed directly after running our algorithm for finding the Rashomon set.

\section{Related Work}\label{sec:related}

\textbf{Rashomon sets}. Leo Breiman \citep{breiman2001statistical} proposed the \textit{Rashomon effect} to describe the fact that many accurate-but-different models exist for the same data. Rashomon sets occur in  healthcare, finance, criminal justice, etc.  \citep{d2020underspecification, marx2020predictive}. They have been used for decision making \citep{tulabandhula2014robust} and robustness of estimation \citep{coker2021theory}. Semenova et al$.$ \citep{semenova2019study} show that when the Rashomon set is large, models with various important properties such as interpretability and fairness can exist inside it. Other works use the Rashomon set to study the range of variable importance among all well-performing models \citep{fisher2019all, dong2020exploring}, which provides a better sense of how important a variable is in general, rather than how important it is to a specific model. 
Other works find diverse near-optimal solutions for integer linear programs \citep{danna2007generating, ahanor2022diversitree}. Kissel et al$.$ \cite{kissel2021forward} find collections of accurate models through model path selection. They collect models using forward selection and use these as a collection of accurate models.  
Our work differs from these works in that we find the \textit{full} Rashomon set of a class of sparse decision trees on real problems. 

\textbf{Model enumeration}.  Ruggieri~\citep{ruggieri2017enumerating,ruggieri2019complete} enumerates trees built from greedy top-down decision tree induction. Our method differs as our trees are not built using greedy induction but by searching the entire binary decision tree space, and our goal is to enumerate all well-performing trees, not simply the features in the greedy trees. Hara et al$.$
\cite{hara2017enumerate,hara2018approximate} enumerate linear models and rule lists or rule sets in descending order
of their objective values. However, they find only one model for any given set of itemsets (association rules). Thus, they enumerate and store only a small subset of the Rashomon set.

\textbf{Decision trees}. Decision tree algorithms have a long history \citep{morgan1963problems, payne1977algorithm, loh2014fifty}, but the vast majority of work on trees has used greedy induction \citep{breiman1984classification, quinlan1993c} to avoid solving the NP-complete problem of finding an optimal tree \citep{laurent1976constructing}. However, greedy tree induction provides suboptimal trees, which has propelled research since the 1990s on mathematical optimization for finding optimal decision trees \citep{hu2019optimal,bennett1996optimal,farhangfar2008fast,bertsimas2017optimal,verwer2019learning,vilas2021optimal,gunluk2021optimal,aghaei2021strong,narodytska2018learning}, as well as dynamic programming with branch-and-bound  \citep{lin2020generalized,mctavish2021smart,aglin2020learning,demirovic2020murtree}. We refer readers to two recent reviews of this area \cite{carrizosa2021mathematical,rudin2022interpretable}. 

\textbf{Bayesian trees.} Bayesian analysis has long aimed to produce multiple almost-optimal models through sampling from the posterior distribution. While it might reasonably seem like a method such as BART \cite{bart_package} or random forest \cite{breiman2001random} might sample effectively from the Rashomon set, as our experiments show (Figure \ref{fig:comp_baselines}), they do not.

\section{Bounds for Reducing the Search Space}
\label{sec:notation}

We denote the training dataset as $\{(\x_i, \y_i)\}_{i=1}^n$, where $\x_i \in \{0,1\}^p$ are binary features. Our notation uses $y_i \in \{0,1\}$ as labels, while our method can be generalized to multiclass classification as well. Let $\ell(t, \x, \y) = \frac{1}{n}\sum_{i=1}^n\mathbf{1}[\hat{y}_i \neq y_i]$ be the loss of tree $t$ on the training set, where $\{\hat{y}_i\}_{i=1}^n$ are predicted labels given $t$. We define the objective function as the combination of the misclassification loss and a sparsity penalty on the number of leaves: $Obj(t, \x, \y) = \ell(t, \x, \y) + \lambda H_t$, where $H_t$ is the number of leaves in tree $t$ and $\lambda$ is a regularization parameter. 

\begin{definition}($\epsilon$-Rashomon set)\label{def:eps-rashomon}
Let $t_{\refrm}$ be a benchmark model or reference model from $\mathcal{T}$, where $\mathcal{T}$ is a set of binary decision trees. The $\epsilon$-Rashomon set is a set of all trees $t\in \mathcal{T}$ with $Obj(t,\x,\y)$ at most $(1+\epsilon)\times Obj(t_{\refrm},\x,\y)$: 
\begin{equation}\label{eq:rset_definition}
    \hspace*{-3pt}
    R_{\textrm{set}}(\epsilon, t_{\refrm}, \mathcal{T})\!:=\!\{t \!\in\! \mathcal{T}\!:\! Obj(t,\x,\y)\!\leq\! (1+\epsilon)\!\times\! Obj(t_{\refrm},\x,\y) \}.
\end{equation}
\end{definition}
For example, if we permit models within 2\% of the reference objective value, we would set $\epsilon$ to be 0.02. 
Note that we use $R_{\textrm{set}}(\epsilon)$ to represent $R_{\textrm{set}}(\epsilon, t_{\refrm}, \mathcal{T})$ when $t_{\refrm}$ and $\mathcal{T}$ are clearly defined. We use $\theta_{\epsilon}:=(1+\epsilon)\times Obj(t_{\refrm},\x,\y)$ to denote the threshold of the Rashomon set. Typically, the reference model is an empirical risk minimizer $t_{\textrm{ref}}\in 
\arg\min_{t\in\textrm{trees}} Obj(t,\x,\y)$. Recent advances in decision tree optimization have allowed us to find this empirical risk minimizer, specifically using the GOSDT algorithm \citep{lin2020generalized,mctavish2021smart}.
Our goal is to store $R_{set}(\epsilon, t_{\textrm{ref}}, \mathcal{T})$, sample from it, and compute statistics from it.

Customized analytical bounds, leveraging tools from \cite{lin2020generalized,hu2019optimal},
help reduce the search space for Rashomon set construction. As every possible tree is being grown in the process of dynamic programming, some of its leaves will have been determined (``fixed''), and others will not have yet been determined (``unfixed''). Bounds for these incomplete trees compare two quantities: the performance that might be achieved in the best possible case, when every unfixed part of the tree has perfect classification for all of its points, and  $\theta_{\epsilon}$, the threshold of the Rashomon set. If the first is larger -- i.e., worse -- than $\theta_{\epsilon}$, we know that extensions of the tree will never be in the Rashomon set.
Each partial tree $t$ (that is, a tree that can be extended) is represented in terms of five variables: $t_{\fix}$ are the ``fixed'' leaves we do not extend during this part of the search (there will be a different copy of the tree elsewhere in the search where these leaves will potentially be split), $\delta_{\fix}$ are the labels of the data points within the fixed leaves, $t_{\splitrm}$ are the ``unfixed'' leaves we could potentially split during the exploration of this part of the search space (with labels $\delta_{\splitrm}$) and $H_t$ is the number of leaves in the current tree. Thus, our current position in the search space is a partial tree $t = (t_{\fix},\delta_{\fix}, t_{\splitrm}, \delta_{\splitrm}, H_t)$.
The theorems below allow us to exclude large portions of the search space.
\begin{theorem}\label{corollary:lb} (Basic Rashomon Lower Bound)
Let $\theta_{\epsilon}$ be the threshold of the Rashomon set. Given a tree $t = (t_{\fix},\delta_{\fix}, t_{\splitrm}, \delta_{\splitrm}, H_t)$, let $b(t_{\fix}, \x, \y):=\ell(t_{\fix}, \x, 
\y)+\lambda H_t$ denote the lower bound of the objective for tree $t$. If $b(t_{\fix}, \x, \y) > \theta_{\epsilon}$, then the tree $t$ and all of its children are not in the $\epsilon$-Rashomon set.
\end{theorem}
We can tighten the basic Rashomon lower bound by using knowledge of \textit{equivalent points} \cite{angelino2018learning}. Data points are equivalent if they have exactly the same feature values. Let $\Omega$ be a set of leaves. \textit{Capture} is an indicator function that equals 1 if $\x_i$ falls into one of the leaves in $\Omega$, and 0 otherwise, in which case we say that $\textrm{cap}(\x_i, \Omega) = 1$. Let $e_u$ be a set of equivalent points and $q_u$ be the minority class label among points in $e_u$. A dataset consists of multiple sets of equivalent points. Let $\{e_u\}_{u=1}^U$ enumerate these sets. The bound below incorporates equivalent points. 
\begin{theorem}\label{thm:equiv}
(Rashomon Equivalent Points Bound) Let $\theta_{\epsilon}$ be the threshold of the Rashomon set. Let $t$ be a tree with leaves $t_{\fix}, t_{\splitrm}$ and lower bound $b(t_{\fix}, \x, \y)$. Let $b_{\textrm{equiv}}(t_{\splitrm}, \x, \y):=\frac{1}{n}\sum_{i=1}^n \sum_{u=1}^U \textrm{cap}(\x_i, t_{\splitrm}) \wedge \mathbbm{1}[\x_i \in e_u] \wedge \mathbbm{1}[y_i=q_u]$ be the lower bound on the misclassification loss of the unfixed leaves. Let $B(t,\x, \y):= b(t_{\fix}, \x, \y) + b_{\textrm{equiv}}(t_{\splitrm}, \x, \y)$ be the Rashomon lower bound of $t$. If $B(t,\x, \y) > \theta_{\epsilon}$, tree $t$ and all its children are not in the $\epsilon$-Rashomon set.
\end{theorem}
We can use this bound recursively on all subtrees we discover during the process of dynamic programming. If, at any time, we find the sum of the lower bounds of subproblems created by a split exceed the threshold of the Rashomon set, the split that led to these subproblems will never produce any member of the Rashomon set. This is formalized in the Rashomon Equivalent Points Bound for Subtrees,
Theorem \ref{corollary:lb_dev} in Appendix \ref{app:theorems}, which dramatically helps reduce the search space. We also use the ``lookahead'' bound, used in GOSDT \citep{lin2020generalized}, which looks one split forward and cuts all the children of tree $t$ if $B(t,\x, \y) + \lambda > \theta_{\epsilon}$.


\section{Storing, Extracting, and Sampling the Rashomon Set}
\label{sec:new_data_structure}

The key to \ourmethod's scalability is a novel \emph{Model Set} representation.
The \textbf{Model Set (MS)} is a set of hierarchical maps; each map is a \textbf{Model Set Instance (MSI)}.
Conceptually, we identify a MSI by a <subproblem, objective> pair; in reality, we use pointers to improve execution time and reduce memory consumption.
%
A MSI can represent a terminal (leaf) node, an internal node, or both (See Appendix \ref{app:datastructure_detail}).
A leaf MSI stores only the subproblem's prediction and the number of false positives and negatives (or false predictions of each class in the case of multiclass classification).
An internal MSI, $M$, is a map whose keys are the features on which to split the subproblem and whose values are an array of pairs, each referring to left and right MSIs whose objectives sum to the objective of $M$.
\ourmethod{}'s efficiency stems from
the fact that the loss function for decision trees takes on a discrete number of values (approximately equal to the number of samples in the training data set), while the number of trees in the Rashomon set is frequently orders of magnitude larger. Therefore, many trees (and subtrees) have the same objective. By grouping together trees with the same objective, we avoid massive amounts of data duplication and computation. See Appendix~\ref{app:datastructure_detail} for an example. 
Equipped with these data structures, we now present our main algorithm. 
%

\subsection{\ourmethod{} Implementation}\label{sec:gosdt}
We implement \ourmethod{} in GOSDT \citep{lin2020generalized}, which uses a dynamic-programming-with-bounds formulation to find the optimal sparse decision trees.
Each \emph{subproblem}
is defined by a support set $s_a\in \{0, 1\}^n$ such that the $i^{th}$ element is 1 iff $\x_i$ satisfies the Boolean assertion $a$ that corresponds to a decision path in the tree.
For each subproblem in the dynamic program, GOSDT keeps track of upper and lower bounds on its objective.
It stores these subproblems and their bounds in a \emph{dependency graph}, which expresses the relationships between subproblems.
\ourmethod{} transforms GOSDT in two key ways.
First, while searching the space, \ourmethod{} prunes the search space by removing only those subproblems whose objective lower bound is greater than the thresholds defined by the Rashomon set bound, $\theta_\epsilon$, rather than GOSDT's objective-based upper bound.
Second, rather than finding the \textit{single} best model expressed by the dependency graph, \ourmethod{} returns \textit{all models in the Rashomon set defined by $\theta_\epsilon$}.

\subsection{Extraction Algorithm}\label{sec:newalgorithm}
\ourmethod{} (Alg. \ref{alg:Rashomon}) constructs the dependency graph using the bounds from Section \ref{sec:notation}, and Extract (Alg. \ref{alg:extract_short}) extracts the Rashomon set from the dependency graph.

\noindent \textbf{\ourmethod{} (Algorithm \ref{alg:Rashomon})}: \textbf{Line 1}: Call GOSDT to find the best objective.
\textbf{Line 2}: Using the best objective from Line 1, compute $\theta_\epsilon$, as defined in Definition \ref{def:eps-rashomon}. \textbf{Lines 3-6}: Configure and execute (the modified) GOSDT to produce a dependency graph containing all subproblems in the Rashomon set.
\textbf{Lines 7-10}: Initialize the parameters needed by $\mathbf{extract}$ and then call it to extract the Rashomon set from the dependency graph.

\noindent \textbf{Extract (Algorithm \ref{alg:extract_short})}: We present an abbreviated version of the algorithm here with the full details in Appendix \ref{app:algorithm}. 
\textbf{Line 1}: Check to see if we already have the Rashomon set for the given problem and scope; if so, return immediately.
\textbf{Lines 2-3}: If we can make a prediction for the given subproblem that produces loss less than or equal to scope (using Theorems \ref{corollary:lb} and \ref{thm:equiv}), then the leaf for this subproblem should be part of trees in the Rashomon set, so we add it to our Model Set.
\textbf{Lines 4-12}: Loop over each feature and consider splitting the current problem on that feature.
\textbf{Lines 5-6}: Skip over any splits that either do not appear in the dependency graph or whose objectives produce a value greater than scope (using Theorem \ref{corollary:lb_dev}).
\textbf{Lines 7-9}: Find all subtrees for left and right that should appear in trees in the Rashomon set, and construct the set of MSI identifiers for each.
\textbf{Lines 10-12}: Now, take the cross product of the sets of MSI identifiers.
For each pair, determine if the sum of the objectives for those MSI are within scope (using Theorem \ref{corollary:lb_dev}).
If so, we add the left/right pair to the appropriate MSI, creating a new MSI if necessary.
When this loop terminates, all trees in the Rashomon set are represented in MS.

\subsection{Sampling from the Rashomon Set}
\label{sec:sampling}
If we can store the entire Rashomon set in memory, then sampling is unnecessary.
However, sometimes the set is too large to fit in memory (e.g., the COMPAS data set~\cite{LarsonMaKiAn16} with a regularization of 0.005 and a Rashomon threshold that is within 15\% of optimal produces $10^{12}$ trees).
Our Model Set representation permits easy uniform sampling of the Rashomon set that can be used to explore the set with a much lower computational burden. Appendix~\ref{app:sample} presents a sampling algorithm.

\begin{algorithm}[ht]
\caption{$\ourmethod(\mathbf{x}, \mathbf{y},\lambda, ~\epsilon)~\rightarrow~R_{\textrm{set}}$} 
\label{alg:Rashomon}
\comment{Given a dataset $(\x,\y)$, $\lambda$, and $\epsilon$, return the set, $R_{\textrm{set}}$, of all trees whose objective is in $\theta_\epsilon$. }\\
1: $opt\leftarrow\textit{gosdt.optimal$\_$obj}(\mathbf{x}, \mathbf{y},\lambda)$ \comment{Use GOSDT to find $opt$,  the objective of the optimal tree.}\\
2: $\theta_\epsilon\leftarrow opt~*~(1+\epsilon)$ \comment{Compute $\theta_\epsilon$, which is the threshold of the Rashomon Set. 
}\\
3: \textit{gosdt.best$\_$current$\_$obj} $\leftarrow\theta_\epsilon$ \comment{Set and fix GOSDT's best current objective $R^c$ to $\theta_\epsilon$}\\
\comment{Disable leaf accuracy bound and incremental accuracy bound in GOSDT (see Algorithm 7 in \cite{lin2020generalized}), which are used to find optimal trees but could remove near-optimal trees in the Rashomon set} \\
4: $\textit{gosdt.leaf\_accuracy}\leftarrow False$ \comment{Disable leaf accuracy bound.}\\
5: $\textit{gosdt.incremental\_accuracy}\leftarrow False$ \comment{Disable incremental accuracy bound.}\\
6: $\textit{gosdt.fit}(\mathbf{x}, \mathbf{y},\lambda)$ \comment{Run the branch-and-bound algorithm using new settings of bounds.}\\
7: $G\leftarrow\textit{gosdt.get\_graph}()$ \comment{Return dependency graph obtained from branch-and-bound.}\\
8: global $MS~\leftarrow~\{\}$
\comment{Initialize Rashomon Model Set $MS$, which is a global data structure.}\\
9: $P~\leftarrow~ones({|\mathbf{y}|})$ \comment{ Create the subproblem representation for the entire dataset.} \\
10: extract$(G, P, \theta_\epsilon)$ \comment{Fill in $MS$ with trees in the Rashomon set.}\\ 
11: \textbf{return} $MS_P$
\end{algorithm}

\begin{algorithm}[ht]

\caption{$\textrm{extract}(G, sub, scope)$ (Detailed algorithm in Appendix ~\ref{app:algorithm})} \label{alg:extract_short}
\comment{Given a dependency graph, $G$; a subproblem, $sub$; and a maximum allowed objective value, $scope$, populate the global variable $MS$ with the Rashomon set for sub within scope.}\\
\comment{Check if we have already solved the subproblem $sub$. SOLVED is presented in Alg.~\ref{alg:newextract}.} \\
1: \textbf{if} SOLVED$(MS, sub, scope)$ \textbf{then} \textbf{return} \\
2: \textbf{if} $G[sub] \le scope$ \textbf{then} \comment{Check if we should create a leaf. (Theorems \ref{corollary:lb} and \ref{thm:equiv}).}\\
3:$\>\>\>\>\MS~\leftarrow~\MS~\cup~\textit{newLeaf}(sub)$ \comment{\textit{newLeaf} is presented in Alg.~\ref{alg:newextract}.}\\
\comment{Consider splits on each feasible split feature skipping those not in G or with bounds too large.}\\
4: \textbf{for} each feature $j~\in~[1,M]$ \textbf{do}\\
5: $\>\>\>\>sub_l,sub_r~\leftarrow~split(sub,j)$\\
6: $\>\>\>\>$\textbf{if} either $sub_l, sub_r$ \textbf{not in} $G$ \textbf{or} $G[sub_l]+G[sub_r]>scope$ \textbf{then continue}

$\>\>\>\>\>\>\>\>\>$\comment{Find Model Sets Instances for left and right.}\\
7: $\>\>\>\>\textit{extract}(G, sub_l, scope-G[sub_r])$ \\
8: $\>\>\>\>\textit{extract}(G, sub_r, scope-G[sub_l])$ \\
9: $\>\>\>\>\textit{left}, \textit{right}~\leftarrow~\MS_{sub_l}, \MS_{sub_r}$

10: $\>\>\>$\textbf{for} each $(m_l, m_r)\in (\textit{left}\times \textit{right})$ \textbf{do} \comment{Consider cross product of left/right MSI.}

$\>\>\>\>\>\>\>\>\>\>\>\>\>\>$\comment{Skip trees with objective outside of scope. (Theorem \ref{corollary:lb_dev}).} \\
11: $\>\>\>\>\>\>\>$\textbf{if} $obj(m_l) + obj(m_r) > scope$ \textbf{then continue} \\
12: $\>\>\>\>\>\>\>\MS~\leftarrow~\MS~\cup~ add(sub, obj(m_l) + obj(m_r), m_l, m_r)$ \comment{Add pair to Model Set.}

\comment{$\MS$ now contains all MSI for sub with objective less than or equal to scope.} \\ 
13: \textbf{return}
\end{algorithm}


\section{Applications of the Rashomon Set}
Besides allowing users an unprecedented level of control over model choice, having access to the Rashomon set unlocks powerful new capabilities. We present three example applications here.

\subsection{Variable Importance for Models in the Rashomon Set via Model Class Reliance}\label{sec:MCR}
The problem with classical variable importance techniques is that they generally provide the importance of one variable to one model. However, just because a variable is important to one model does not mean that it is important in general. To answer this more general question, we consider model class reliance (MCR) \citep{fisher2019all}. MCR provides the range of variable importance values across the set of all well-performing models. $\textrm{MCR}_-$ and $\textrm{MCR}_+$ denote the lower and upper bounds of this range, respectively. A feature with a large $\textrm{MCR}_-$ is important in all well-performing models; a feature with a small $\textrm{MCR}_+$ is unimportant to every well-performing model. Past work managed to calculate $\textrm{MCR}_-$ only for convex loss in linear models \citep{fisher2019all}. For decision trees, the problem is nonconvex and intractable using previous methods. A recent work \citep{smith2020model} estimates MCR for random forest, leveraging the fact that the forest's trees are grown greedily from the top down; the optimized sparse trees we consider are entirely different, and as we will show, many good sparse trees are difficult to obtain through sampling. 
Since \ourmethod{} can enumerate the whole Rashomon set of decision trees, we can \textit{directly} calculate variable importance for every tree in the Rashomon set and then find its minimum and maximum to compute the MCR, as described in Appendix \ref{app:mcr}. See Section \ref{sec:exp_mcr} for the results. If the Rashomon set is too large to enumerate, sampling (Section \ref{sec:sampling}) can be used to obtain sample estimates for the MCR (shown in Section \ref{sec:exp_mcr_sampling}). 

\subsection{Rashomon Sets Beyond Accuracy: Constructing the Rashomon Set for Other Metrics}\label{sec:other_metrics}
For imbalanced datasets, high accuracy is not always meaningful. Metrics such as balanced accuracy and F1-score are better suited for these datasets. We show that, given the Rashomon set constructed using accuracy, we can directly find the Rashomon sets for balanced accuracy and F1-score.   

Let $q^+$ be the proportion of positive samples and $q^-$ be the proportion of negative samples, i.e., $q^+ + q^-=1$. Let $q_{\min} := \min(q^+, q^-)$ and $q_{\max} := \max(q^+, q^-)$. Let $FPR$ and $FNR$ be the false positive and false negative rates. 
We define the Accuracy Rashomon set as $A_{\theta}:=\{t\in \mathcal{T}: q^- FPR_t + q^+ FNR_t+\lambda H_t \leq \theta\}$, where $\theta$ is the objective threshold of the Accuracy Rashomon set, 
similar to $\theta_\epsilon$ in Section \ref{sec:notation}. 
The next two theorems guide us to find the Balanced Accuracy or F1-Score Rashomon set from the Accuracy Rashomon set. 
Their proofs are presented in Appendix \ref{app:theorems}.
We use $\delta$ in these theorems to denote the objective threshold of Balanced Accuracy or F1-Score Rashomon sets.

\begin{theorem}\label{thm:bacc}(Accuracy Rashomon set covers Balanced Accuracy Rashomon set) Let 
$B_{\delta}:=\{t \in \mathcal{T}: \frac{FPR_t + FNR_t}{2} + \lambda H_t \leq \delta\}$ be the Balanced Accuracy Rashomon set.
If \[\theta \geq \min\Big(2q_{\max}\delta, q_{\max} + (2\delta-1)q_{\min} + (1- 2q_{\min}) \lambda 2^d\Big),\] where $d$ is the depth limit, then $\forall t \in B_{\delta}, t\in A_{\theta}$.
\end{theorem}

\begin{theorem}\label{thm:f1-score}(Accuracy Rashomon set covers F1-Score Rashomon set) Let 
\[F_{\delta}:=\left\{t \in \mathcal{T}: \frac{q^- FPR_t + q^+ FNR_t}{2q^+ + q^- FPR_t - q^+ FNR_t} + \lambda H_t \leq \delta\right\}\] be the F1-score Rashomon set. Suppose $q^+ \in (0,1)$, $q^- \in (0,1)$, and $\delta - \lambda H_t \in (0,1)$.
If $\theta \geq \min\Big(\max\big(\frac{2q^+\delta}{1-\delta}, \frac{2q^+(\delta - \lambda 2^d)}{1-(\delta - \lambda 2^d)} + \lambda 2^d\big),  \mathbbm{1}[\delta < \sqrt{2}-1]\frac{2\delta}{1+\delta} + \mathbbm{1}[\delta \geq \sqrt{2}-1](\delta + 3 -2\sqrt{2})\Big)$, then $\forall t \in F_{\delta}$, $t\in A_{\theta}$. 
\end{theorem}

We can use Theorems \ref{thm:bacc} and \ref{thm:f1-score} above to find all trees in the Balanced Accuracy or F1-score Rashomon set directly by searching through the Accuracy Rashomon set with objective threshold $\theta$ that satisfies the inequality constraint. In our implementation, we first use GOSDT to find the optimal tree w.r.t$.$ misclassification loss and then calculate its objective w.r.t$.$ balanced accuracy and F1-score. Then we set $\delta$ and the corresponding $\theta$. This guarantees that the Rashomon sets for balanced accuracy and F1-score objective are not empty. 
Experiments in Section \ref{sec:exp:other_metrics} illustrate this calculation.

\subsection{Sensitivity to Missing Groups of Samples}\label{sec:sensitivity}
Though sparse decision trees are usually robust (since predictions are made separately in each leaf), we are also interested in how a sample or a group of samples influences all well-performing models (i.e., whether this subset of points is \textit{influential} \cite{koh2017understanding}.) Influence functions cannot be calculated for decision trees since they require differentiability. 
The following two theorems help us find optimal or near-optimal trees for a dataset in which a group of instances has been removed by searching through the Rashomon set obtained from the full dataset.
Here, we consider the misclassification loss.

Let $\tilde{t}^*$ be the optimal tree trained on  the reduced dataset $\{\x_{[\backslash K, \cdot]}, \y_{[\backslash K]}\}$ where $K$ is a set of indices of instances that we remove from the original dataset. Let $|K|$ denote the cardinality of the set $K$. 
Overloading notation to include the dataset, let $ R_{set}(\epsilon, t^*, \mathcal{T},\x,\y) = R_{set}(\epsilon, t^*, \mathcal{T})$ (see Eq \eqref{def:eps-rashomon}) be the Rashomon set of the original dataset, 
where $t^*$ is the optimal tree trained on the original dataset. We define the $\epsilon'$-Rashomon set on the reduced dataset as $$R_{set}(\epsilon', \tilde{t}^*,\mathcal{T}, \x_{[\backslash K, \cdot]}, \y_{[\backslash K]})\h2:=\h2\left\{t \in \mathcal{T}: Obj(t, \x_{[\backslash K, \cdot]}, \y_{[\backslash K]}) \h2\leq\h2 (1+\epsilon')\h2\times\h2 Obj(\tilde{t}^*, \x_{[\backslash K, \cdot]}, \y_{[\backslash K]})\right\}.$$

\begin{theorem}\label{thm:opt_tree_rm_instance}(Optimal tree after removing a group of instances is still in full-dataset Rashomon set)
If $\epsilon \geq \frac{2|K|}{n\times Obj(t^*,\x,\y)}$, then $\tilde{t}^* \in R_{set}(\epsilon, t^*, \mathcal{T},\x,\y)$.  
\end{theorem}

Now we consider not only the optimal tree on the reduced dataset but also the near-optimal trees. 

\begin{theorem}\label{thm:rset_rm_instance}(Rashomon set after removing a group of instances is within full-dataset Rashomon set)
If $\epsilon \geq \epsilon' + \frac{(2+\epsilon')|K|}{n \times Obj(t^*, \x, \y)}$, then $\forall t \in R_{set}(\epsilon', \tilde{t}^*, \mathcal{T},\x_{[\backslash K, \cdot]}, \y_{[\backslash K]}), t\in R_{set}(\epsilon, t^*, \mathcal{T}, \x, \y)$.  
\end{theorem}

\section{Experiments}



Our evaluation answers the following questions:
1. How does \ourmethod{} compare to baseline methods for searching the hypothesis space? (\S\ref{sec:compare_baselines}),
2. How quickly can we find the entire Rashomon set? (\S\ref{sec:compare_baselines}),
3. What does the Rashomon set look like?  What can we learn about its structure? (\S\ref{sec:rset_visual}),
4. What does MCR look like for real datasets? (\S\ref{sec:exp_mcr}),
5. How do balanced accuracy and F1-score Rashomon sets compare to the accuracy Rashomon set?  (\S\ref{sec:exp:other_metrics}), and
6. How does removing samples affect the Rashomon set? (\S\ref{sec:exp:influence}).

Finding the Rashomon set is computationally difficult due to searching an exponentially growing search space. 
We use datasets from the UCI Machine Learning Repository \citep[Car Evaluation, Congressional Voting Records, Monk2, and Iris, see][]{Dua:2019}, a penguin dataset \citep{gorman2014ecological}, a criminal recidivism dataset \citep[COMPAS, shared by][]{LarsonMaKiAn16}, the Fair Isaac (FICO) credit risk dataset \citep{competition} used for the Explainable ML Challenge, and four coupon datasets (Bar, Coffee House, Cheap Restaurant, and Expensive Restaurant) ~\cite{wang2017bayesian} that come from surveys. More details are in Appendix \ref{app:datasets}.

\subsection{Performance and Timing Experiments}\label{sec:compare_baselines}
Our method directly constructs the Rashomon set of decision trees of a dataset.
While, to the best of our knowledge, there is no previous directly comparable work,
there are several methods one might naturally consider to find this set. 
One might use methods that sample from the high-posterior region of tree models,
though we would not know how many samples we need to extract the full Rashomon set.
Thus, the first baseline method we consider is sampling trees from the posterior distribution of Bayesian Additive trees \citep{chipman1998bayesian, chipman2010bart}. We used the R package BART \citep{bart_package}, setting the number of trees in each iteration to 1. 
Many ensemble methods combine a diverse set of trees. The diversity in this set comes from fitting on different subsets of data. Trees produced by these methods would be natural approaches for finding the Rashomon set. We thus generated trees from three different methods (Random Forest \citep{breiman2001random}, CART \citep{breiman1984classification}, and GOSDT \citep{lin2020generalized}), on many subsets of our original data. 

\begin{figure*}
    \centering
    \includegraphics[height=0.25\textwidth]{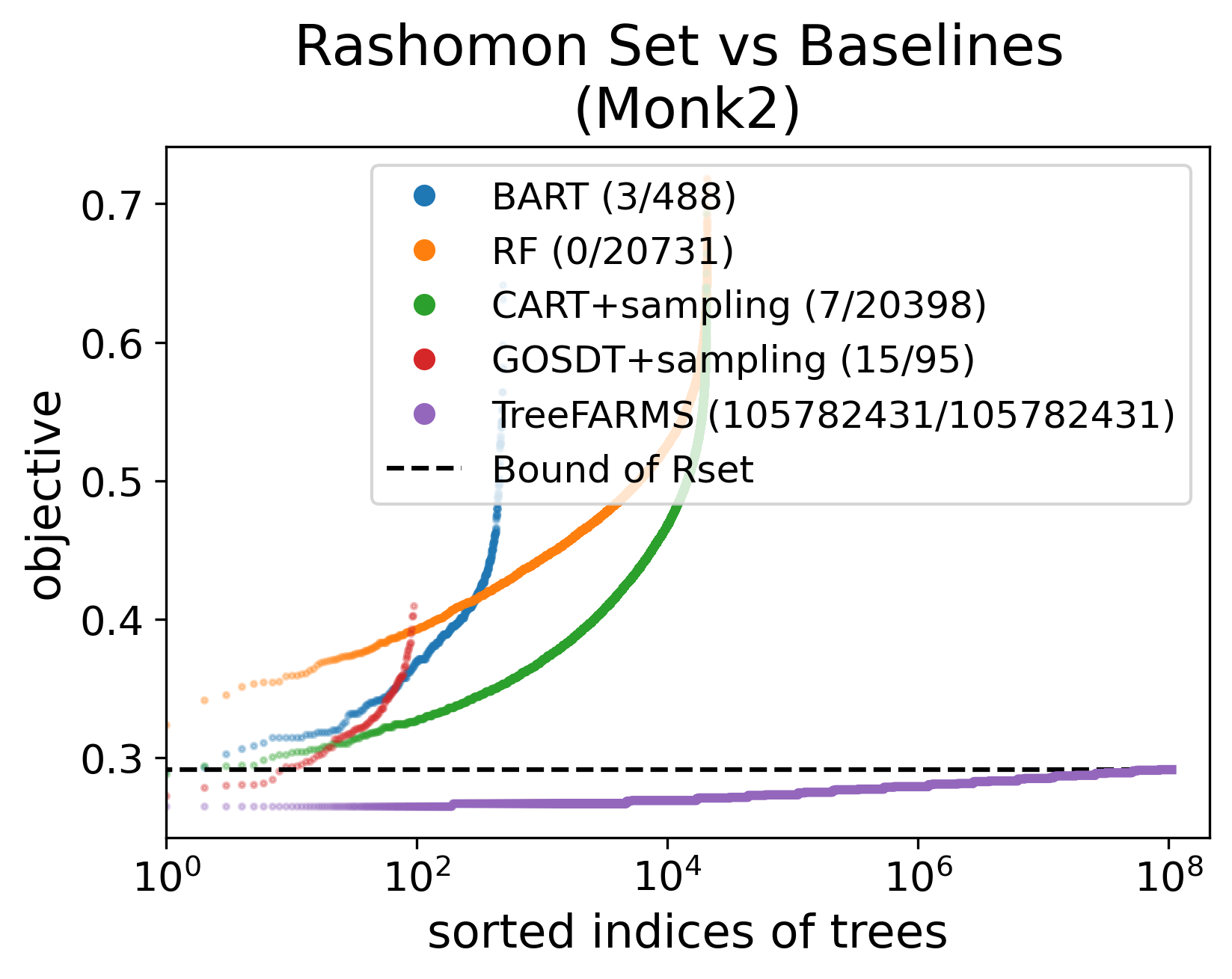}
     \includegraphics[height=0.25\textwidth]{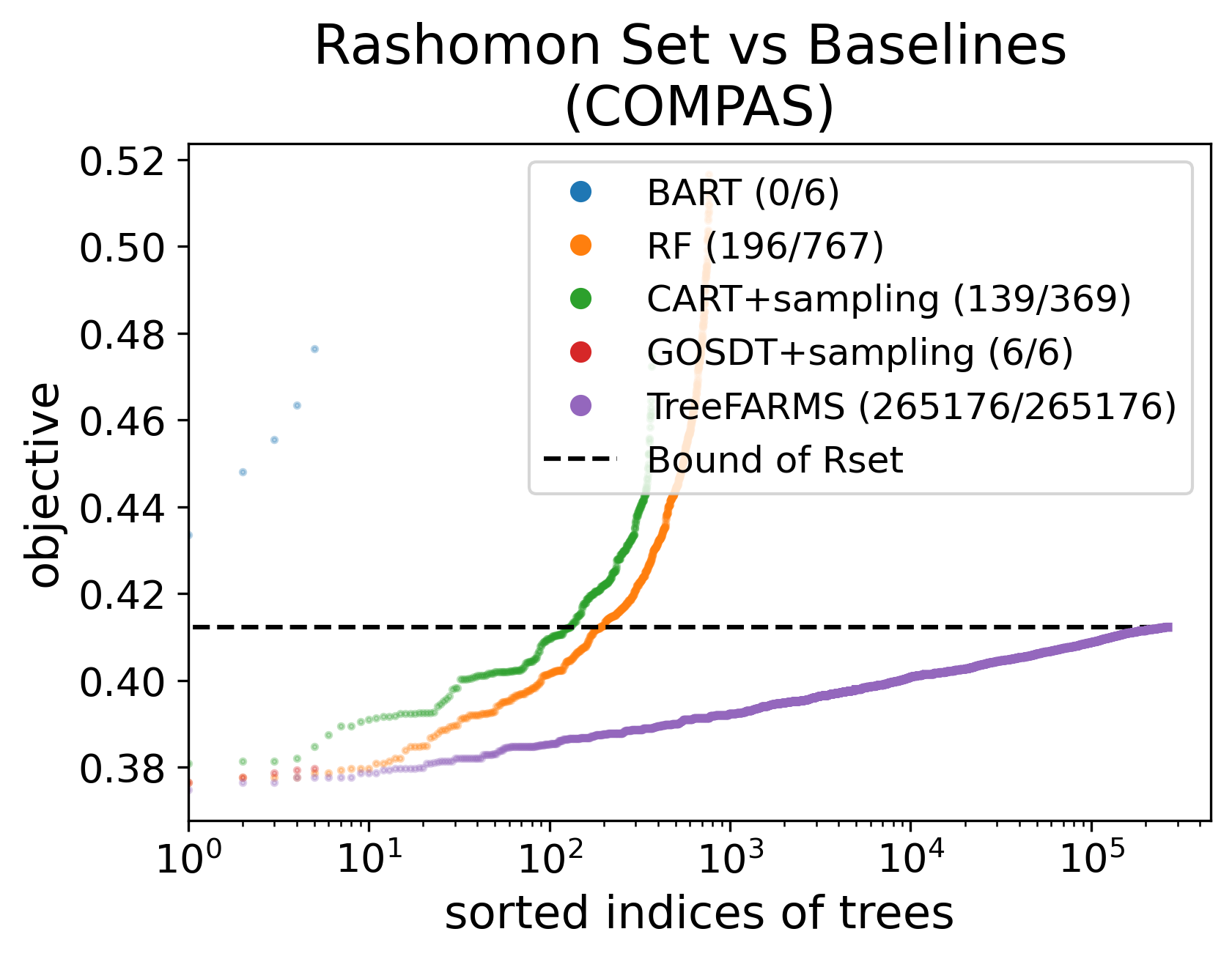}
     \includegraphics[height=0.25\textwidth]{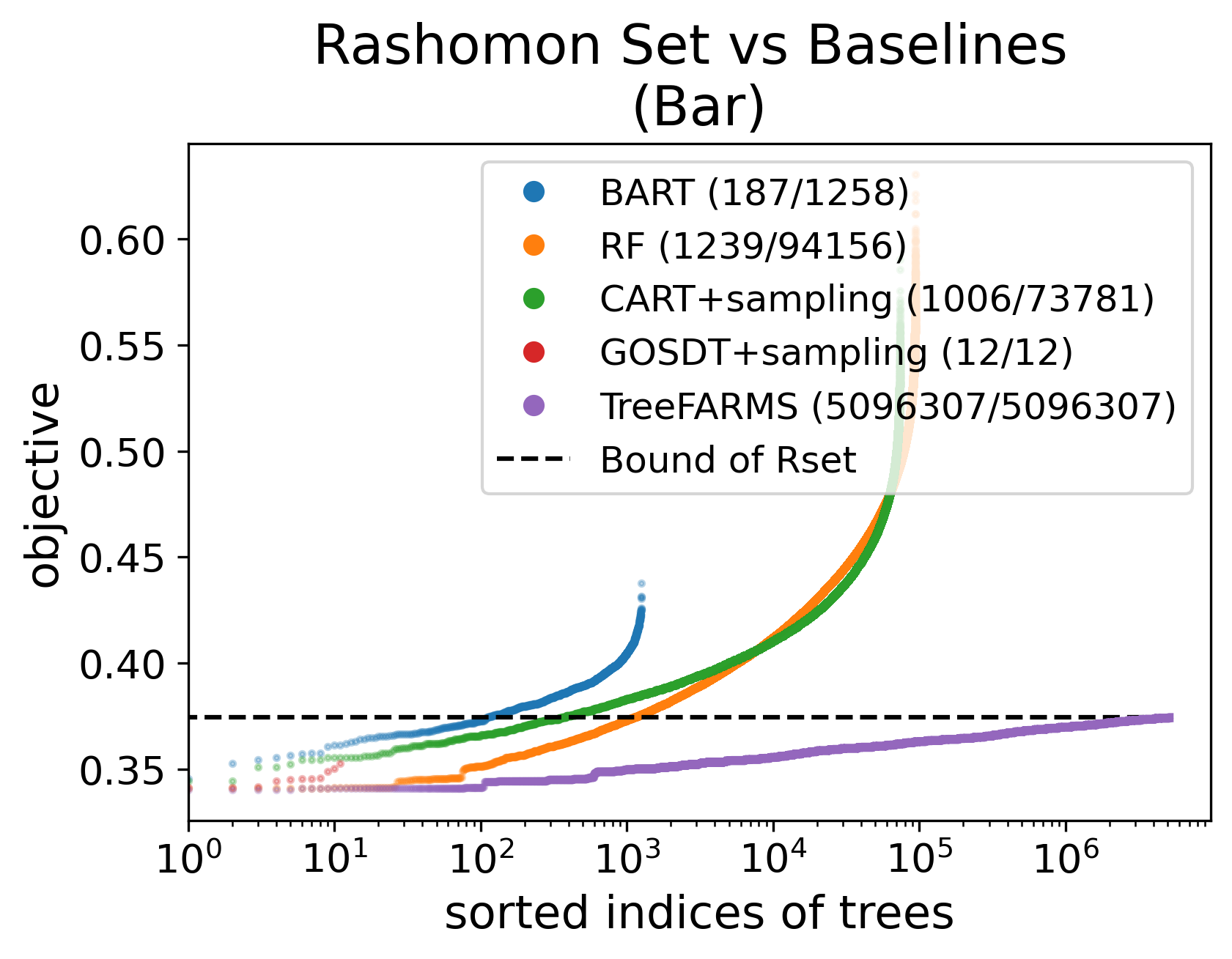}
    \caption{Comparison of trees in the Rashomon set ($\lambda=0.01, \epsilon=0.1$) and trees generated by baselines. Trees in the Rashomon set have objective below the dashed line. (A/B) in legend represents that A trees among B trees trained by the baseline are in the Rashomon set. For example, RF (196/767) means 196 trees among 767 distinct trees trained by Random Forest are in the Rashomon set. Indices are in log scale to accommodate differences in orders of magnitude of tree counts among methods. }
    \label{fig:comp_baselines}
\end{figure*}

Figure \ref{fig:comp_baselines} compares the Rashomon set with the four baselines on the Monk2, COMPAS, and Bar datasets. We show the number of distinct trees versus the objective value. We sort the trees with respect to their objective values, so all methods show an increasing trend. \ourmethod{} (in \textcolor{violet}{purple}) found \textbf{\textit{orders of magnitude more distinct trees in the Rashomon set}} than any of the four baselines on all of the datasets. The baseline methods tend to find many duplicated trees. For example, in 46 seconds, BART finds only 488 distinct trees on Monk2, whereas \ourmethod{} found $10^8$. Other methods find $\sim$20,000 distinct trees. Most trees found by the baselines are not even in the Rashomon set, i.e., most of their trees have objective values higher than the threshold of the Rashomon set. 


Figure \ref{fig:lambda_epsi_size} shows run times, specifically, Fig$.$ \ref{fig:lambda_epsi_size}(c) shows that \ourmethod{} \textit{\textbf{finds trees in the Rashomon set at a dramatically faster rate}} than the baselines. 
Appendix  \ref{app:exp_scalability} has more results. 

The takeaway from this experiment is that \textbf{\textit{natural baselines find at best a tiny sliver of the Rashomon set}}. Further, \textbf{\textit{the way we discovered this was to develop a method that actually enumerates the Rashomon set}}. We would not have known, using any other way we could think of, that sampling-based approaches barely scratch the surface of the Rashomon set. 

\begin{figure*}[!ht]
\centering

\includegraphics[width=0.95\textwidth]{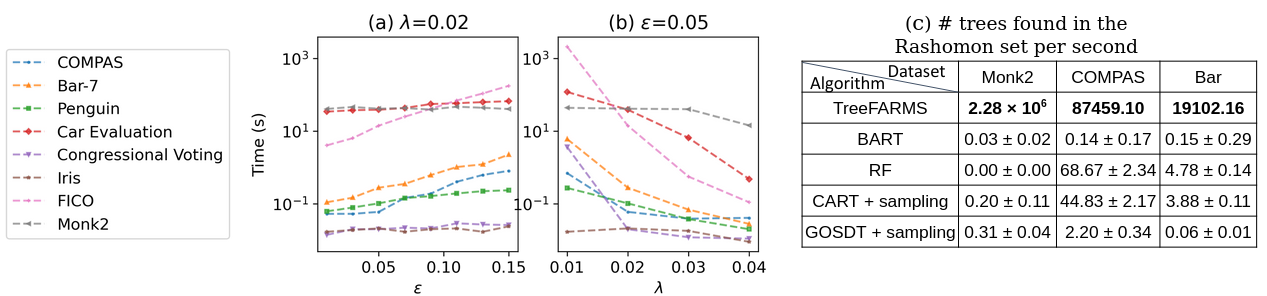}
\caption{
(a), (b) Run times for computing Rashomon sets as a function of $\epsilon$ and $\lambda$ respectively. 
(c) Number of trees in the Rashomon set found by each method per second. The total time is approximately 46, 3, and 270 seconds for Monk2, COMPAS, and Bar respectively. 
\ourmethod{} is the only algorithm \textit{guaranteed} to find all trees in the Rashomon set. (Appendix \ref{app:exp_scalability} has results for all the datasets.)
}
\label{fig:lambda_epsi_size}
\end{figure*}

\subsection{Variable Importance: Model Class Reliance}\label{sec:exp_mcr}
Without \ourmethod{}, it has not been possible to compute overall variable importance calculations such as MCR for complex function classes with interaction terms such as decision trees. Here, we exactly compute MCR on the COMPAS and Bar datasets (see Figure \ref{fig:tentative_mcr}). For the COMPAS dataset (left subfigure)features related to prior counts generally have high $\textrm{MCR}_+$, which means these features are very important for some trees in the Rashomon set. For the Bar dataset (right subfigure), features ``Bar\_1-3'' and ``Bar\_4-8'' have dominant $\textrm{MCR}_+$ and $\textrm{MCR}_-$ compared with other features, indicating that for all well-performing trees, these features are the most important. This makes sense, since people who go to bars regularly would be likely to accept a coupon for a bar.

\begin{figure}
    \includegraphics[width=0.48\textwidth]{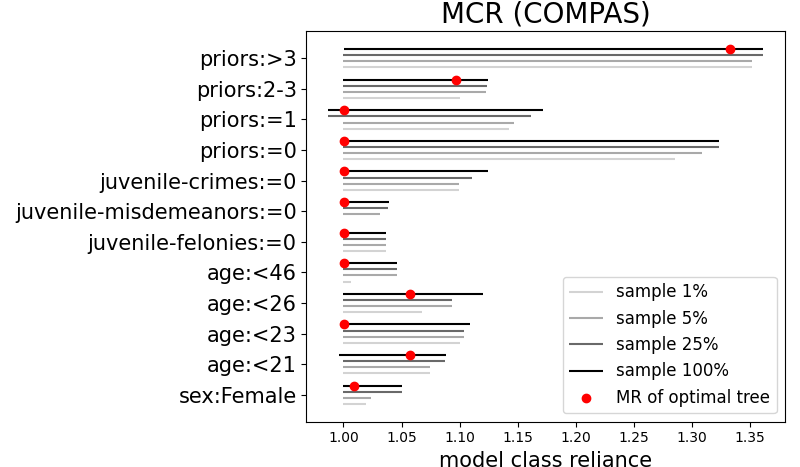}
    \includegraphics[width=0.48\textwidth]{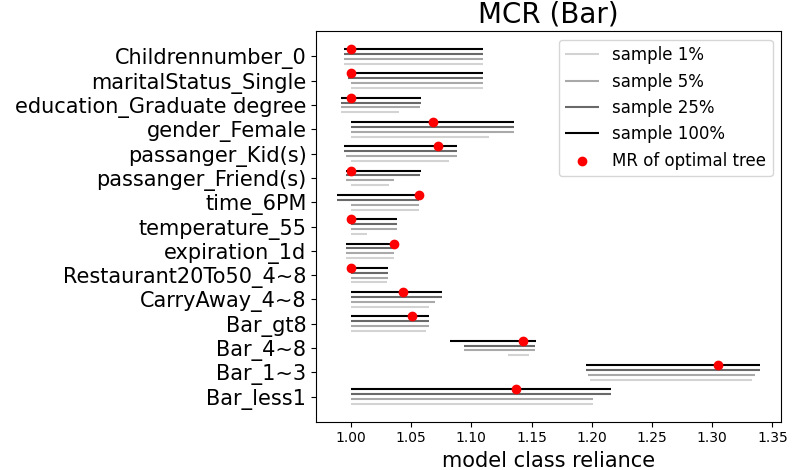}
    \caption{Variable Importance: Model class reliance on the COMPAS and Bar ($\lambda=0.01, \epsilon=0.05$). Red dots indicate the model reliance (variable importance) calculated from the optimal tree. 
    Each line connects $\textrm{MCR}_-$ and $\textrm{MCR}_+$ showing the range of variable importance among all good models.
    }
    \label{fig:tentative_mcr}
\end{figure}

\textbf{Sampling for MCR}: 
\label{sec:exp_mcr_sampling}
Sampling has a massive memory benefit over storing the whole Rashomon set, because we do not need to store the samples.
Since MCR requires computing extreme value statistics (max and min over the Rashomon set), it poses a test for the sampling technique posed in Section \ref{sec:sampling}. 
Figure \ref{fig:tentative_mcr} shows sampled MCR  and its convergence to true MCR. 25\% of the samples are usually sufficient. More results are shown in Appendix \ref{app:exp_mcr}.


\subsection{Balanced Accuracy and F1-score Rashomon set from Accuracy Rashomon set} \label{sec:exp:other_metrics}

As discussed in Section \ref{sec:other_metrics}, Rashomon sets of balanced accuracy and F1-Score are contained in the Rashomon set for accuracy. Figure \ref{fig:bacc-f1} shows trees in the Accuracy Rashomon set, which covers the Balanced Accuracy Rashomon set (left) and F1-score Rashomon set (right). The black dashed line indicates the corresponding objective thresholds 
and blue dots below the dashed line are trees in these Rashomon sets. 
Note that the tree with the minimum misclassification objective is not the tree that optimizes other evaluation metrics. For example, in the left subfigure, a single root node that predicts all samples as 1 has the optimal accuracy objective (in yellow), while another 6-leaf tree minimizes the balanced accuracy objective (in green). Actually, many trees have better balanced accuracy objective than the tree that minimizes the accuracy objective. A similar pattern holds for the F1-score Rashomon set (see right subfigure). Some trees that have worse accuracy objectives are better in terms of the F1-score objective. More figures are shown in Appendix \ref{app:exp_other_metrics}.

\begin{figure}
    \centering
    \includegraphics[height=0.27\textwidth]{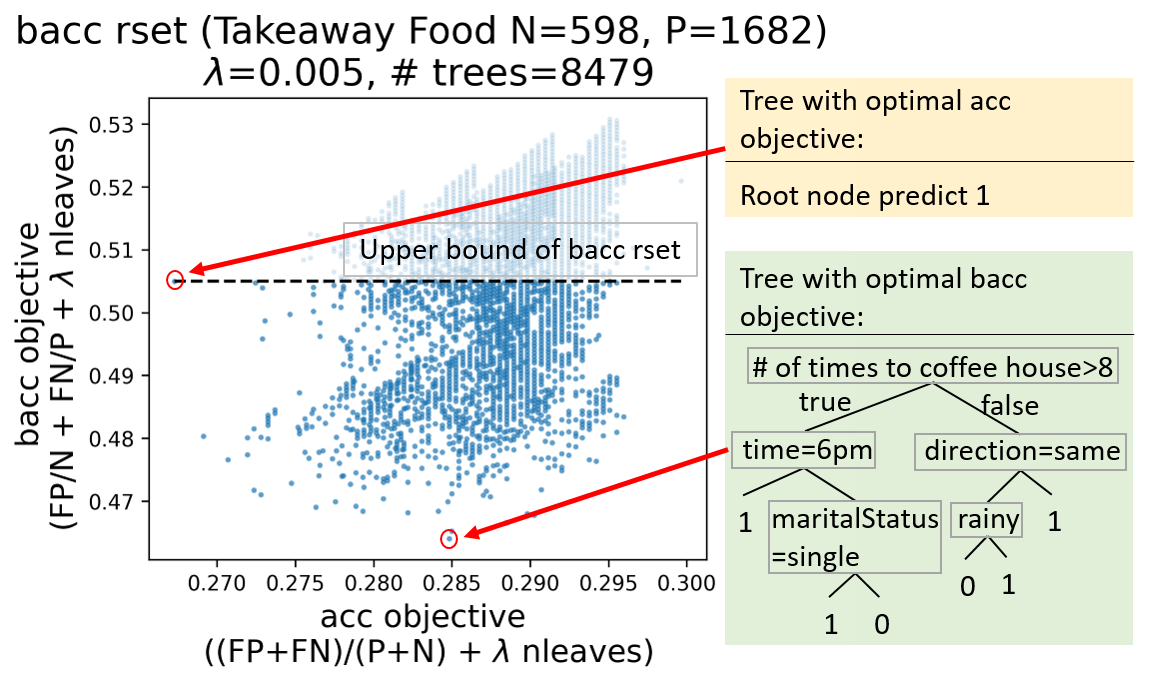}
    \includegraphics[height=0.27\textwidth]{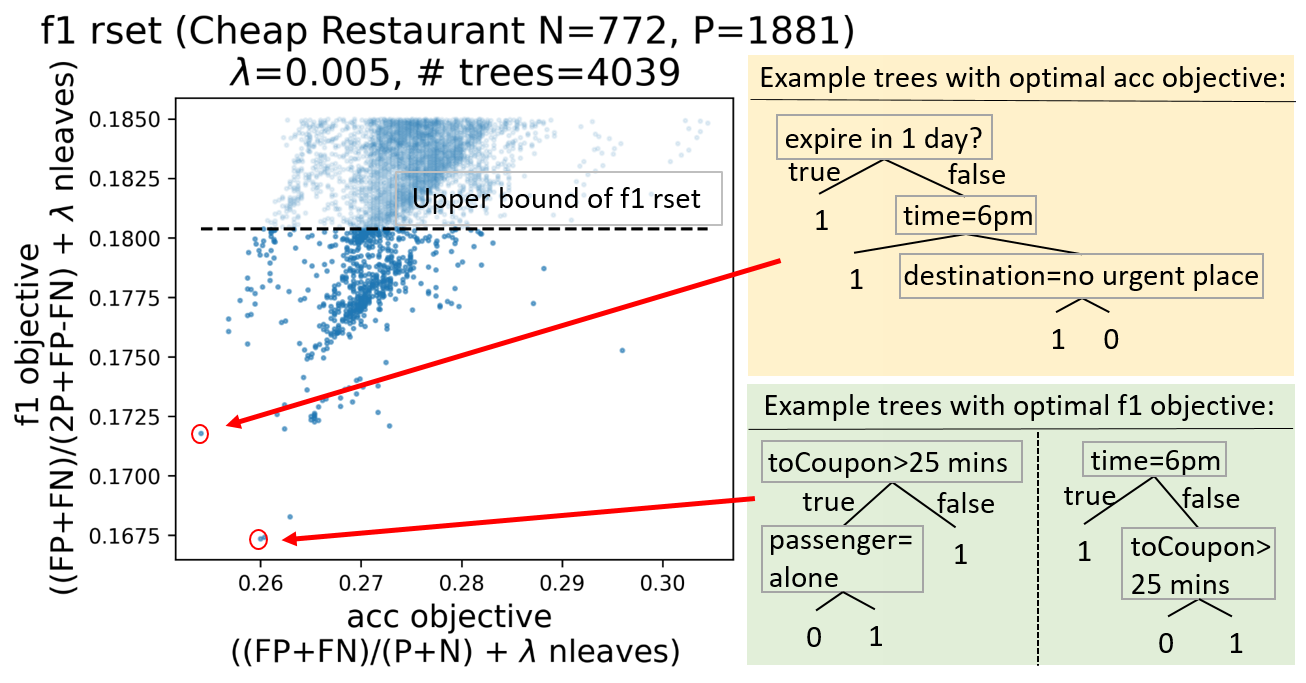}
    \caption{Example of Balanced Accuracy Rashomon set (left) and F1-Score Rashomon set (right). \# trees indicates the number of trees in the Balanced Accuracy or F1-score Rashomon set. Trees in the yellow region have  optimal accuracy objective and trees in the green region have optimal balanced accuracy or F1-score objective. }
    \label{fig:bacc-f1}
\end{figure}

\subsection{Rashomon set after removing a group of samples}
\label{sec:exp:influence}
\begin{figure}
    \centering
    \includegraphics[height=0.3\textwidth]{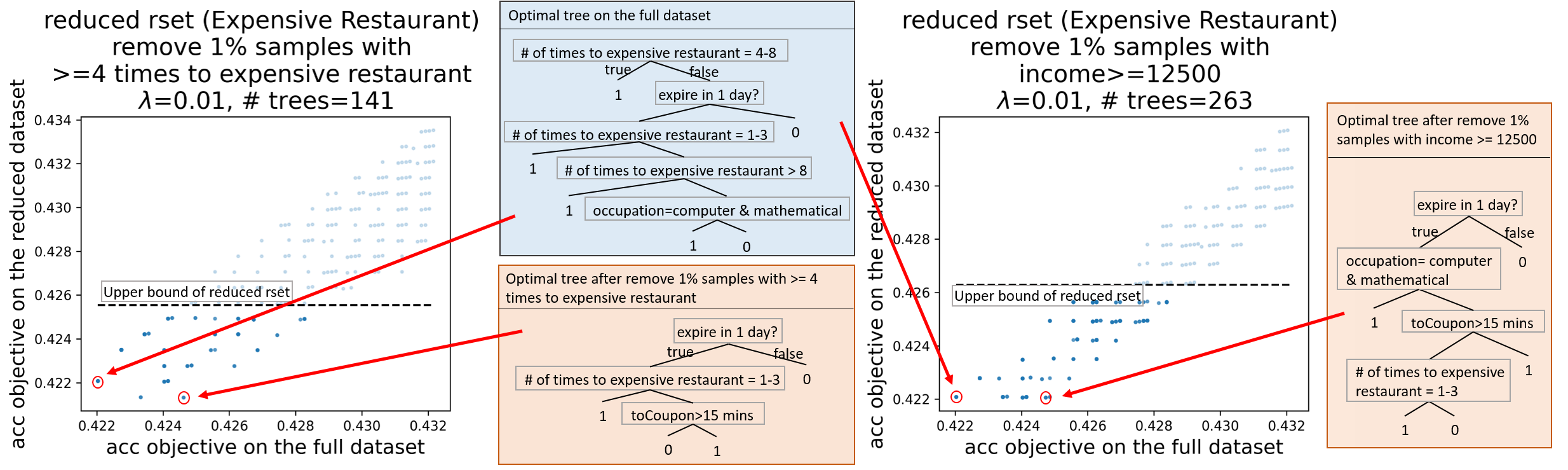}
    
    \caption{Example Rashomon sets and optimal trees after we remove the 1\% of samples with ``number of times to expensive restaurant $\geq$ 4'' (left) and ``income $\geq$ \$12,500'' (right) on the Expensive Restaurant dataset. The optimal tree on the full dataset is shown in the gray region and optimal trees on the corresponding reduced datasets are in the orange region. }
    \label{fig:reduced}
\end{figure}

Figure \ref{fig:reduced} shows accuracy objective on the full dataset versus the objective on the reduced dataset after 1\% of samples with ``number of times to expensive restaurant $\geq$ 4'' (left) and ``income $\geq$ \$12,500'' (right) are removed. The black dashed line indicates the objective threshold of the reduced Rashomon set and blue dots below the dashed line are trees in the reduced Rashomon set. As we can see, both scatter plots show a high correlation between the accuracy objective on the full dataset and the reduced dataset, indicating sparse near-optimal trees are robust to the shift in sample distribution. In other words, well-performing trees trained on the full dataset are usually still well-performing if some samples are removed. Optimal trees on the reduced dataset might be different, as we see by comparing the trees in the orange region and blue region. 
More results are shown in Appendix \ref{app:exp_sensitivity}. 

\section{Conclusion}
This work opens the door to interesting discussions on variable importance, distributional shift, and user options. By efficiently representing \textit{all} optimal and slightly suboptimal models for complex nonlinear function classes with interactions between variables, we provide a range of new user-centered capabilities for machine learning systems, and a new understanding of the importance of variables. Importantly, \ourmethod{} allows users a \textit{choice} rather than handing them a single model.

\section*{Acknowledgements}
We thank the anonymous reviewers for their suggestions and insightful questions. The authors acknowledge funding from the National Science Foundation under grants IIS-2147061 and IIS-2130250, National Institute on Drug Abuse under grant R01 DA054994, Department of Energy under grants DE-SC0021358 and DE-SC0023194, and National Research Traineeship Program under NSF grants DGE-2022040 and CCF-1934964. We acknowledge the support of the Natural Sciences and Engineering Research Council of Canada (NSERC).
Nous remercions le Conseil de recherches en sciences naturelles et en génie du Canada (CRSNG) de son soutien.

\section*{Code Availability }
Implementations of \ourmethod{} is available at https://github.com/ubc-systopia/treeFarms.

\bibliography{ref}
\bibliographystyle{unsrt}

\section*{Checklist}

\begin{enumerate}

\item For all authors...
\begin{enumerate}
  \item Do the main claims made in the abstract and introduction accurately reflect the paper's contributions and scope?
    \answerYes{}
  \item Did you describe the limitations of your work?
    \answerYes{\ourmethod's models should not be interpreted as causal. \ourmethod{} creates sparse binary decision trees and thus has limited capacity. \ourmethod's models inherit flaws from data it was trained on. \ourmethod{} is not yet customized to a given application, which can be done in future work.}
  \item Did you discuss any potential negative societal impacts of your work?
    \answerNA{} We cannot think of any.
  \item Have you read the ethics review guidelines and ensured that your paper conforms to them?
    \answerYes{}
\end{enumerate}

\item If you are including theoretical results...
\begin{enumerate}
  \item Did you state the full set of assumptions of all theoretical results?
    \answerYes{}
  \item Did you include complete proofs of all theoretical results?
    \answerYes{see Appendix \ref{app:theorems}}
\end{enumerate}

\item If you ran experiments...
\begin{enumerate}
  \item Did you include the code, data, and instructions needed to reproduce the main experimental results (either in the supplemental material or as a URL)?
    \answerYes{see Code Availability and Appendix \ref{app:datasets} and \ref{app:more_exp_results}}
  \item Did you specify all the training details (e.g., data splits, hyperparameters, how they were chosen)?
    \answerYes{see Appendix \ref{app:datasets} and \ref{app:more_exp_results}}
  \item Did you report error bars (e.g., with respect to the random seed after running experiments multiple times)?
    \answerYes{see Appendix \ref{app:more_exp_results}}
        \item Did you include the total amount of compute and the type of resources used (e.g., type of GPUs, internal cluster, or cloud provider)?
    \answerYes{see Appendix \ref{app:datasets} and \ref{app:more_exp_results}}
\end{enumerate}

\item If you are using existing assets (e.g., code, data, models) or curating/releasing new assets...
\begin{enumerate}
  \item If your work uses existing assets, did you cite the creators?
    \answerYes{}
  \item Did you mention the license of the assets?
     \answerYes{The code and data we used all have MIT licenses.}
  \item Did you include any new assets either in the supplemental material or as a URL?
      \answerYes{We include our code as a URL in Code Availability.}
  \item Did you discuss whether and how consent was obtained from people whose data you're using/curating?
    \answerNA{}
  \item Did you discuss whether the data you are using/curating contains personally identifiable information or offensive content?
    \answerNA{}
\end{enumerate}

\item If you used crowdsourcing or conducted research with human subjects...
\begin{enumerate}
  \item Did you include the full text of instructions given to participants and screenshots, if applicable?
    \answerNA{}
  \item Did you describe any potential participant risks, with links to Institutional Review Board (IRB) approvals, if applicable?
    \answerNA{}
  \item Did you include the estimated hourly wage paid to participants and the total amount spent on participant compensation?
    \answerNA{}
\end{enumerate}

\end{enumerate}


\newpage
\appendix
\addcontentsline{toc}{section}{Appendix} 
\part{Appendix} 
\parttoc 
\newpage

\section{The Model Set Representation}
\label{app:datastructure_detail}

In Section ~\ref{sec:new_data_structure}, we described the Model Set representation for the Rashomon Set. Figures ~\ref{fig:modelset} and ~\ref{fig:LRsets} illustrate these data structures in more detail.
Figure ~\ref{fig:modelset} introduces major components of our Model Set Instance (MSI) representation. 
Figure ~\ref{fig:LRsets} provides an example of the representation on a toy dataset. 

The Model Set representation's memory efficiency stems from two key properties: grouping models for the same subproblem together and then referencing these subproblems rather than duplicating them.
Figure ~\ref{fig:LRsets} illustrates this for a tiny, 10-sample dataset. Let's say that our Rashomon threshold is 0.40 and that we can split the entire problem on some feature to produce left and right children, each containing exactly half the data set -- the first five samples for the left subproblem and the last five samples for the right subproblem.

\begin{figure}[hb]
    \centering
    \includegraphics[width=1.0\textwidth]{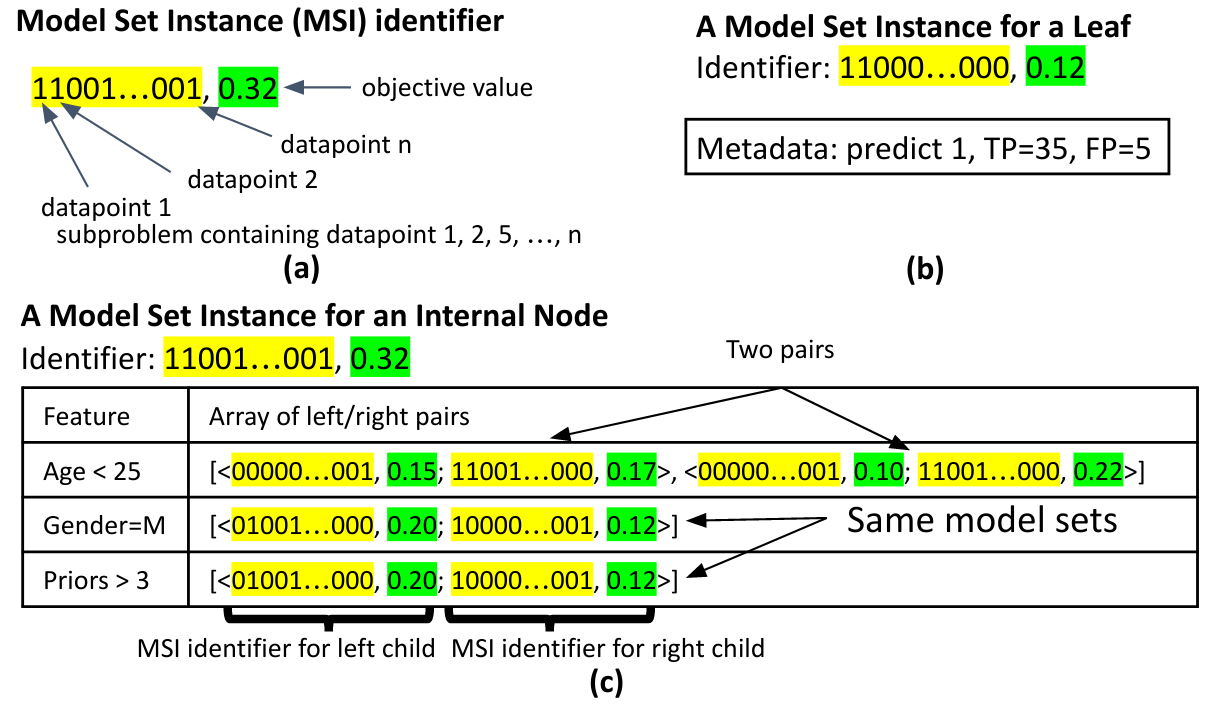}
    \caption{Model Set Instance (MSI) Representation. Each MSI represents a subproblem (yellow)/objective (green) pair (a). The subproblem is described by a bitvector identifying the samples in the subproblem. Figure (b) represents a leaf subproblem with objective 0.12. Figure (c) represents a subproblem that can be split to produce an objective of 0.32. There are three features (age, gender, and priors) for which a split produces this objective. If we split on age<25, then there are two different pairs of identifiers that both produce objectives of 0.32 (represented by the two pairs of MSI identifiers). Each pair of identifiers would further split on different features and have different tree structures. 
    The splits on gender and priors happen to produce the same pair of model set instances, as shown by the matching MSI identifiers in the bottom two lines of the table. Referencing these sets of trees by MSIs avoids massive data duplication.}
    \label{fig:modelset}
\end{figure}

\textit{extract} (see Algorithm \ref{alg:extract_short}, \ref{alg:newextract}) looks up the left subproblem
and finds that there are three MSIs for it, L1 represents a leaf and L2 and L3 each represent internal nodes with different objectives, each of which produce that objective in two ways. Thus, on the left, there are five possible trees. On the right, we have only two MSIs, one leaf (R1) and one internal node (R2). These represent three trees. Although the cross-product of the left and right children produce fifteen trees, we need only consider the 6 possible objectives values produced by combining each pair of model sets (i.e., L1+R1, L1+R2, L2+R1, etc.).
All but one of these satisfy the Rashomon threshold of 0.40 (L1+R1 produces an objective of 0.50, which is greater than the threshold), and they produce only three new model sets: one with objective 0.40 containing the two trees produced by combining L1 and R2 and the two trees produced by combining L2 and R1; one with objective 0.30 containing the four trees produced by combining L2 and R2 and the two trees produced by combining L3 and R1; and one with the objective 0.20 containing the four trees represented by combining L3 and R2.
\begin{figure}
    \centering
    \includegraphics[width=1.0\textwidth]{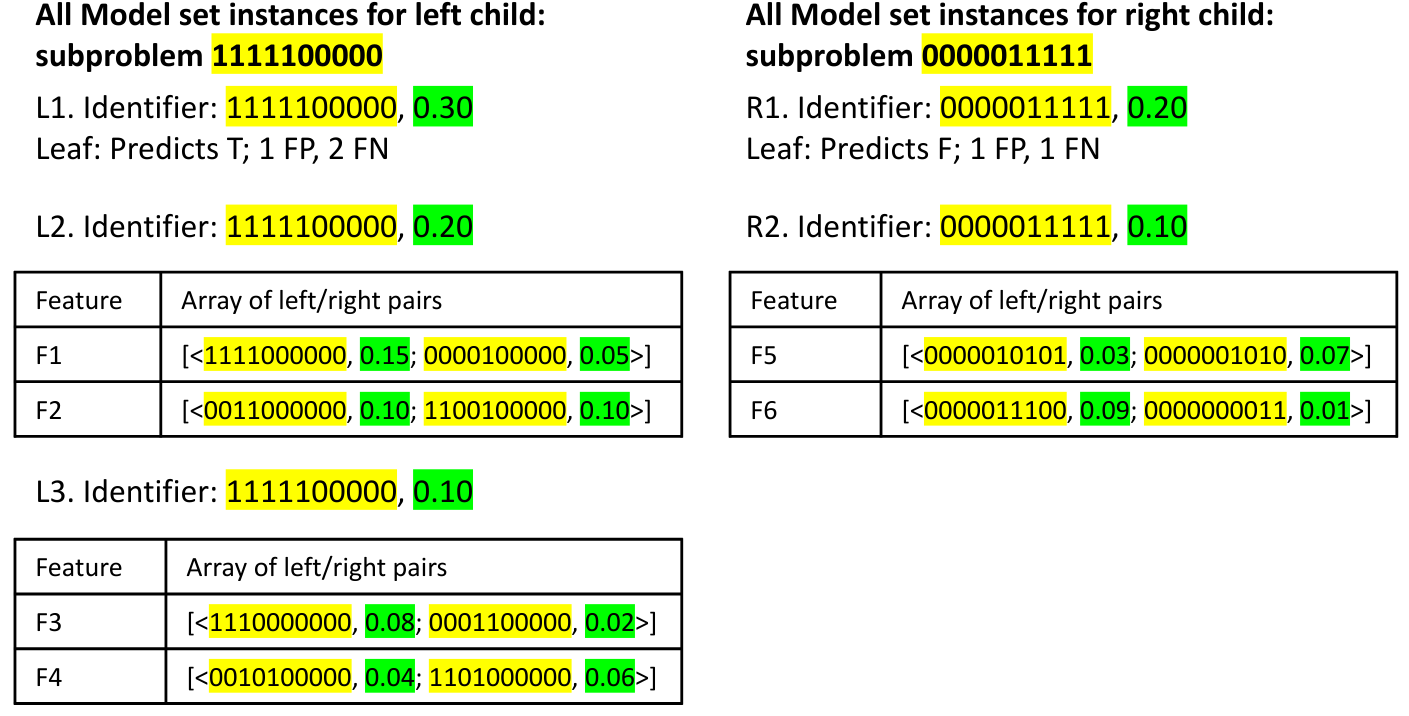}
    \caption{An example model set representation of the left and right children of a root split on a 10-sample dataset with Rashomon threshold of 0.40. In the left child, L1 represents a leaf, and L2 and L3 represent internal nodes with different objectives. In the right child, R1 represents a leaf, and R2 represents an internal node. Each of the internal nodes (L2, L3, R2) produce their objective in two ways. All children Model Set Instances of internal nodes (L2, L3, R2) contain only one tree. Thus, on the left, there are five possible trees, and the right represents three trees. }
    \label{fig:LRsets}
\end{figure}

Our implementation augments the Model Set with two indexes:
the \textbf{Subproblem Index (SI)} and the \textbf{Subproblem-Objective Index (SOI)}.
The \textbf{Subproblem Index (SI)} is an index from a subproblem to the set of all MSI identifiers for that subproblem,  with objectives under a certain value named \textit{scope}. This value helps ensure that we do not extract MSI that are provably not in the Rashomon set using Theorem \ref{corollary:lb_dev}. 
The \textbf{Subproblem-Objective Index (SOI)} maps from a subproblem/objective pair to an MSI. For example, in Figure \ref{fig:LRsets}, SI maps subproblem 1111100000 to identifiers of L1, L2, L3, and SOI maps the identifier of L1 to its metadata.

\section{The \ourmethod{} Algorithm}
\label{app:algorithm}
In Section~\ref{sec:newalgorithm}, we presented a condensed version of the extraction algorithm (Algorithm~\ref{alg:extract_short}); Algorithm ~\ref{alg:newextract} presents the full detail.

We maintain global data structures, $\MS$, the collection of all Model Set Instances in the hierarchical Model Set, $SI$ the Subproblem Index, and $SOI$, the Subproblem-Objective Index as described in Section~\ref{sec:new_data_structure} and Appendix~\ref{app:datastructure_detail}.
They are all initialized to be empty before calling the recursive procedure \textit{extract}, shown below, on the full data set.

Let $\MS_{s,o}$ be the MSI with objective $o$ for subproblem $s$. We implement $\MS_{s,o}$ using the Subproblem-Objective Index, $SOI$.


Let $\MS_{s}$ return a pair consisting of a scope and set of MSI identifiers, $(\textit{scope, msiSet})$.
The scope returned is the objective such that all trees with objective under that value are represented by \textit{msiSet} for subproblem $s$. 
The set of MSI identifiers contains all MSI for subproblem $s$ with objective less than or equal to $\textit{scope}$.
We implement $\MS_{s}$ using the Subproblem Index, $SI$.

More formally:

\textbf{if} $\textit{scope},\{\textit{msiSet}\}~\in~\MS_{s}$ \textbf{then}

$\>\>\>\>\textit{scope},\{\textit{msiSet}\}~\in~SI[s]$ \textbf{and}
$\forall~m$ \textbf{ in } $\textit{msiSet},~obj(m) \le scope$

When $extract$ returns, the Rashomon Set is represented by $MS_p$ where $p$ is the complete data set.

\begin{algorithm}
\caption{\textit{extract}$(G, sub, scope)$}
\label{alg:newextract}
\begin{tabbing}
xxx \= xx \= xx \= xx \= xx \= xx \kill
\comment{Given a dependency graph, $G$, a subproblem, $sub$, and a scope, $scope$, populate $MS$ with the} \\
\comment{Rashomon set for sub.}\\
\textit{scope'}, \textit{msi\_set}' ~$\leftarrow~SI[sub]$\\
\comment{When we have solved a problem for a given scope value, we guarantee to have found all possible } \\
\comment{objectives less than or equal to that scope. This line implements SOLVED in Algorithm \ref{alg:extract_short}.} \\
\textbf{if} \textit{scope'} is not None \textbf{and} \textit{scope}$ ~\leq~ $\textit{scope'} \textbf{then}\\
\>\comment{This subproblem is already solved.} \\
\> \textbf{return}\\

$\textit{msi\_set}~\leftarrow~\{\}$\\

$p~\leftarrow~G[sub]$ \comment{Find problem sub in the dependency graph.} \\
$\textit{base}\_\textit{bound}~\leftarrow~p.ub$ \comment{Objective if this node is a leaf.} \\

\textbf{if} ~\textit{base\_bound}~$\leq~$\textit{scope} \textbf{then} \comment{It is possible for this subproblem to be a leaf}\\
\> \comment{Select prediction that minimizes loss for this node. This line and the following implement }\\
\> \comment{newLeaf in Algorithm \ref{alg:extract_short}.}\\
\> \textit{prediction}$~\leftarrow$~\textrm{CALCULATE}$\_$\textrm{PREDICTION}$(p)$\\
\> \comment{Create single leaf tree represented in an MSI. \textit{newMSI} constructs an MSI using data } \\
\> \comment{provided in its arguments.} \\
\> \textit{msi}$~\leftarrow~\textit{newMSI}(\textit{leaf, sub, base\_bound, prediction})$ \\
\> $\textit{msi\_set}~\leftarrow~\textit{msi\_set} \cup \{\textit{msi.identifier}\}$\\
\> $SOI[sub, base\_bound]~\leftarrow~$\textit{msi}\\

\comment{Check all possible features on which we might split.} \\
 \textbf{for} $\textrm{each}~\textrm{feature}~j~\in~[1,M]$ \textbf{do} \\

\>\comment{Create subproblems by splitting problem by feature j.} \\
 \> $sub_l,sub_r~\leftarrow~split(sub,j)$ \\
 \>\comment{If either subproblem is not in dependency graph, we need not consider this split.}\\
 \> \textbf{if} $sub_l$ \textbf{not~in} $G$ \textbf{or} $sub_r$ \textbf{not~in} $G$ \textbf{then}\\
 \>\> \textbf{continue}\\
 \> $p_l~\leftarrow~G[sub_l]$ \comment{$p_l$ is a node in the dependency graph.}\\
 \> $p_r~\leftarrow~G[sub_r]$ \comment{$p_r$ is a node in the dependency graph.}\\
 \> \textbf{if} $p_l.lb + p_r.lb > scope$ \textbf{then}\\
 \>\>\textbf{continue}\\

 \> \comment{Populate Model Sets for the left and right subproblems. Recursively call \textit{extract} and } \\
 \> \comment{decrement scope using the lower bound of the other side. }\\
 \> $\textit{extract}(G, sub_l, scope-p_r.lb)$ \\
 \> $\textit{extract}(G, sub_r, scope-p_l.lb)$ \\
\>\comment{$MS$ now contains all Model Sets for the children; find all left and right model set} \\
\>\comment{instances.}\\
\> $\textit{left\_scope}, \textit{left\_msi\_set}~\leftarrow~SI[sub_l]$\\
\>\comment{Remove all  MSI whose objectives are too big.} \\
\> $\textit{left\_msi\_set}~\leftarrow~ \{ msi ~|~msi \in \textit{left\_msi\_set}~ \textbf{and}~ obj(msi) \leq scope\} $\\
\> $\textit{right\_scope}, \textit{right\_msi\_set}~\leftarrow~SI[sub_r]$\\
\> $\textit{right\_msi\_set}~\leftarrow~\{ msi ~|~msi \in \textit{right\_msi\_set}~ \textbf{and}~ obj(msi) \leq scope\} $\\

\>\comment{For each pair of model set instances in the cross product of left and right, if the sum of }\\
\>\comment{their objectives is less than or equal to $scope$, either create a new Model Set for this }\\
\>\comment{problem/objective, or add this pair to an existing Model Set for this problem/objective pair.}\\
 \> \textbf{for} each pair of model set instances $(m_l, m_r)\in (\textit{left\_msi\_set}\times \textit{right\_msi\_set})$ \textbf{do} \\

\>\> \comment{Skip leaf trivial extensions from the output.} \\

  \>\> \textbf{if} $\textit{is\_leaf}(m_l)$ \textbf{and} $\textit{is\_leaf}(m_r)$ \textbf{and} $m_l.\textit{predicts} == m_r.\textit{predicts}$ \textbf{then} \\

 \>\>\>  \textbf{continue}\\

\>\>$new\_obj \leftarrow m_l.obj + m_r.obj$ \\
\>\>\textbf{if } $new\_obj > scope$ \textbf{then} \comment{Skip combinations that exceed current scope.} \\
\>\>\>  \textbf{continue}\\
\>\>$msi~\leftarrow~SOI[sub, new\_obj]$\\
\>\> \textbf{if} not exists $msi$ \textbf{then} \\
\>\>\> $msi~\leftarrow~\textit{newMSI}(internal, sub, new\_obj)$ \\
\>\> \comment{Now, add this left/right pair to the model set for feature j.}\\
\>\> $msi[j].append(<m_l, m_r>)$\\
\>\> $\textit{msi\_set}~\leftarrow~\textit{msi\_set} \cup \{\textit{msi.identifier}\}$\\
\>\> $SOI[sub, new\_obj]~\leftarrow~msi$\\

\comment{We store solved instance should we need it later. }\\
$SI[sub]~\leftarrow~\textit{scope, msi\_set}$ \\
 \textbf{return}
\end{tabbing}
\end{algorithm}

\begin{algorithm}
\caption{CALCULATE\_PREDICTION$(p)$ $\rightarrow$ $\textit{prediction}$ \\ \comment{Compute the prediction of the given node.}\label{alg:calc-pred}}
\begin{tabbing}
xxx \= xx \= xx \= xx \= xx \= xx \kill

\comment{Prediction if we end this node as a leaf. Changes to 1 if there are more positives than negatives.} \\
$\textit{prediction} \leftarrow 0$  \\
\comment{ Calculate the total positive and negative weights of each equivalence class.} \\
\comment{$s \in \{0,1\}^n$ is a bitvector indicating datapoints we are considering for this subproblem. } \\
$s \leftarrow p.keys$  \\
\comment{Compute negatives using samples from only those points captured in this subproblem. }\\
$\textit{negatives} \leftarrow \sum_{i=1}^n \mathbbm{1}[y_i = 0 \land s_i=1]$ \\
$\textit{positives} \leftarrow \sum_{i=1}^n \mathbbm{1}[y_i=1 \land s_i=1]$ \\
\comment{ Set the leaf prediction based on class with the higher selected total weight. } \\
\textbf{if} \textit{negatives} $<$ \textit{positives} \textbf{then} \\
\> \comment{Leaf predicts the majority class as 1 since positive weights are higher.}\\ \> \comment{Ties are predicted negative w.l.o.g$.$ since the error rate is the same either way.} \\
\> $\textit{prediction}\leftarrow 1$ \\
\textbf{return} \textit{prediction}
\end{tabbing}
\end{algorithm}


\section{Sampling from the Rashomon Set}
\label{app:sample}

To facilitate sampling from Rashomon sets too large to materialize in memory, we add a small amount of metadata to our Model Set Representation.
Each MSI for a splittable problem retains a count of the total number of trees (unique trees, not MSIs) as well as counts for each entry in the map for the MSI (that is, a count of the number of unique trees attributable to each possible split).
Assume that $count(msi)$ returns the total number of trees represented by an entire Model Set Instance and $count(msi[j])$ returns the number of trees represented by a particular entry in the map of $msi$.
Algorithms~\ref{alg:sample} and \ref{alg:sample_recurse} present a brute-force algorithm for translating an index value between $[0, |R_{set}|)$ into a unique tree. To sample a tree, we randomly draw an index from the count of all trees (or the sum of count of all MSIs at the problem set level), and then call Algorithm \ref{alg:sample}. 

\begin{algorithm}
\label{alg:treeFarmFull}

\end{algorithm}

\begin{algorithm}
\caption{\textit{ndx\_to\_tree}$(\MS, ndx)~ \rightarrow~t$}
\comment{Find the $ndx^{th}$ tree in the Rashomon set represented by $\MS$}
\label{alg:sample}
\begin{tabbing}
xxx \= xx \= xx \= xx \= xx \= xx \kill
\comment{Get all MSI representing the full data set}\\
$\textit{msi\_set, scope}~\leftarrow~\MS_p$\\
\textbf{for} $msi$ \textbf{in} $\textit{msi\_set}$ \textbf{do} \\
\> \textbf{if} $count(msi) < ndx$ \textbf{then}\\
\>\> \textbf{break}\\
\> \comment{Align indices to the start of the next MSI} \\
\> $ndx~\leftarrow~ndx-count(msi)$\\
\comment{Either $ndx$ is larger than the size of the Rashomon set or $msi$ is the Model Set Instance}\\
\comment{in which we will find the tree corresponding to $ndx$.}\\
\textbf{if} $ndx \geq count(msi)$ \textbf{then}\\
\> \comment{No tree with this $ndx$ exists.}\\
\> \textbf{return NULL}\\
\textbf{return} \textit{ndx\_to\_tree\_in\_msi}$(msi, ndx)$
\end{tabbing}
\end{algorithm}

\begin{algorithm}

\caption{$\textit{ndx\_to\_tree\_in\_msi}(msi, ndx)~\rightarrow~t$}
\comment{Recursive procedure to construct the $ndx^{th}$ tree in Model Set Instance $msi$}
\label{alg:sample_recurse}
\begin{tabbing}
xxx \= xx \= xx \= xx \= xx \= xx \kill
\textbf{if} $has\_terminal(msi)$ \textbf{then}\\
\>\comment{this is a leaf; see if we want to return it, or a more complicated tree (with the same }\\
\>\comment{objective).}\\
\> \textbf{if} $ndx = 0$ \textbf{then}\\
\>\> \textbf{return} $\textit{make\_leaf\_node}(msi(\textit{prediction}))$\\
\> \comment{Continue on to a more complicated model, but account for the leaf model.}\\
\> $ndx~\leftarrow~ndx - 1$\\
\\
\comment{Iterate over features to determine in which feature this index appears.}\\
\textbf{for} $f$ \textbf{in} $msi.keys()$ \textbf{do}\\
\> \textbf{if} $count(msi[f]) < ndx$ \textbf{then}\\
\>\> \textbf{break}\\
\> $ndx~\leftarrow~ndx-count(msi[f])$\\
\\
\comment{At this point, f is the feature upon which we will split; next we need to find the pair in which $ndx$} \\
\comment{appears.}\\
\textbf{for} each pair ($\textit{left\_msi}, \textit{right\_msi}$) \textbf{in} $msi[f]$ \textbf{do}\\
\> $\textit{left\_count} ~\leftarrow~ count(\textit{left\_msi})$\\
\> $\textit{right\_count} ~\leftarrow~ count(\textit{right\_msi})$\\
\> \textbf{if} $ndx < \textit{left\_count} \times \textit{right\_count}$ \textbf{then}\\
\>\> \textbf{break}\\
\> $ndx~\leftarrow~ndx - (\textit{left\_count} \times \textit{right\_count})$\\
\\
\comment{We now have the precise pair in which we'll find $ndx$.}\\
$node~\leftarrow~\textit{make\_internal\_node}(f)$\\
$node['true']~\leftarrow~\textit{ndx\_to\_tree\_in\_msi}(\textit{left\_msi}, ndx ~//~ \textit{left\_count})$\\
$node['false']~\leftarrow~\textit{ndx\_to\_tree\_in\_msi}(\textit{left\_msi}, ndx ~\%~ \textit{left\_count})$\\
\textbf{return} $node$
\end{tabbing}
\end{algorithm}

\section{Model Class Reliance of Decision Trees}
\label{app:mcr}
In this appendix, we present the model class reliance calculation in detail. 
Given a dataset $\{\x, \y\}$, the error of a decision tree $t$ is defined by a nonnegative loss function $\ell$:
\begin{equation}
    e_{\textrm{orig}}(t) := \frac{1}{n}\sum_{i=1}^n \ell(t, \y_{[i]} \x_{[i, k]}, \x_{[i, \backslash k]}), 
\end{equation}
where $\x_{[i,k]}$ is the $k^{th}$ feature of sample $i$ and $\x_{[i, \backslash k]}$ is the set of  remaining features. 
To show how much the accuracy of a fixed tree $t$ relies on variable $k$, we define the permutation loss:
\begin{equation}
  e_{\textrm{divide}}(t) \! :=\!\frac{1}{2 \lfloor n/2\rfloor}\hspace*{-2pt}\sum_{i=1}^{\lfloor n/2 \rfloor} [\ell\{t, (\y_{[i]}, \x_{[i, \backslash k]}, \x_{[i+\lfloor n/2 \rfloor, k]})\} 
   + \ell \{t, (\y_{[i+\lfloor n/2 \rfloor]}, \x_{[i+\lfloor n/2 \rfloor, \backslash k]}, \x_{[i, k]})\} ].
\end{equation}

We define the model reliance (MR) by division, i.e., $\textrm{MR}(t) := (e_{\textrm{divide}}(t) + \lambda H_t)/(e_{\textrm{orig}}(t)+\lambda H_t)$. This definition is slightly different from \citepAppendix{appfisher2019all} as we include the leaf penalty term in both numerator and denominator to take tree complexity into consideration. If MR$(t)$ is large, it means the error went up substantially when  feature $k$ was altered, thus it is important.

MR evaluates how important a feature is for a specific tree $t$. However, such an estimation may overestimate or underestimate the feature's general importance. For example, a feature with high model reliance with respect to a tree $t$ might have low model reliance for another tree $t'$. Therefore, we are more interested in how much \textit{any} well-performing model from a decision tree class $\mathcal{T}$ relies on a feature. Given an $\epsilon$-Rashomon set (see Eq \eqref{eq:rset_definition}), 
the model class reliance is defined by:
\begin{equation}
    [\textrm{MCR}_-, \textrm{MCR}_+]:=[\min_{t\in R_{\set}(\epsilon, t_{\textrm{ref}}, \mathcal{T})}\textrm{MR}(t), \max_{t\in R_{\set}(\epsilon, t_{\textrm{ref}}, \mathcal{T})}\textrm{MR}(t)].
\end{equation}
According to \citepAppendix{appfisher2019all}, $\textrm{MCR}_-$ is usually easy to calculate as long as the loss function $\ell$ is convex and there is no leaf penalty term, but $\textrm{MCR}_+$ is hard to calculate since even if $\ell$ is convex, the maximization problem is usually non-convex. In the case of trees (or of any discrete functional class with interaction terms), this problem is discrete and certainly nonconvex, and
it is hard to compute either $\textrm{MCR}_-$ or $\textrm{MCR}_+$ without access to the type of algorithm presented in this work for enumerating the Rashomon set.


\section{Theorems and Proofs}\label{app:theorems}

We first recall some notation. A leaf set $t=(l_1, l_2, ..., l_{H_t})$ contains $H_t$ distinct leaves, where $l_i$ is the classification rule of leaf $i$. If a leaf is labeled, then $\hat{y}_i$ is the label prediction for all data in leaf $i$. A partially grown tree $t$ with the leaf set $t=(l_1, l_2, ..., l_{H_t})$ could be rewritten as $t=(t_{\fix}, \delta_{\fix}, t_{\splitrm}, \delta_{\splitrm}, H_t)$, where $t_{\fix}$ is a set of fixed leaves that are not permitted to be further split (those leaves will be split in another instance of this tree separately) and $\delta_{\fix}$ is a set of predicted labels of leaves in $t_{\fix}$. Similarly, $t_{\splitrm}$ is the set of leaves that can be further split and $\delta_{\splitrm}$ are the predicted labels of leaves $t_{\splitrm}$. We denote $\sigma(t)$ as the set of all $t$'s child trees whose fixed leaves contain $t_{\fix}$.

\textbf{Theorem \ref{corollary:lb}} \textit{ (Basic Rashomon Lower Bound)
Let $\theta_{\epsilon}$ be the threshold of the Rashomon set. Given a tree $t = (t_{\fix},\delta_{\fix}, t_{\splitrm}, \delta_{\splitrm}, H_t)$, we denote $b(t_{\fix}, \x, \y):=\ell(t_{\fix}, \x, 
\y)+\lambda H_t$ as the lower bound of the objective for tree $t$. If $b(t_{\fix}, \x, \y) > \theta_{\epsilon}$, then the tree $t$ and all of its children are not in the $\epsilon$-Rashomon set.}

\begin{proof}
According to the definition of the lower bound for tree $t$, we know 
$$Obj(t, \x, \y) = \ell(t, \x, \y) + \lambda H_t \geq \ell(t_{\fix}, \x, \y) + \lambda H_t = b(t_{\fix}, \x, \y).$$
By the same logic, for any child tree $t'=(t'_{\fix}, \delta'_{\fix}, t'_{\splitrm}, \delta'_{\splitrm}, H_{t'}) \in \sigma(t)$, we will also have $Obj(t', \x, \y) \geq b(t', \x, \y)$.  Since the child trees all have the fixed leaves of the parent, $t_{\fix} \subseteq t'_{\fix}$, and they have more leaves, $H_{t'} \geq H_t$, we have: $$b(t'_{\fix}, \x, \y) = \ell(t'_{\fix}, \x, \y) +\lambda H_{t'} \geq \ell(t_{\fix}, \x, \y) + \lambda H_t = b(t_{\fix}, \x, \y).$$
Thus, if $b(t_{\fix}, \x, \y) > \theta_{\epsilon}$, then $Obj(t', \x, \y) \geq b(t'_{\fix}, \x, \y) > \theta_{\epsilon}$, i.e., $t'$ is not in the $\epsilon$-Rashomon set. 

Therefore, if $b(t_{\fix}, \x, \y) > \theta_{\epsilon}$, we can eliminate the tree $t$ and its all of children from the search space.  
\end{proof}

We use notation defined in the main text for equivalent points. As a reminder, a set of points is equivalent if they have the same feature values; thus, those points will always receive identical predictions, and hence, some will be misclassified if they have opposite labels.

\textbf{Theorem \ref{thm:equiv}}
\textit{(Rashomon Equivalent Points Bound) Let $\theta_{\epsilon}$ be the threshold of the Rashomon set. Let $t$ be a tree with leaves $t_{\fix}, t_{\splitrm}$ and lower bound $b(t_{\fix}, \x, \y)$. Let $b_{\textrm{equiv}}(t_{\splitrm}, \x, \y):=\frac{1}{n}\sum_{i=1}^n \sum_{u=1}^U \textrm{cap}(\x_i, t_{\splitrm}) \wedge \mathbbm{1}[\x_i \in e_u] \wedge \mathbbm{1}[y_i=q_u]$ be the lower bound on the misclassification loss of leaves that can be further split. Let $B(t,\x, \y):= b(t_{\fix}, \x, \y) + b_{\textrm{equiv}}(t_{\splitrm}, \x, \y)$ be the Rashomon lower bound of $t$. If $B(t,\x, \y) > \theta_{\epsilon}$, tree $t$ and all its children are not in the $\epsilon$-Rashomon set.}  

\begin{proof}
According to the definition of $b(t_{\fix}, \x, \y)$ and $b_{\textrm{equiv}}(t_{\splitrm}, \x, \y)$, we know
$$Obj(t, \x, \y) = \ell(t_{\fix}, \x, \y) + \ell(t_{\splitrm}, \x, \y) + \lambda H_t \geq b(t_{\fix}, \x, \y) + b_{\textrm{equiv}}(t_{\splitrm}, \x, \y)=B(t, \x, \y).$$ 

For $t' \in \sigma(t)$, $Obj(t', \x, \y) \geq b(t_{\fix}, \x, \y) + b_{\textrm{equiv}}(t_{\splitrm}, \x, \y)=B(t, \x, \y)$. When $B(t, \x, \y) > \theta_{\epsilon}$, $Obj(t, \x, \y) > \theta_{\epsilon}$ and thus $Obj(t', \x, \y) > \theta_{\epsilon}$. 
\end{proof}

\begin{theorem}\label{corollary:lb_dev} (Rashomon Equivalent Points Bound for Subtrees)
Let $t$ be a tree such that the root node is split by a feature, where two subtrees $t_{\leftrm}$ and $t_{\rightrm}$ are generated with $H_{t_{\leftrm}}$ and $H_{t_{\rightrm}}$ leaves for $t_{\leftrm}$ and  $t_{\rightrm}$ respectively. Let $B(t_{\leftrm}, \x, \y)$ and $B(t_{\rightrm}, \x, \y)$ be the Rashomon equivalent points bound for the left and right subtrees respectively. (Note that $B(t_{\leftrm}, \x, \y) \leq \ell(t_{\leftrm}, \x, \y)+\lambda H_{t_{\leftrm}}$ and $B(t_{\rightrm}, \x, \y) \leq \ell(t_{\rightrm}, \x, \y)+\lambda H_{t_{\rightrm}}$).
If 
$B(t_{\leftrm}, \x, \y) + B(t_{\rightrm}, \x, \y) > \theta_{\epsilon}$, then the tree $t$ is not a member of the $\epsilon$-Rashomon set.
\end{theorem}

\begin{proof}
$$Obj(t, \x, \y) = \ell(t_{\leftrm},\x,\y) + \ell(t_{\rightrm}, \x, \y) + \lambda(H_{t_{\leftrm}} + H_{t_{\rightrm}}) \geq B(t_{\leftrm}, \x, \y) + B(t_{\rightrm},\x,\y).$$
If 
$B(t_{\leftrm}, \x, \y) + B(t_{\rightrm}, \x, \y) > \theta_{\epsilon}$, then $Obj(t, \x, \y) > \theta_{\epsilon}$. Therefore, tree $t$ is not in the $\epsilon$-Rashomon set. 
\end{proof}

We recall some notations used for Theorem \ref{thm:bacc} and \ref{thm:f1-score}. Let $q^+$ be the proportion of positive samples and $q^-$ be the proportion of negative samples, i.e., $q^+ + q^-=1$. We denote $q_{\min} := \min(q^+, q^-)$ and $q_{\max} := \max(q^+, q^-)$. Let $FPR$ and $FNR$ be the false positive and false negative rates. 
We notate the Accuracy Rashomon set as $A_{\theta}:=\{t\in \mathcal{T}: q^- FPR_t + q^+ FNR_t+\lambda H_t \leq \theta\}$, where $\theta$ is the objective threshold of the Accuracy Rashomon set, 
similar to $\theta_\epsilon$ in Section \ref{sec:notation}. 
We denote $\delta$ in these theorems as the objective threshold of Balanced Accuracy or F1-Score Rashomon sets.

\textbf{Theorem} \ref{thm:bacc}
\textit{(Accuracy Rashomon set covers Balanced Accuracy Rashomon set)  Let 
$B_{\delta}:=\{t \in \mathcal{T}: \frac{FPR_t + FNR_t}{2} + \lambda H_t \leq \delta\}$ be the Balanced Accuracy Rashomon set.
If \[\theta \geq \min\Big(2q_{\max}\delta, q_{\max} + (2\delta-1)q_{\min} + (1- 2q_{\min}) \lambda 2^d\Big),\] where $d$ is the depth limit, then $\forall t \in B_{\delta}, t\in A_{\theta}$.}

\begin{proof}
To prove this theorem, we first get the bound values through geometric intuition, and then prove the inequalities formally. Consider the plane with $FPR$ and $FNR$ being two axes, shown in Figure \ref{fig:bacc_explanation}, $\forall t \in B_{\delta}$, inequalities $\frac{FPR_t + FNR_t}{2} +\lambda H_t \leq \delta$ and $0 \leq FPR_t, FNR_t \leq 1$ bound the feasible region. Our goal is to find the upper bound of the accuracy objective in this feasible region, so that when $\theta$ is greater than this upper bound, we have $\forall t \in B_{\delta}, t\in A_{\theta}$. As shown in Figure \ref{fig:bacc_explanation} there are two different cases that can bound the accuracy of tree $t$. 

\begin{figure}
    \centering
    \includegraphics[width=0.95\linewidth]{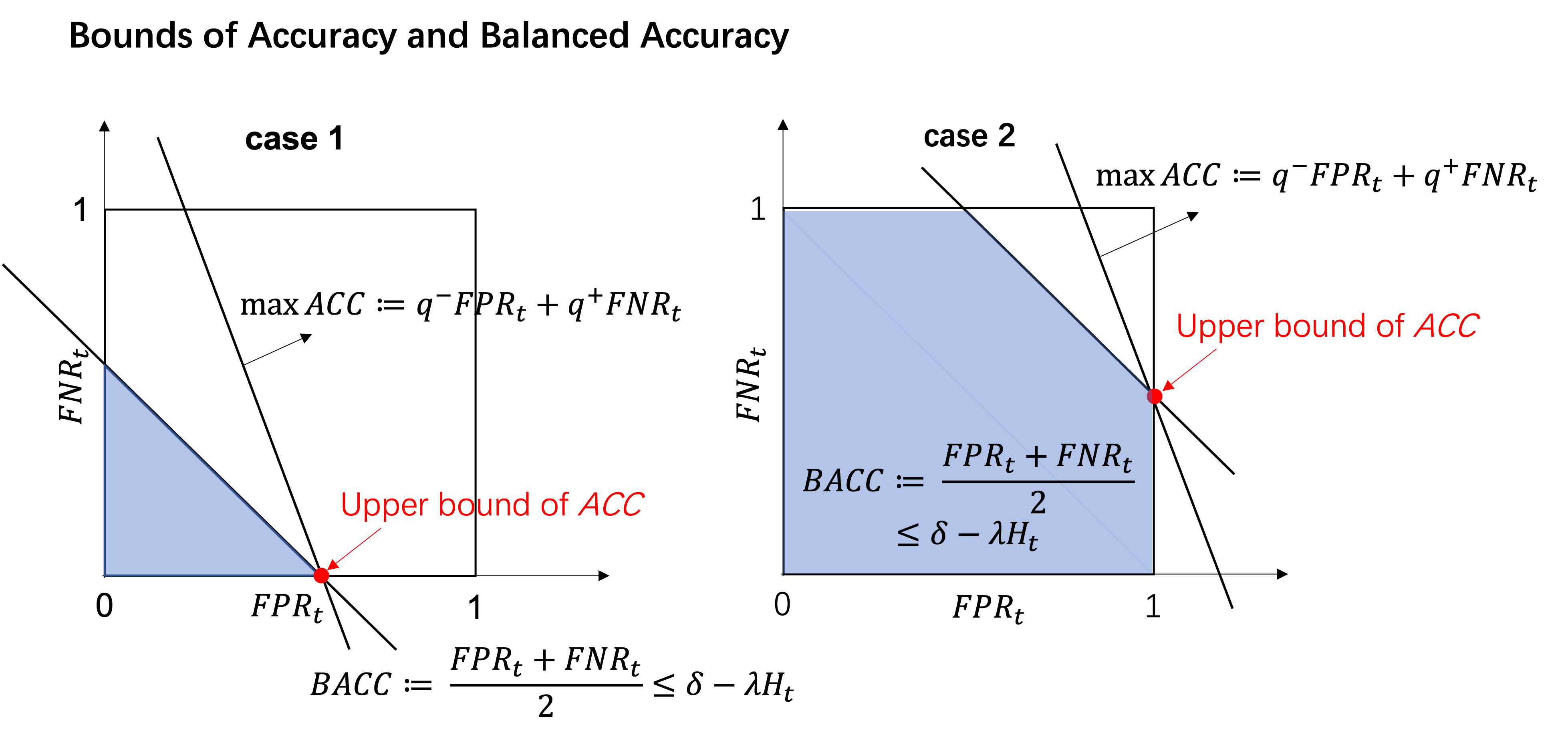}
    \caption{Two cases that can bound the accuracy of tree $t$. Note that for a tree $t$, its $FNR_t$ and $FPR_t$ are in range $[0,1]$. Left: when $2(\delta-\lambda H_t)<1$, the blue area identifies the feasible region. The point that maximize the accuracy is either $(0, 2(\delta-\lambda H_t))$ or $(2(\delta-\lambda H_t), 0)$. Right: when $2(\delta-\lambda H_t)\geq 1$, the blue area identifies the feasible region. The point that maximize the accuracy is either $(2(\delta-\lambda H_t)-1, 1)$ or $(1, 2(\delta-\lambda H_t)-1)$. }
    \label{fig:bacc_explanation}
\end{figure}

\noindent Case 1: 

When $2(\delta - \lambda H_t) < 1$ (see Figure \ref{fig:bacc_explanation} left), the accuracy loss is a line in the plane and thus maximized at $(0, 2(\delta-\lambda H_t))$ or $(2(\delta-\lambda H_t), 0)$. The corresponding maximum value is $\max(q^+ \times 2(\delta-\lambda H_t), q^-\times 2(\delta-\lambda H_t))=q_{\max}2(\delta-\lambda H_t)$. Inspired by this geometric intuition, formally we want to show,

\begin{equation} \label{eq:bacc_case1}
    \begin{aligned}
        q^- FPR_t + q^+ FNR_t &\leq  q_{\max} 2(\delta-\lambda H_t).
    \end{aligned}
\end{equation}
Eq \eqref{eq:bacc_case1} can be shown as follows,
\begin{equation}
    \begin{aligned}
        q^- FPR_t + q^+ FNR_t &\leq q_{\max}(FPR_t + FNR_t)\quad (q_{\max}\geq q^+, q_{\max}\geq q^-)\\
        &\leq q_{\max} 2(\delta-\lambda H_t) \quad (\textrm{see definition of } B_{\delta})
    \end{aligned}
\end{equation}
Therefore,
\begin{equation}
    \begin{aligned}
        q^- FPR_t + q^+ FNR_t + \lambda H_t &\leq q_{\max} 2(\delta-\lambda H_t) + \lambda H_t \quad\\
        &= 2 q_{\max}\delta - 2q_{\max}\lambda H_t + \lambda H_t\\
        &= 2 q_{\max}\delta - (2q_{\max}-1)\lambda H_t\\
        &\leq 2 q_{\max}\delta \quad (\textrm{since } q_{\max}\geq 0.5).
    \end{aligned}
\end{equation}

\noindent Case 2:
Inspired by Figure \ref{fig:bacc_explanation} right. When $2(\delta - \lambda H_t) \geq 1$, the the accuracy loss is maximized at either $(1, 2(\delta-\lambda H_t)-1)$ or $(2(\delta-\lambda H_t)-1, 1)$ on the $FPR$-$FNR$ plane, and the corresponding maximum value is $\max(q^+ \times (2(\delta-\lambda H_t)-1) + q^-,q^-\times (2(\delta-\lambda H_t)-1) + q^+)=q_{\max} + (2(\delta-\lambda H_t)-1)q_{\min}$. First, we want to show,  
\begin{equation}\label{eq:bacc_case2_main}
    q^-FPR_t + q^+ FNR_t \leq q_{\max} + (FPR_t + FNR_t-1)q_{\min}.
\end{equation}
\begin{equation}\label{eq:bacc_case2_induction}
\begin{aligned}
    \textrm{Eq \eqref{eq:bacc_case2_main} right -- Eq \eqref{eq:bacc_case2_main} left} &= q_{\max} + (FPR_t + FNR_t-1) q_{\min} - q^- FPR_t - q^+ FNR_t\\
    &=q_{\max} - q_{\min} + q_{\min}(FPR_t+FNR_t) - q^- FPR_t - q^+ FNR_t\\
    &=q_{\max} - q_{\min} + (q_{\min} - q^-)FPR_t + (q_{\min} - q^+)FNR_t.
\end{aligned}
\end{equation}
If $q^-\geq q^+$, then Eq \eqref{eq:bacc_case2_induction} = $q^--q^+ + (q^+ - q^-)FPR_t = (q^- - q^+) - (q^- - q^+)FPR_t \geq 0$ since $q^- - q^+ \geq 0$ and $1-FPR_t \geq 0$. Similarly, if $q^+ \geq q^-$, then Eq \eqref{eq:bacc_case2_induction} $=(q^+ - q^-) - (q^+ - q^-)FNR_t \geq 0$. Therefore, Eq \eqref{eq:bacc_case2_induction} is always greater than or equal to 0. Hence, Eq \eqref{eq:bacc_case2_main} holds. Therefore,
\begin{equation}
    \begin{aligned}
    q^- FPR_t + q^+ FNR_t &\leq q_{\max} + (FPR_t + FNR_t-1)q_{\min} \\
    & \leq q_{\max} + (2(\delta-\lambda H_t)-1) q_{\min} \quad (\textrm{see definition of } B_{\delta}),
    \end{aligned}
\end{equation}
which is the same as the maximum value we got intuitively from Figure \ref{fig:bacc_explanation} right. Adding $\lambda H_t$ to both sides, we have
\begin{equation}
    \begin{aligned}
    q^- FPR_t + q^+ FNR_t + \lambda H_t &\leq q_{\max} + (2(\delta-\lambda H_t)-1) q_{\min} + \lambda H_t\\
    &= q_{\max} + (2\delta-1) q_{\min} - 2 q_{\min}\lambda H_t + \lambda H_t\\
    &= q_{\max} + (2\delta-1) q_{\min} + (1- 2 q_{\min})\lambda H_t\\
    & \leq q_{\max} + (2\delta-1) q_{\min} + (1- 2 q_{\min}) \lambda 2^d \quad (\textrm{if depth is at most }d).
    \end{aligned}
\end{equation}
Combining these two cases, 
\begin{equation}
    q^- FPR_t + q^+ FNR_t + \lambda H_t\leq \min\bigg(2 q_{\max}\delta, q_{\max} + (2\delta-1)q_{\min} + (1- 2 q_{\min}) \lambda 2^d \bigg).
\end{equation}
Therefore, if $\theta \geq \min\bigg(2 q_{\max}\delta, q_{\max} + (2\delta-1)q_{\min} + (1- 2 q_{\min}) \lambda 2^d \bigg)$, then $\forall t \in B_{\delta}$, $q^-FPR_t + q^+FNR_t + \lambda H_t \leq \theta$. In other words, $\forall t \in B_{\delta}$, $t \in A_{\theta}$. 
\end{proof}

\textbf{Theorem} \ref{thm:f1-score} \textit{(Accuracy Rashomon set covers F1-score Rashomon set) Let 
\[F_{\delta}:=\left\{t \in \mathcal{T}: \frac{q^- FPR_t + q^+ FNR_t}{2q^+ + q^- FPR_t - q^+ FNR_t} + \lambda H_t \leq \delta\right\}\] be the F1-score Rashomon set. Suppose $q^+ \in (0,1)$, $q^- \in (0,1)$, and $\delta - \lambda H_t \in (0,1)$.
If $\theta \geq \min\Big(\max\big(\frac{2q^+\delta}{1-\delta}, \frac{2q^+(\delta - \lambda 2^d)}{1-(\delta - \lambda 2^d)} + \lambda 2^d\big),  \mathbbm{1}[\delta < \sqrt{2}-1]\frac{2\delta}{1+\delta} + \mathbbm{1}[\delta \geq \sqrt{2}-1](\delta + 3 -2\sqrt{2})\Big)$, then $\forall t \in F_{\delta}$, $t\in A_{\theta}$.
}

\begin{proof}
Similar to Theorem \ref{thm:bacc}, we first get the bound values through geometric intuitions, and then prove the inequalities formally. First, we want to represent the F1 loss with $FPR$ and $FNR$ so that we can put it in the $FPR$-$FNR$ plane. The F1-score is harmonic mean of precision and recall and is usually written in terms of $FP$ and $FN$ as $1-\frac{FP+FN}{2N^+ + FP - FN}$, where $N^+$ is the number of positive samples. 
With some operations we can rewrite this formula in terms of $FPR$ and $FPN$, 
\begin{equation*}
\begin{aligned}
    \frac{FP+FN}{2N^+ + FP - FN} &= \frac{(FP+FN)/ n}{(2q^+\times n + FP - FN)/n}\\
    &= \frac{\frac{FP}{n} \times \frac{q^-}{q^-} + \frac{FN}{n} \times \frac{q^+}{q^+}}{2q^+ + \frac{FP}{n}\times \frac{q^-}{q^-} -\frac{FN}{n} \times \frac{q^+}{q^+} }\\
    &= \frac{q^-FPR + q^+FNR}{2q^+ + q^-FPR - q^+FNR }
\end{aligned}
\end{equation*}
Here, 
for simplicity, we use $f1_t:= \frac{q^- FPR_t + q^+ FNR_t}{2q^+ + q^- FPR_t - q^+ FNR_t}$ to denote the F1 loss, which is 1 minus the F1-score of tree $t$. Based on the definition of $F_{\delta}$, $\forall t \in F_{\delta}$, we have

\begin{equation}\label{eq:f1-lp}
    \begin{aligned}
       & \frac{q^- FPR_t + q^+ FNR_t}{2q^+ + q^- FPR_t - q^+ FNR_t} + \lambda H_t \leq \delta\\
       \Rightarrow \quad & q^- FPR_t + q^+ FNR_t  \leq  (\delta - \lambda H_t) \times (2q^+ + q^- FPR_t - q^+ FNR_t)\\
       \Rightarrow \quad & q^+ FNR_t (1+\delta -\lambda H_t) \leq 2q^+ (\delta - \lambda H_t) + q^-(\delta-\lambda H_t-1) FPR_t\\
       \Rightarrow \quad & FNR_t \leq \frac{2(\delta - \lambda H_t)}{1+\delta-\lambda H_t} + \frac{q^-(\delta-\lambda H_t - 1)}{q^+(1+\delta-\lambda H_t)} FPR_t \quad (1+\delta-\lambda H_t>1)
    \end{aligned}
\end{equation}

This is a line. The slope of the F1-score boundary line is $\frac{q^-(\delta-\lambda H_t - 1)}{q^+(1+\delta-\lambda H_t)}$ and the slope of the accuracy boundary is $\frac{-q^-}{q^+}$. Since $\delta - \lambda H_t \in (0,1), \frac{\delta-\lambda H_t - 1}{1+\delta - \lambda H_t} < 0$ and $|\frac{\delta-\lambda H_t - 1}{1+\delta - \lambda H_t}| < 1$. Therefore, both slopes are negative and the slope of F1-score boundary is always larger than the slope of the accuracy boundary (see Figure \ref{fig:f1_explanation}). In addition, we can also show the intercept, which is also the upper bound of $FNR$ in Eq \eqref{eq:f1-lp}, $\frac{2(\delta-\lambda H_t)}{1+\delta+\lambda H_t}<1$, since $\delta-\lambda H_t\in (0,1)$ and $\frac{2(\delta-\lambda H_t)}{1+\delta+\lambda H_t}$ is monotonically increasing with $\delta-\lambda H_t$. Therefore,
\begin{equation} \label{eq:fnr_le1}
    FNR<1.
\end{equation}
Since $0 \leq FPR, FNR \leq 1$, $\forall t \in F_{\delta}$, there are also two different cases that can bound the accuracy of tree $t$ (see Figure \ref{fig:f1_explanation}). 
\begin{figure}
    \centering
    \includegraphics[width=0.95\linewidth]{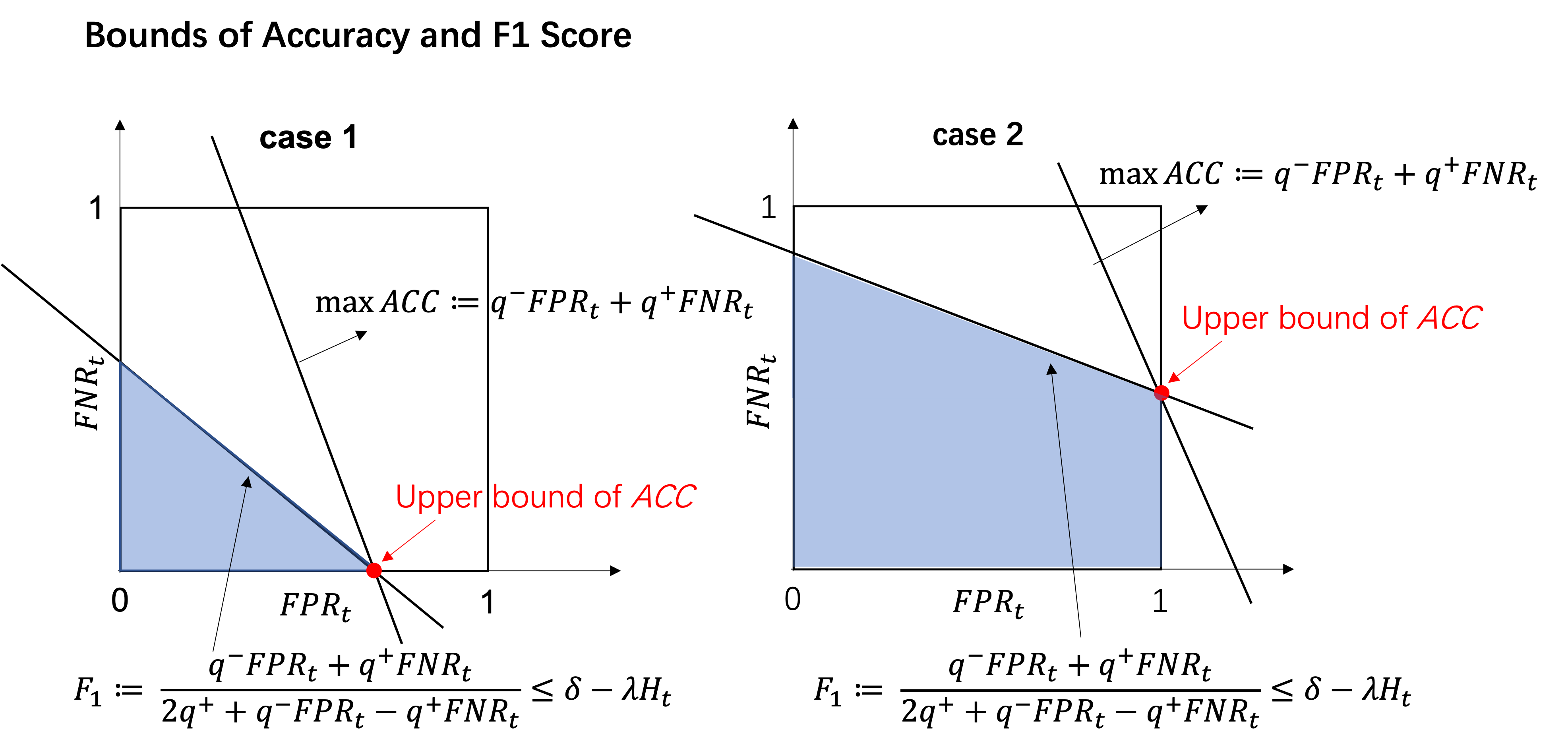}
    \caption{Two cases that can bound the accuracy of tree $t$. Note that for a tree $t$, its $FNR_t$ and $FPR_t$ are in range $[0,1]$. Left: when $FPR_t < 1$, the blue area identifies the feasible region. The point that maximize the accuracy is when $FNR_t=0$.
    Right: when $FPR_t = 1$, the blue area identifies the feasible region. The point that maximize the accuracy is when the line intersects the right boundary. }
    \label{fig:f1_explanation}
\end{figure}

\noindent Case 1: 
Based on Figure \ref{fig:f1_explanation} left, when $FPR_t < 1$, the point that maximizes the accuracy is when $FNR_t=0$. Using Eq \eqref{eq:f1-lp}, we know the point is $(\frac{-2q^+(\delta - \lambda H_t)}{q^-(\delta-\lambda H_t-1)}, 0)$, and the corresponding loss is $\frac{2q^+(\delta - \lambda H_t)}{1 + \lambda H_t -\delta}$. Formally, we want to show 

\begin{equation}\label{eq:fscore-case1}
    q^-FPR_t + q^+FNR_t \leq \frac{2q^+(\delta - \lambda H_t)}{1-(\delta - \lambda H_t)}
\end{equation}
Eq \eqref{eq:fscore-case1} is shown as follows, 
\begin{eqnarray}
    q^-FPR_t + q^+FNR_t \!&\leq&\! \frac{q^-FPR_t + q^+FNR_t}{1-FNR_t} \ \  ( \textrm{Eq \eqref{eq:fnr_le1}}, \textrm{and denominator}\leq 1) \\
    &=& \frac{2q^+ (q^-FPR_t + q^+FNR_t)}{2q^+(1-FNR_t)}\\
    &=&\frac{2q^+ (q^-FPR_t + q^+FNR_t)/(2q^+ + q^- FPR_t - q^+ FNR_t)}{(2q^+ - 2q^+ FNR_t)/(2q^+ + q^- FPR_t - q^+ FNR_t)}\\
    &=&\frac{2q^+ (q^-FPR_t + q^+FNR_t)/(2q^+ + q^- FPR_t - q^+ FNR_t)}{\frac{2q^+ + q^-FPR_t - q^+FNR_t - q^-FPR_t -q^+ FNR_t}{2q^+ + q^- FPR_t - q^+ FNR_t}}\\
    &=& \frac{2q^+ (q^-FPR_t + q^+FNR_t)/(2q^+ + q^- FPR_t - q^+ FNR_t)}{1-f1_t}\\
    &=&\frac{2q^+ f1_t}{1-f1_t}. \label{eq:f1_case1_induction}
\end{eqnarray}
Note that the denominator cannot be 0, because $FNR<1$ as shown in Eq \eqref{eq:fnr_le1}. Eq \eqref{eq:f1_case1_induction} is monotonically increasing in $f1_t$, since its first derivative is $\frac{2q^+(1-f1_t) - 2q^+f1_t(-1)}{(1-f1_t)^2} = \frac{2q^+}{(1-f1_t)^2} > 0$. Since $t\in F_{\delta}$, the maximum F1-score is $\delta-\lambda H_t$. Therefore,
\begin{equation}
    q^-FPR_t + q^+FNR_t \leq \frac{2q^+ f1_t}{1-f1_t} \leq \frac{2q^+(\delta - \lambda H_t)}{1-(\delta - \lambda H_t)}.
\end{equation}
Adding the leaf penalty on both sides, we get 
\begin{equation}
    q^-FPR_t + q^+FNR_t + \lambda H_t \leq \frac{2q^+(\delta - \lambda H_t)}{1-(\delta - \lambda H_t)} + \lambda H_t.
\end{equation}

\noindent Since $H_t$ is a variable, for simplicity, we let $a:=\lambda H_t$. Then we can define $g(a) :=  \frac{2q^+(\delta - a)}{1-(\delta - a)} + a$. 

\noindent Solving $g'(a) = \frac{-2q^+}{(1-\delta + a)^2} + 1= 0$, we get $a = \delta-1 \pm \sqrt{2q^+}$. Since $a$ is the leaf penalty, $a \geq 0$, we have only one solution for $a$. That is, $a=\delta-1 + \sqrt{2q^+}$. Now we consider $g''(a) = \frac{4q^+(1-\delta+a)}{(1-\delta+a)^4}$, and, plugging in our solution for $a$, we see  $g''(a)\geq 0$. Therefore, $g(a)$ achieves the minimum value when $a=\delta-1 + \sqrt{2q^+}$. 

If $0\leq \delta-1 + \sqrt{2q^+} \leq \lambda 2^d$, then for $a \in [0, \delta-1+\sqrt{2q^+}], g'(a)\leq 0$, indicating $g(a)$ monotonically decreases in this range, while for $a \in [\delta-1+\sqrt{2q^+}, \lambda 2^d], g'(a)\geq 0$, indicating $g(a)$ monotonically increases. Therefore, for $a \in [0, \lambda 2^d]$, $g(a)$ first increases and then decreases, and thus maximized at the two ends, i.e., 
$$\max g(a) = \max(g(a=0), g(a=\lambda 2^d)) = \max\left(\frac{2q^+\delta}{1-\delta}, \frac{2q^+(\delta - \lambda 2^d)}{1-(\delta - \lambda 2^d)} + \lambda 2^d\right).$$ 

\noindent In summary, 
\begin{equation}
     q^-FPR_t + q^+FNR_t + \lambda H_t \leq \max\bigg(\frac{2q^+\delta}{1-\delta}, \frac{2q^+(\delta - \lambda 2^d)}{1-(\delta - \lambda 2^d)} + \lambda 2^d\bigg).
\end{equation}

\noindent Case 2: Inspired by Figure \ref{fig:f1_explanation} right, we can find that the maximizer when $FPR_t = 1$ at $(1, \frac{(q^++1)(\delta - \lambda H_t)-q^-}{q^+(1+\delta-\lambda H_t)})$ using Eq \eqref{eq:f1-lp}. And the corresponding loss is $\frac{2(\delta - \lambda H_t)}{1+\delta-\lambda H_t}$. Formally, we want to show 

\begin{equation}\label{eq:fscore-case2}
    q^-FPR_t + q^+FNR_t \leq \frac{2(\delta-\lambda H_t)}{1+\delta - \lambda H_t}. 
\end{equation}
To show Eq (\ref{eq:fscore-case2}),
\begin{eqnarray}
    q^-FPR_t + q^+FNR_t &\!\leq\!& \frac{q^-FPR_t + q^+FNR_t}{q^+ + q^- FPR_t} \ (FPR_t\! \leq\! 1, \textrm{denominator}\!\in\! (0, \!1])\\
    &\!=\!& \frac{(q^-FPR_t + q^+FNR_t)/(2q^+ + q^- FPR_t - q^+ FNR_t)}{(q^+ + q^- FPR_t)/(2q^+ + q^- FPR_t - q^+ FNR_t)}\\
    &=&\frac{f1_t}{(1+f1_t)/2}\\
    &=& \frac{2f1_t}{1+f1_t}.\label{eq:f1_case2_induction}
\end{eqnarray}

\noindent Eq \eqref{eq:f1_case2_induction} is monotonically increasing with $f1$, since its derivative $=\frac{2(1+f1_t) - 2f1_t}{(1+f1_t)^2} = \frac{2}{(1+f1_t)^2} \geq 0$. Moreover, since $t \in F_{\delta}$, the maximum f1-score loss is $\delta - \lambda H_t$. Therefore, 
\begin{equation}
    q^-FPR_t + q^+FNR_t \leq \frac{2f1_t}{1+f1_t} \leq \frac{2(\delta - \lambda H_t)}{1+\delta - \lambda H_t}. 
\end{equation}
\noindent Adding leaf penelty to two sides, we can get 
\begin{equation}\label{eq:f1_case2_change_var}
    q^-FPR_t + q^+FNR_t + \lambda H_t \leq \frac{2(\delta - \lambda H_t)}{1+\delta - \lambda H_t} + \lambda H_t. 
\end{equation}

\noindent Let's change the variable. Let $a=\lambda H_t$. Then the right-hand side of Eq \eqref{eq:f1_case2_change_var} is a function of $a$. Let $g(a) := \frac{2(\delta - a)}{1+\delta - a} + a$.

\begin{equation}
     \frac{dg}{da} = \frac{-2(1+\delta-a) - 2(\delta-a)(-1)}{(1+\delta-a)^2} + 1
     =\frac{-2}{(1+\delta-a)^2}+1
\end{equation}
Solving $\frac{dg}{da} = 0$, we get 
\begin{equation}
    a = 1+\delta \pm \sqrt{2}.  
\end{equation}
Since $a$ is the leaf penalty term, $\delta-a>0$ according to our assumption. If $a = 1+\delta + \sqrt{2}$, $\delta-a = -2.414<0$ which contradicts the assumption. 
Therefore, only $a = 1+\delta - \sqrt{2}$ could be the valid stationary point. We then calculate $g''(a) = \frac{-4(1+\delta-a)}{(1+\delta-a)^4} \leq 0$. Therefore, $a=1+\delta-\sqrt{2}$ is the maximizer of $g(a)$.
\begin{enumerate}[label=(\roman*)]
    \item when $\delta + 1 -\sqrt{2} < 0$, we still have that $a$ cannot be smaller than 0. Therefore, the maximum value of $g(a)$ happens when $a=0$, which is equal to $\frac{2\delta}{1+\delta}$. 
    \item when $\delta + 1 -\sqrt{2} \geq 0$, the maximum value of $g(a)$ is when $a=\delta + 1 - \sqrt{2}$, which is equal to $\delta+3-2\sqrt{2}$. 
\end{enumerate}

\noindent In summary, 
\begin{equation}
    q^-FPR_t + q^+FNR_t + \lambda H_t \leq \mathbbm{1}[\delta < \sqrt{2}-1]\frac{2\delta}{1+\delta} + \mathbbm{1}[\delta \geq \sqrt{2}-1](\delta + 3 -2\sqrt{2}). 
\end{equation}

Combining the two cases together, 
\begin{equation}
\begin{aligned}
    q^-FPR_t + q^+FNR_t + \lambda H_t \leq \min\Bigg(&\max\bigg(\frac{2q^+\delta}{1-\delta}, \frac{2q^+(\delta - \lambda 2^d)}{1-(\delta - \lambda 2^d)} + \lambda 2^d\bigg), \\
    &\mathbbm{1}[\delta < \sqrt{2}-1]\frac{2\delta}{1+\delta} + \mathbbm{1}[\delta \geq \sqrt{2}-1](\delta + 3 -2\sqrt{2})\Bigg).
\end{aligned}
\end{equation}

Therefore, If $\theta \geq \min\Big(\max\big(\frac{2q^+\delta}{1-\delta}, \frac{2q^+(\delta - \lambda 2^d)}{1-(\delta - \lambda 2^d)} + \lambda 2^d\big),  \mathbbm{1}[\delta < \sqrt{2}-1]\frac{2\delta}{1+\delta} + \mathbbm{1}[\delta \geq \sqrt{2}-1](\delta + 3 -2\sqrt{2})\Big)$, then $\forall t \in F_{\delta}$, $q^-FPR_t + q^+FNR_t + \lambda H_t \leq \theta$. In other words, $\forall t \in F_{\delta}$, $t \in A_{\theta}$. 
\end{proof}

Recall notation for Theorem \ref{thm:opt_tree_rm_instance} and \ref{thm:rset_rm_instance}. Let $\tilde{t}^*$ be the optimal tree trained on  $\{\x_{[\backslash K, \cdot]}, \y_{[\backslash K]}\}$ where $K$ is a set of indices of instances that we wish to analyze. We denote $|K|$ as the cardinality of the set $K$. 
Overloading notation to include the dataset, let $ R_{set}(\epsilon, t^*, \mathcal{T},\x,\y) = R_{set}(\epsilon, t^*, \mathcal{T})$ (see Eq \eqref{def:eps-rashomon}) be the Rashomon set of the original dataset, 
where $t_{\textrm{ref}}=t^*$ is the optimal tree trained on the original dataset, and we define the $\epsilon'$-Rashomon set on the reduced dataset as $$R_{set}(\epsilon', \tilde{t}^*,\mathcal{T}, \x_{[\backslash K, \cdot]}, \y_{[\backslash K]})\h2:=\h2\left\{t \in \mathcal{T}: Obj(t, \x_{[\backslash K, \cdot]}, \y_{[\backslash K]}) \h2\leq\h2 (1+\epsilon')\h2\times\h2 Obj(\tilde{t}^*, \x_{[\backslash K, \cdot]}, \y_{[\backslash K]})\right\}.$$

\textbf{Theorem} \ref{thm:opt_tree_rm_instance} \textit{(Optimal tree after removing a group of instances is still in full-dataset Rashomon set)
If $\epsilon \geq \frac{2|K|}{n\times Obj(t^*,\x,\y)}, \tilde{t}^* \in R_{set}(\epsilon, t^*, \mathcal{T}, \x,\y)$. } 

\begin{proof}
The objective of $t^*$ on the original dataset and the reduced dataset are $$Obj(t^*, \x, \y) = \frac{1}{n}\sum_{i=1}^n \mathbbm{1}[y_i\neq \hat{y}_i^{t^*}] + \lambda H_{t^*} \ \  \textrm{and} \ \  Obj(t^*, \x_{[\backslash K,\cdot]}, \y_{[\backslash K]} )=\frac{1}{n-|K|} \sum_{i\notin K}\mathbbm{1}[y_i\neq \hat{y}_i^{t^*}]+ \lambda H_{t^*}.$$ 
Similarly, the objective of $\tilde{t}^*$ on the original dataset and the reduced dataset are
$$Obj(\tilde{t}^*, \x,\y) = \frac{1}{n}\sum_{i=1}^n \mathbbm{1}[y_i \neq \hat{y}_i^{\tilde{t}^*}] + \lambda H_{\tilde{t}^*} \ \  \textrm{and} \ \  Obj(\tilde{t}^*, \x_{[\backslash K,\cdot]}, \y_{[\backslash K]})=\frac{1}{n-|K|} \sum_{i\notin K}\mathbbm{1}[y_i\neq \hat{y}_i^{\tilde{t}^*}]+ \lambda H_{\tilde{t}^*}.$$

Since $\tilde{t}^*$ is the optimal tree trained on $\{\x_{[\backslash K, \cdot]}, \y_{[\backslash K]}\}$, 
\begin{equation}\label{eq:opt_tree_rm_instance_step0}
    Obj(\tilde{t}^*, \x_{[\backslash K, \cdot]}, \y_{[\backslash K]}) \leq Obj(t^*, \x_{[\backslash K, \cdot]}, \y_{[\backslash K]}). 
\end{equation}

Step 1: bound the difference between $Obj(t^*, \x, \y)$ and $Obj(t^*, \x_{[\backslash K, \cdot]}, \y_{[\backslash K]})$

Since $\sum_{i=1}^n \mathbbm{1}[y_i \neq \hat{y}_i^{t^*}] = \sum_{i\notin K}  \mathbbm{1}[y_i \neq \hat{y}_i^{t^*}] + \sum_{i \in K} \mathbbm{1}[y_i \neq \hat{y}_i^{t^*}]$, we can get 
\begin{equation}\label{eq:opt_tree_rm_instance_step1}
    \begin{aligned}
    Obj(t^*, \x, \y) - Obj(t^*, \x_{[\backslash K, \cdot]}, \y_{[\backslash K]}) &= \frac{1}{n} \sum_{i=1}^n\mathbbm{1}[y_i \neq \hat{y}_i^{t^*}] - \frac{1}{n-|K|} \sum_{i\notin K}\mathbbm{1}[y_i \neq \hat{y}_i^{t^*}]\\
    &= \frac{1}{n} (\sum_{i\notin K}  \mathbbm{1}[y_i \neq \hat{y}_i^{t^*}] + \sum_{i \in K} \mathbbm{1}[y_i \neq \hat{y}_i^{t^*}]) - \frac{1}{n-|K|} \sum_{i\notin K}\mathbbm{1}[y_i \neq \hat{y}_i^{t^*}]\\
    &=\frac{(n-|K|)\sum_{i\notin K}  \mathbbm{1}[y_i \neq \hat{y}_i^{t^*}] + (n-|K|) \sum_{i \in K}\mathbbm{1}[y_i \neq \hat{y}_i^{t^*}]}{n(n-|K|)} \\
    & \quad - \frac{n \sum_{i\notin K}  \mathbbm{1}[y_i \neq \hat{y}_i^{t^*}]}{n(n-|K|)}\\
    &=\frac{(n-|K|)\sum_{i \in K} \mathbbm{1}[y_i \neq \hat{y}_i^{t^*}] - |K|\times \sum_{i\notin K}  \mathbbm{1}[y_i \neq \hat{y}_i^{t^*}]}{n(n-|K|)}.
    \end{aligned}
\end{equation}
Since $\sum_{i \in K}\mathbbm{1}[y_i \neq \hat{y}_i^{t^*}] \in \{0,1,\cdots,|K|\}$, $$(n-|K|) \sum_{i \in K}\mathbbm{1}[y_i \neq \hat{y}_i^{t^*}] \in \{0,n-|K|,\cdots, (n-|K|)|K|\}.$$ 

Similarly, $\sum_{i\notin K}\mathbbm{1}[y_i \neq \hat{y}_i^{t^*}] \in \{0,1,\cdots, n-|K|\}$, therefore, $$|K|\sum_{i\notin K}\mathbbm{1}[y_i \neq \hat{y}_i^{t^*}] \in \{0,|K|,\cdots, |K|(n-|K|)\}.$$ 

Combining the extreme values of these two terms together, we know Eq \eqref{eq:opt_tree_rm_instance_step1} is
\begin{equation}\label{eq:opt_tree_rm_instance_step1_end}
    Obj(t^*, \x, \y) - Obj(t^*, \x_{[\backslash K, \cdot]}, \y_{[\backslash K]}) \geq \frac{-|K|}{n}. 
\end{equation}

Step 2: bound the difference between $Obj(\tilde{t}^*, \x_{[\backslash K, \cdot]}, \y_{[\backslash K]})$ and $Obj(\tilde{t}^*, \x, \y)$. 
\begin{equation}\label{eq:opt_tree_rm_instance_step2}
    \begin{aligned}
    Obj(\tilde{t}^*, \x_{[\backslash K, \cdot]}, \y_{[\backslash K]}) - Obj(\tilde{t}^*, \x, \y) &= \frac{1}{n-|K|} \sum_{i\notin K}\mathbbm{1}[y_i \neq \hat{y}_i^{\tilde{t}^*}] - \frac{1}{n} \sum_{i=1}^n\mathbbm{1}[y_i \neq \hat{y}_i^{\tilde{t}^*}]\\
    &= \frac{1}{n-|K|} \sum_{i\notin K}\mathbbm{1}[y_i \neq \hat{y}_i^{\tilde{t}^*}] - \frac{1}{n} \left[\sum_{i\in K}\mathbbm{1}[y_i \neq \hat{y}_i^{\tilde{t}^*}] + \sum_{i\notin K}\mathbbm{1}[y_i \neq \hat{y}_i^{\tilde{t}^*}]\right]\\
    &= \frac{n \sum_{i\notin K}\mathbbm{1}[y_i \neq \hat{y}_i^{\tilde{t}^*}]}{(n-|K|)n}  \\
    &\quad - \frac{(n-|K|)(\sum_{i\in K}\mathbbm{1}[y_i \neq \hat{y}_i^{\tilde{t}^*}] + \sum_{i\notin K}\mathbbm{1}[y_i \neq \hat{y}_i^{\tilde{t}^*}])}{n(n-|K|)} \\
    &= \frac{|K|\sum_{i\notin K}\mathbbm{1}[y_i \neq \hat{y}_i^{\tilde{t}^*}] - (n-|K|)\sum_{i\in K}\mathbbm{1}[y_i \neq \hat{y}_i^{\tilde{t}^*}]}{n(n-|K|)}.
    \end{aligned}
\end{equation}
Since $\sum_{i\notin K}\mathbbm{1}[y_i \neq \hat{y}_i^{\tilde{t}^*}]) \in \{0,1,\cdots, n-|K|\}$, $$|K|\sum_{i\notin K}\mathbbm{1}[y_i \neq \hat{y}_i^{\tilde{t}^*}]) \in \{0,|K|,\cdots,|K|(n-|K|)\}.$$
Similarly, $\sum_{i\in K}\mathbbm{1}[y_i \neq \hat{y}_i^{\tilde{t}^*}]) \in \{0,1,\cdots, |K|\}$, $$(n-|K|)\sum_{i\in K}\mathbbm{1}[y_i \neq \hat{y}_i^{\tilde{t}^*}]) \in \{0,n-|K|,\cdots, |K|(n-|K|)\}.$$

Combining the extreme values of these two terms, we know Eq \eqref{eq:opt_tree_rm_instance_step2} obeys 
\begin{equation}\label{eq:opt_tree_rm_instance_step2_end}
    Obj(\tilde{t}^*, \x_{[\backslash K, \cdot]}, \y_{[\backslash K]}) - Obj(\tilde{t}^*, \x, \y) \geq \frac{-|K|}{n}.
\end{equation}

Given Eq \eqref{eq:opt_tree_rm_instance_step0}, \eqref{eq:opt_tree_rm_instance_step1_end}, \eqref{eq:opt_tree_rm_instance_step2_end}, we can get 
\begin{equation}
    \begin{aligned}
    Obj(t^*, \x, \y) &\geq Obj(t^*, \x_{[\backslash K, \cdot]}, \y_{[\backslash K]}) + \frac{-|K|}{n} \quad \textrm{(see Eq \eqref{eq:opt_tree_rm_instance_step1_end})}\\
    &\geq Obj(\tilde{t}^*, \x_{[\backslash K, \cdot]}, \y_{[\backslash K]}) + \frac{-|K|}{n} \quad \textrm{(see Eq \eqref{eq:opt_tree_rm_instance_step0})}\\
    &\geq Obj(\tilde{t}^*, \x, \y) + \frac{-2|K|}{n} \quad \textrm{(see Eq \eqref{eq:opt_tree_rm_instance_step2_end})}.
    \end{aligned}
\end{equation}
In other words, $Obj(\tilde{t}^*, \x,\y)\leq Obj(t^*, \x,\y) + \frac{2|K|}{n}$. To guarantee that $Rset(\epsilon, t^*, \mathcal{T}, \x, \y)$ covers $\tilde{t}^*$, i.e., $Obj(t^*, \x,\y) + \frac{2|K|}{n} \leq (1+\epsilon) \times Obj(t^*, \x,\y)$, we impose that
$\epsilon \geq \frac{2|K|}{n \times Obj(t^*, \x,\y)}$. 
\end{proof}

\textbf{Theorem} \ref{thm:rset_rm_instance}  \textit{(Rashomon set after removing a group of instances is within full-dataset Rashomon set)
If $\epsilon \geq \epsilon' + \frac{(2+\epsilon')|K|}{n \times Obj(t^*, \x, \y)}$, then $\forall t \in R_{\set}(\epsilon', \tilde{t}^*, \mathcal{T}, \x_{[\backslash K, \cdot]}, \y_{[\backslash K]})$, we have $t\in R_{\set}(\epsilon, t^*, \mathcal{T}, \x, \y)$.  } 

\begin{proof}
Given a decision tree $t \in R_{set}(\epsilon', \tilde{t}^*,\mathcal{T}, \x_{[\backslash K, \cdot]}, \y_{[\backslash K]})$, we know 
\begin{equation}\label{eq:rset_rm_instance_step1}
\begin{aligned}
    Obj(t, \x_{[\backslash K, \cdot]}, \y_{[\backslash K]}) &\leq (1+\epsilon')\times Obj(\tilde{t}^*, \x_{[\backslash K, \cdot]}, \y_{[\backslash K]}) \\
    & \leq (1+\epsilon')\times Obj(t^*, \x_{[\backslash K, \cdot]}, \y_{[\backslash K]}) \ \ (\tilde{t}^* \textrm{is the optimal tree on } \{\x_{[\backslash K, \cdot]}, \y_{[\backslash K]}\})\\
    &\leq (1+\epsilon')\times \left[ Obj(t^*, \x, \y) + \frac{|K|}{n}\right] \ \ (\textrm{see Eq \eqref{eq:opt_tree_rm_instance_step1_end}}).\\
\end{aligned}
\end{equation}
Then we bound the difference between $Obj(t, \x_{[\backslash K, \cdot]}, \y_{[\backslash K]})$ and $Obj(t, \x, \y)$. 
\begin{equation}\label{eq:rset_rm_instance_step2}
    \begin{aligned}
    Obj(t, \x_{[\backslash K, \cdot]}, \y_{[\backslash K]}) - Obj(t, \x, \y) &= \frac{1}{n-|K|} \sum_{i\notin K}\mathbbm{1}[y_i \neq \hat{y}_i^{t}] - \frac{1}{n} \sum_{i=1}^n\mathbbm{1}[y_i \neq \hat{y}_i^{t}]\\
    &= \frac{1}{n-|K|} \sum_{i\notin K}\mathbbm{1}[y_i \neq \hat{y}_i^{t}] - \frac{1}{n} \left[\sum_{i\in K}\mathbbm{1}[y_i \neq \hat{y}_i^{t}] + \sum_{i\notin K}\mathbbm{1}[y_i \neq \hat{y}_i^{t}]\right]\\
    &= \frac{n \sum_{i\notin K}\mathbbm{1}[y_i \neq \hat{y}_i^{t}] - (n-|K|)\sum_{i\notin K}\mathbbm{1}[y_i \neq \hat{y}_i^{t}]}{(n-|K|)n}  \\
    &\quad - \frac{(n-|K|)(\sum_{i\in K}\mathbbm{1}[y_i \neq \hat{y}_i^{t}])}{n(n-|K|)} \\
    &= \frac{|K|\sum_{i\notin K}\mathbbm{1}[y_i \neq \hat{y}_i^{t}] - (n-|K|)\sum_{i\in K}\mathbbm{1}[y_i \neq \hat{y}_i^{t}]}{n(n-|K|)}.
    \end{aligned}
\end{equation}

Since $\sum_{i\notin K}\mathbbm{1}[y_i \neq \hat{y}_i^{t}] \in \{0,1,\cdots, n-|K|\}$, $$|K|\sum_{i\notin K}\mathbbm{1}[y_i \neq \hat{y}_i^t] \in \{0,|K|,\cdots, |K|(n-|K|)\}.$$
Similarly, $\sum_{i\in K}\mathbbm{1}[y_i \neq \hat{y}_i^{t}]) \in \{0,1,\cdots,|K|\}$, $$(n-|K|)\sum_{i\in K}\mathbbm{1}[y_i \neq \hat{y}_i^{t}]) \in \{0,n-|K|,\cdots, |K|(n-|K|)\}.$$
Combining the extreme values of these two terms we know Eq \eqref{eq:rset_rm_instance_step2} obeys 
\begin{equation}\label{eq:rset_rm_instance_step2_end}
    Obj(t, \x_{[\backslash K, \cdot]}, \y_{[\backslash K]}) - Obj(t, \x, \y) \geq \frac{-|K|}{n}.
\end{equation}
Combining Eq \eqref{eq:rset_rm_instance_step1} and \eqref{eq:rset_rm_instance_step2_end}, we get 
\begin{equation}\label{eq:rset_rm_instance_step3}
\begin{aligned}
    Obj(t, \x, \y) &\leq Obj(t, \x_{[\backslash K, \cdot]}, \y_{[\backslash K]}) + \frac{|K|}{n} \ \ (\textrm{see Eq \ref{eq:rset_rm_instance_step2_end}}) \\
    &\leq (1+\epsilon')\times \left[Obj(t^*, \x, \y) + \frac{|K|}{n}\right]+\frac{|K|}{n} \ \ (\textrm{see Eq \eqref{eq:rset_rm_instance_step1}}) \\
    &=(1+\epsilon')\times Obj(t^*, \x, \y) + \frac{(2+\epsilon')|K|}{n}.
\end{aligned}
\end{equation}
Thus, when $\epsilon \geq \epsilon' + \frac{(2+\epsilon')|K|}{n\times Obj(t^*, \x, \y)}$, Eq \eqref{eq:rset_rm_instance_step3} extends to 
\begin{equation}
\begin{aligned}
    Obj(t, \x, \y) & \leq (1+\epsilon')\times Obj(t^*, \x, \y) + \frac{(2+\epsilon')|K|}{n}\\
    & \leq (1+\epsilon')\times Obj(t^*, \x, \y) + (\epsilon - \epsilon')\times Obj(t^*, \x, \y)\\
    &= (1+\epsilon)\times Obj(t^*, \x, \y).
\end{aligned}
\end{equation}

In other words, when $\epsilon \geq \epsilon' + \frac{(2+\epsilon')|K|}{n\times Obj(t^*, \x, \y)}$,
$\forall t \in R_{set}(\epsilon', \tilde{t}^*, \mathcal{T}, \x_{[\backslash K, \cdot]}, \y_{[\backslash K]}), t \in R_{set}(\epsilon, t^*, \mathcal{T}, \x, \y)$.

\end{proof}

\section{Experimental Datasets}
\label{app:datasets}

We present results for 14 datasets: four from the UCI Machine Learning Repository \citepAppendix{appDua:2019} (Car Evaluation, Congressional Voting Records, Monk2, Iris, and Breast Cancer), a penguin dataset \citepAppendix{appgorman2014ecological}, a recidivism dataset (COMPAS) \citepAppendix{appLarsonMaKiAn16}, the Fair Isaac (FICO) credit risk dataset \citepAppendix{appcompetition} used for the Explainable ML Challenge, and four coupon datasets (Bar, Coffee House, Takeaway Food, Cheap Restaurant, and Expensive Restaurant), which were collected on Amazon Mechanical Turk via a survey \citepAppendix{appwang2017bayesian}. We predict which individuals are arrested within two years of release on the COMPAS dataset, whether an individual will default on a loan for the FICO dataset, and whether a customer will accept a coupon for takeaway food or a cheap restaurant depending on their coupon usage history, current conditions while driving, and coupon expiration time on two coupon datasets. All the datasets are publicly available. 

\subsection{Preprocessing}

Table \ref{tab:datasets_details} summarizes all the datasets after preprocessing.  

\begin{table}[ht]
    \centering
    \begin{tabular}{|c|c|c|}\hline
       Dataset & Samples & Binary Features \\\hline\hline
      Car Evaluation & 1728 & 15 \\\hline
      Congressional Voting Records & 435 & 32  \\\hline 
      Monk2 & 169 & 11 \\\hline
      Penguin & 333 & 13 \\\hline
      Iris & 151 & 15 \\\hline
      Breast Cancer & 699 & 10 \\\hline
      COMPAS & 6907 & 12 \\\hline
      FICO & 10459 & 17 \\\hline
      Bar-7 & 1913 & 14 \\\hline 
      Bar & 1913 & 15 \\\hline
      Coffee House & 3816 & 15 \\\hline
      Takeaway Food & 2280 & 15\\\hline
      Cheap Restaurant & 2653 & 15 \\\hline
      Expensive Restaurant & 1417 & 15  \\\hline
    \end{tabular}
    \caption{Preprocessed datasets}
    \label{tab:datasets_details}
\end{table}

\textbf{Car Evaluation, Congressional Voting Records, Monk2, Penguin}: We preprocess these datasets, which contain only categorical features, using one-hot encoding. 

\textbf{Iris}: We select thresholds produced by splitting each numerical feature into three equal parts. We used an implementation of qcut in Pandas \citepAppendix{appreback2020pandas} to split features. This yields 15 features. 

\textbf{Breast Cancer}: We select features and thresholds that are used by a gradient boosted tree with 40 decision stumps, which are ``Clump\_Thickness $=$ 10'', ``Uniformity\_Cell\_Size$=$1'', ``Uniformity\_Cell\_Size$=$10'', ``Uniformity\_Cell\_Shape$=$1'', ``Marginal\_Adhesion$=$1'', ``Single\_Epithelial\_Cell\_Size$=$2'', ``Bare\_Nuclei$=$1'', ``Bare\_Nuclei$=$10'', ``Normal\_Nucleoli$=$1'', ``Normal\_Nucleoli$=$10''. 

\textbf{COMPAS}: We use the same discretized binary features of COMPAS produced in \citepAppendix{apphu2019optimal}, which are the following: ``sex $=$ Female'', ``age $<$ 21'', ``age $<$ 23'', age $<$ 26, ``age $<$ 46'', ``juvenile felonies $=$ 0'', ``juvenile misdemeanors $=$ 0'', ``juvenile crimes $=$ 0'', ``priors $=$ 0'', ``priors $=$ 1'', ``priors $=$ 2 to 3'', ``priors $>$ 3''.

\textbf{FICO}: We use the same discretized binary features of FICO produced in \citepAppendix{apphu2019optimal}, which are the following: ``External Risk Estimate $<$ 0.49'' , ``External Risk Estimate $<$ 0.65'', ``External Risk Estimate $<$
0.80'', ``Number of Satisfactory Trades $<$ 0.5'', ``Trade Open Time $<$ 0.6'', ``Trade Open Time $<$ 0.85'', ``Trade Frequency $<$ 0.45'', ``Trade Frequency $<$ 0.6'', ``Delinquency $<$ 0.55'', ``Delinquency $<$ 0.75'', ``Installment $<$ 0.5'', ``Installment $<$ 0.7'', ``Inquiry $<$ 0.75'', ``Revolving Balance $<$ 0.4'', ``Revolving Balance $<$ 0.6'', ``Utilization $<$ 0.6'', ``Trade W. Balance $<$ 0.33''.

\textbf{Bar, Coffee House, Takeaway Food, Cheap Restaurant, Expensive Restaurant}: We selected features ``destination'', ``passanger'', ``weather'', ``temperature'', ``time'', ``expiration'', ``gender'', ``age'', ``maritalStatus'', ``childrenNumber'', ``education'', ``occupation'', ``income'', ``Bar'', ``CoffeeHouse'', ``CarryAway'', ``RestaurantLessThan20'', ``Restaurant20To50'', ``toCouponGEQ15min'', ``toCouponGEQ25min'', ``directionSame'' and the label Y, and removed observations with missing values. We used one-hot encoding to transform these categorical features into binary features. We then selected 15 binary features with the highest variable importance value trained using gradient boosted trees with 100 max-depth 3 weak classifiers.

\begin{itemize}
    \item \textbf{Bar}: The selected binary features are ``Bar $=$ 1 to 3'', ``Bar $=$ 4 to 8'', ``Bar $=$ less1'', ``maritalStatus $=$ Single'', ``childrenNumber $=$ 0'', ``Bar $=$ gt8'', ``passanger $=$ Friend(s)'', ``time $=$ 6PM'', ``passanger $=$ Kid(s)'', ``CarryAway $=$ 4 to 8'', ``gender $=$ Female'',	``education $=$ Graduate degree (Masters Doctorate etc.)'', ``Restaurant20To50 $=$ 4 to 8'', ``expiration $=$ 1d'', ``temperature $=$ 55''.
    \item \textbf{Coffee House}: The selected binary features are 
    ``CoffeeHouse $=$ 1 to 3'', ``CoffeeHouse $=$ 4 to 8'', ``CoffeeHouse $=$ gt8'', ``CoffeeHouse $=$ less1'', ``expiration $=$ 1d'', ``destination $=$ No Urgent Place'', ``time $=$ 10AM'', ``direction $=$ same'',	``destination $=$ Home'',	``toCoupon $=$ GEQ15min'', ``Restaurant20To50 $=$ gt8'', ``education $=$ Bachelors degree'', ``time $=$ 10PM'', ``income $=$ \$75000 - \$87499'', ``passanger $=$ Friend(s)''.
    
    \item \textbf{Takeaway Food}:
    The selected binary features are ``expiration $=$ 1d'', ``time $=$ 6PM'', ``CoffeeHouse $=$ gt8'', ``education$=$ Graduate degree'', ``weather $=$ Rainy'', ``maritalStatus $=$ Single'', ``time $=$ 2PM'', ``occupation $=$ Student'', ``income $=$\$62500-\$74999'', ``occupation $=$ Legal'', ``occupation $=$ Installation  Maintenance \& Repair'', ``direction $=$ same'', ``destination $=$ No Urgent Place'', ``income $=$ \$100000 or more'', ``Bar $=$ less1''. 
    
    \item \textbf{Cheap Restaurant}: 
    The selected binary features are ``toCoupon $=$ GEQ25min'', ``expiration $=$ 1d'', ``time $=$ 6PM'', ``destination $=$ No Urgent Place'', ``time $=$ 10PM'', ``CarryAway $=$ less1'', ``passanger $=$ Kid(s)'', ``weather $=$ Snowy'', ``passanger $=$ Alone'', ``income $=$ \$87500 - \$99999'',	``occupation $=$ Retired'', ``CoffeeHouse $=$ gt8'', ``age $=$ 36'', ``weather $=$ Rainy'', ``direction $=$ same''.
    
    \item \textbf{Expensive Restaurant}: The selected binary features are ``expiration $=$ 1d'', ``CoffeeHouse $=$ 1 to 3'', ``Restaurant20To50 $=$ 4 to 8'', ``Restaurant20To50 $=$ 1 to 3'', ``occupation $=$ Office \& Administrative Support'', ``age $=$ 31'', ``Restaurant20To50 $=$ gt8'', ``income $=$\$12500 - \$24999'', ``toCoupon $=$ GEQ15min'', ``occupation $=$ Computer \& Mathematical'', ``time $=$ 10PM'', ``CoffeeHouse $=$ 4 to 8'', ``income$=$\$50000 - \$62499'', ``passanger $=$ Alone'',	``destination $=$ No Urgent Place''.

\end{itemize}

\textbf{Bar-7}: We use the same discretized binary features of Bar-7 produced in \citepAppendix{applin2020generalized}, which are the following: ``passanger $=$ Kids'', ``age $=$ 21'',	``age $=$ 26'',``age $=$ 31'',``age $=$ 36'', ``age $=$ 41'',``age $=$ 46'', ``age $=$ 26'', ``age $=$ 50plus'', ``Bar $=$ 1 to 3'', ``Bar $=$ 4 to 8'', ``Bar $=$ gt8'', ``Bar $=$ less1'', ``Restaurant20to50 $\geq$4, ``direction $=$ same''.

\section{More Experimental Results}\label{app:more_exp_results}

\subsection{Scalability and Efficiency of Calculating Rashomon set with \ourmethod{} versus Baselines}\label{app:exp_scalability}
\textbf{Collection and Setup:} We ran this experiment on 10 datasets: \textbf{Monk2, COMPAS, Car Evaluation, FICO, Congressional Voting Records, Bar, Bar7, Coffee House, Cheap Restaurant, Expensive Restaurant}. To run \ourmethod{}, we set $\lambda$ to 0.01 for Monk2, COMPAS, FICO, Congressional Voting Records, Bar, Bar7, Coffee House and Expensive Restaurant and to 0.005 for Car Evaluation and Cheap Restaurant. These choices yielded sufficiently larger Rashomon sets.
We set $\epsilon$ to 0.15 for Congressional Voting Records and 0.10 for all other datasets. This means we store all models within 10\% of the optimal solution.

We used the R package BART \citepAppendix{appbart_package} and set the number of trees in each iteration to 1. We sampled models from the posterior with 10-iteration intervals between draws. To get sparser models, we used a sparse Dirichlet prior. To get a set of trees from Random Forest and CART, we used RandomForestClassifier from scikit-learn \nocite{scikit-learn}\citepAppendix{appscikit-learn} and set min\_samples\_split to $\max(\lceil 2n\lambda\rceil, 2)$, min\_samples\_leaf to $\lceil n\lambda\rceil$, and max\_leaf\_nodes=$\lfloor 1 / (2 \lambda)\rfloor$. We set max\_features to ``auto'' for Random Forest and ``None'' for CART. For GOSDT, we sampled 75\% of the dataset size, with replacement to improve diversity, and fit an optimal tree on the sampled data. To improve performance for BART, Random Forest, and CART with sampling, we collapsed trivial splits (nodes that have two leaves with the same predicted label) into a single node. We record the time used to construct the Rashomon set and ran the baselines for a similar time to sample trees. Thus, all methods were given the same amount of time to produce good trees. For each dataset and baseline method, we used 5 different random seeds to compute the average and standard deviation of number of trees found in Rashomon set and run time. We did not report errors on \ourmethod{} as it is deterministic. We used results obtained from the first random seed to generate Figure \ref{fig:more-vs-baselines}.

We ran this experiment on a 2.7Ghz (768GB RAM 48 cores) Intel Xeon Gold 6226 processor. We set a 200-GB memory limit.

\textbf{Results:} Figure \ref{fig:more-vs-baselines} compares the Rashomon set with the four baselines on the Car Evaluation, Coffee House, Cheap and Expensive Restaurant datasets. Similar to the patterns shown in Figure \ref{fig:comp_baselines}, the four subfigures on the top show that \ourmethod{} (in purple) found \textbf{\textit{many more distinct trees in the Rashomon set}} than any of the four baselines on all of the datasets. The baseline methods tend to find many duplicated trees, and most trees found by the baselines have objective values higher than the threshold of the Rashomon set, and thus, are not in the Rashomon set. The four subfigures on the bottom show the distribution of objective values of trees from each method. 
We note that in the method GOSDT with sampling the optimal tree for a given subsample of the dataset might not be identical to the optimal tree to the original dataset. 
Thus, GOSDT, which finds provably optimal trees for a given dataset, does not obtain many distinct trees that are in the Rashomon set: instead, many subsamples produce either the same optimal solutions or trees that are outside of the full dataset's Rashomon set, which is why its curve is not visible. 
\ourmethod{} is the only method \textit{guaranteed} to find all trees in the Rashomon set.

Table \ref{tab:more_time_baselines} shows that \ourmethod{} finds \textit{\textbf{more trees in the Rashomon set per second}} than all other baselines on all the datasets. 

\begin{figure}[ht]
    \centering
    \includegraphics[height=0.187\textwidth]{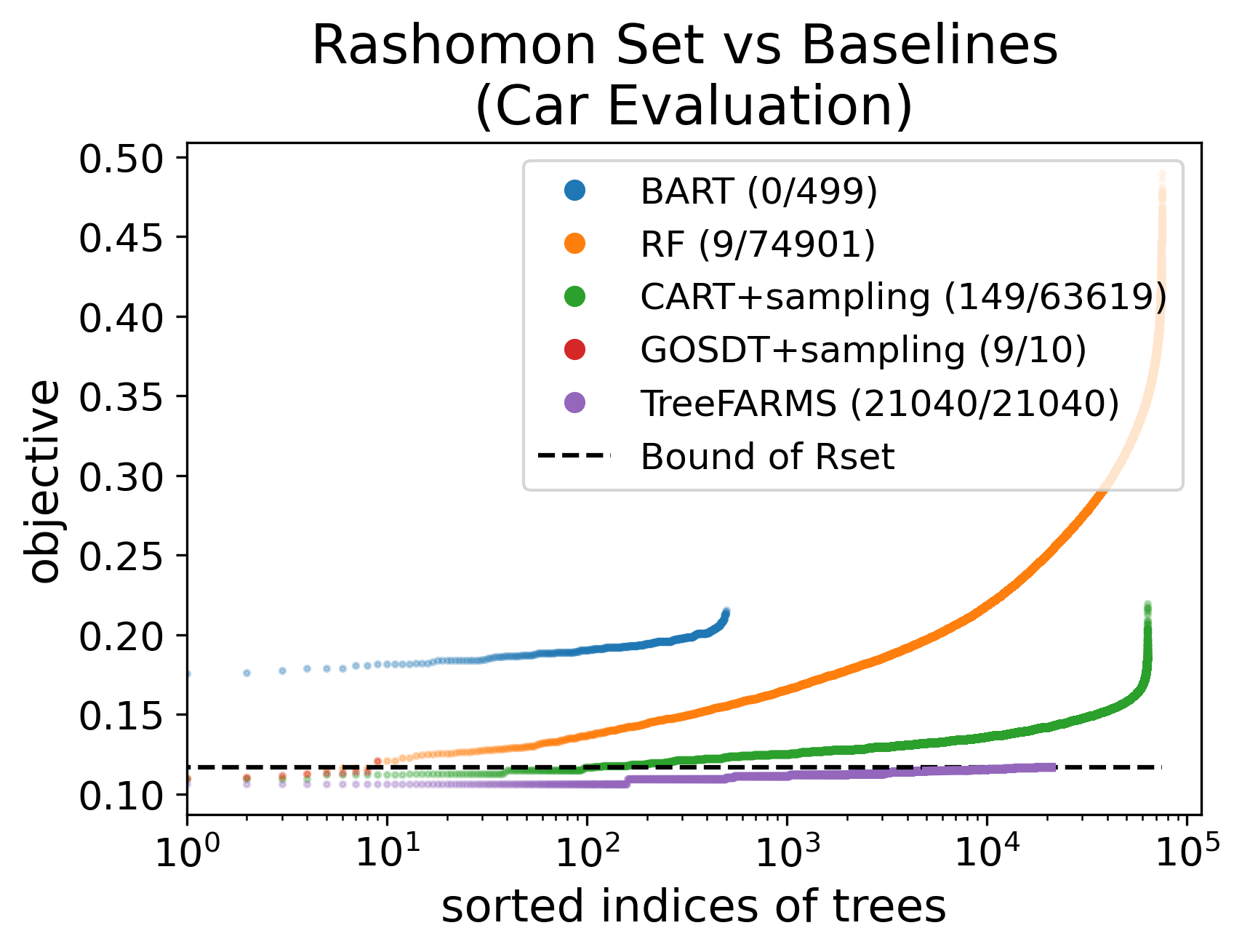}
    \includegraphics[height=0.186\textwidth]{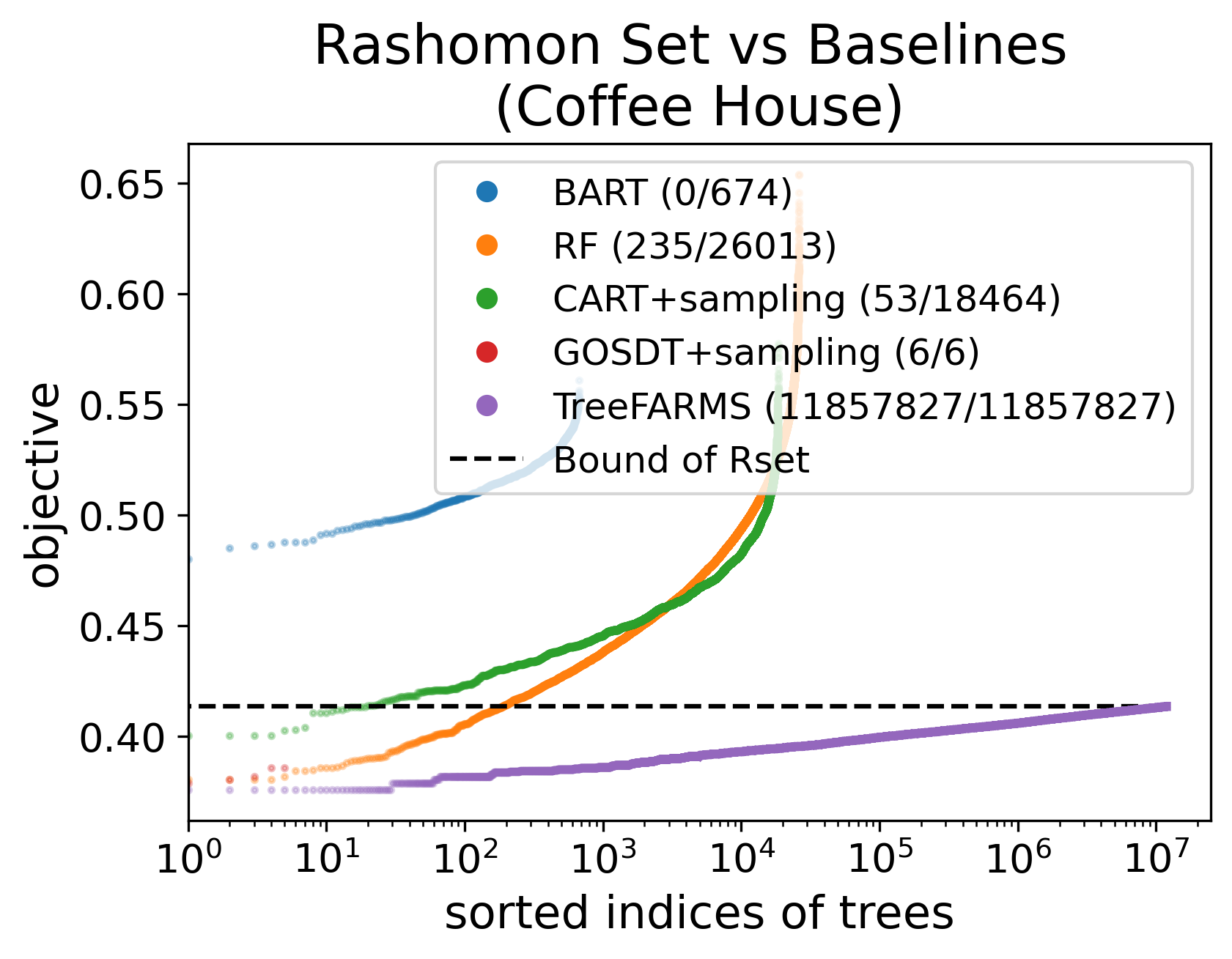}
    \includegraphics[height=0.187\textwidth]{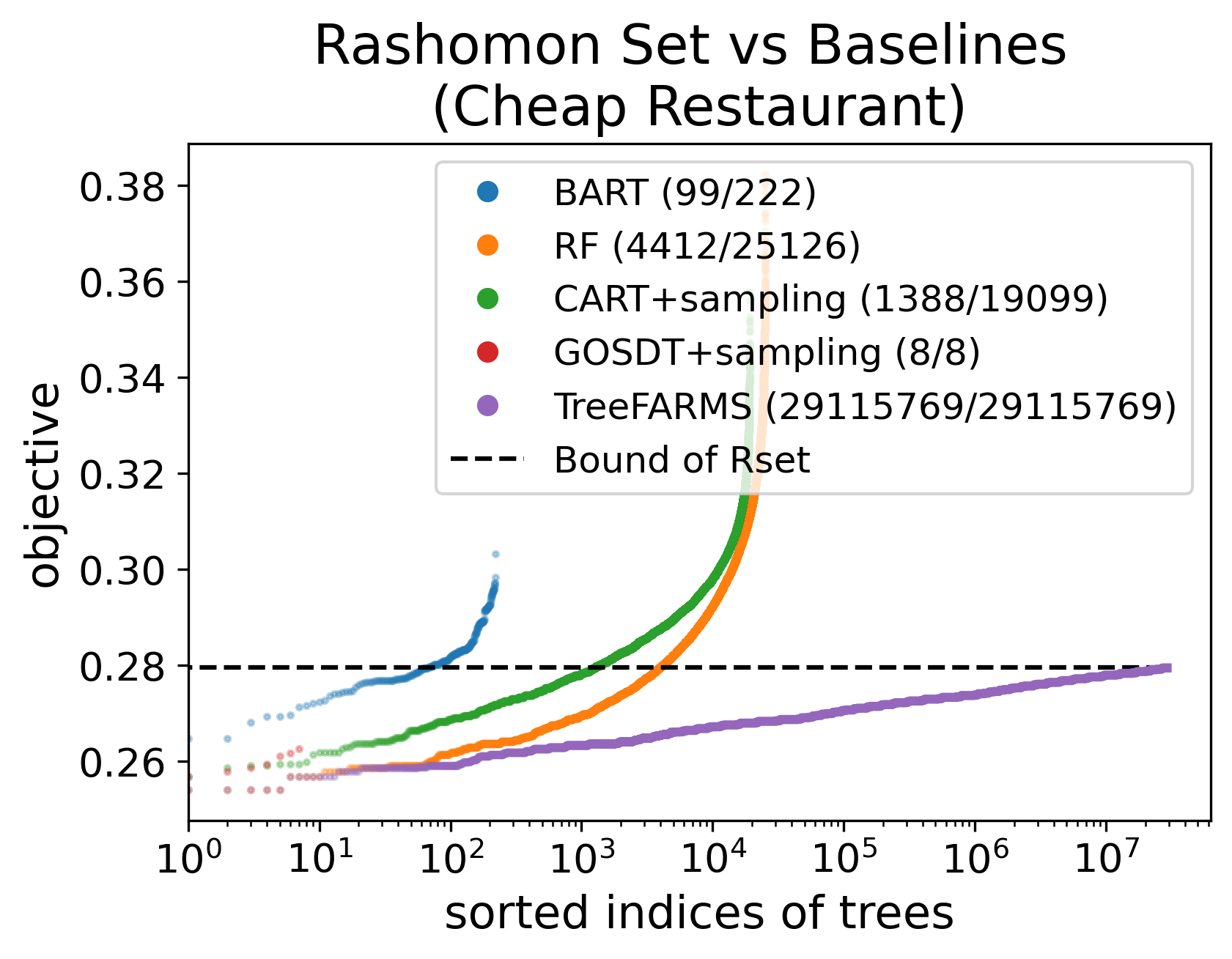}
    \includegraphics[height=0.187\textwidth]{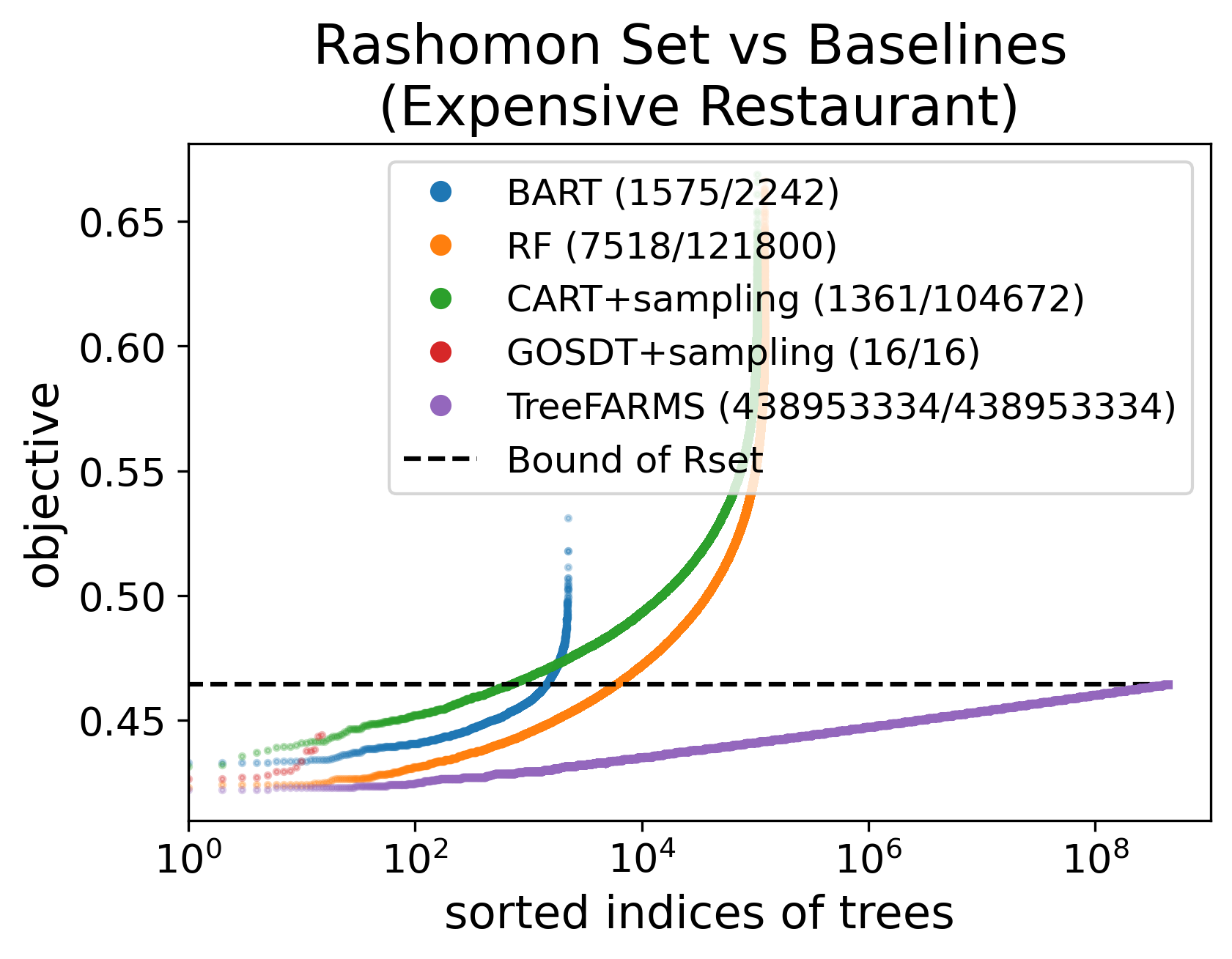}
    \includegraphics[height=0.187\textwidth]{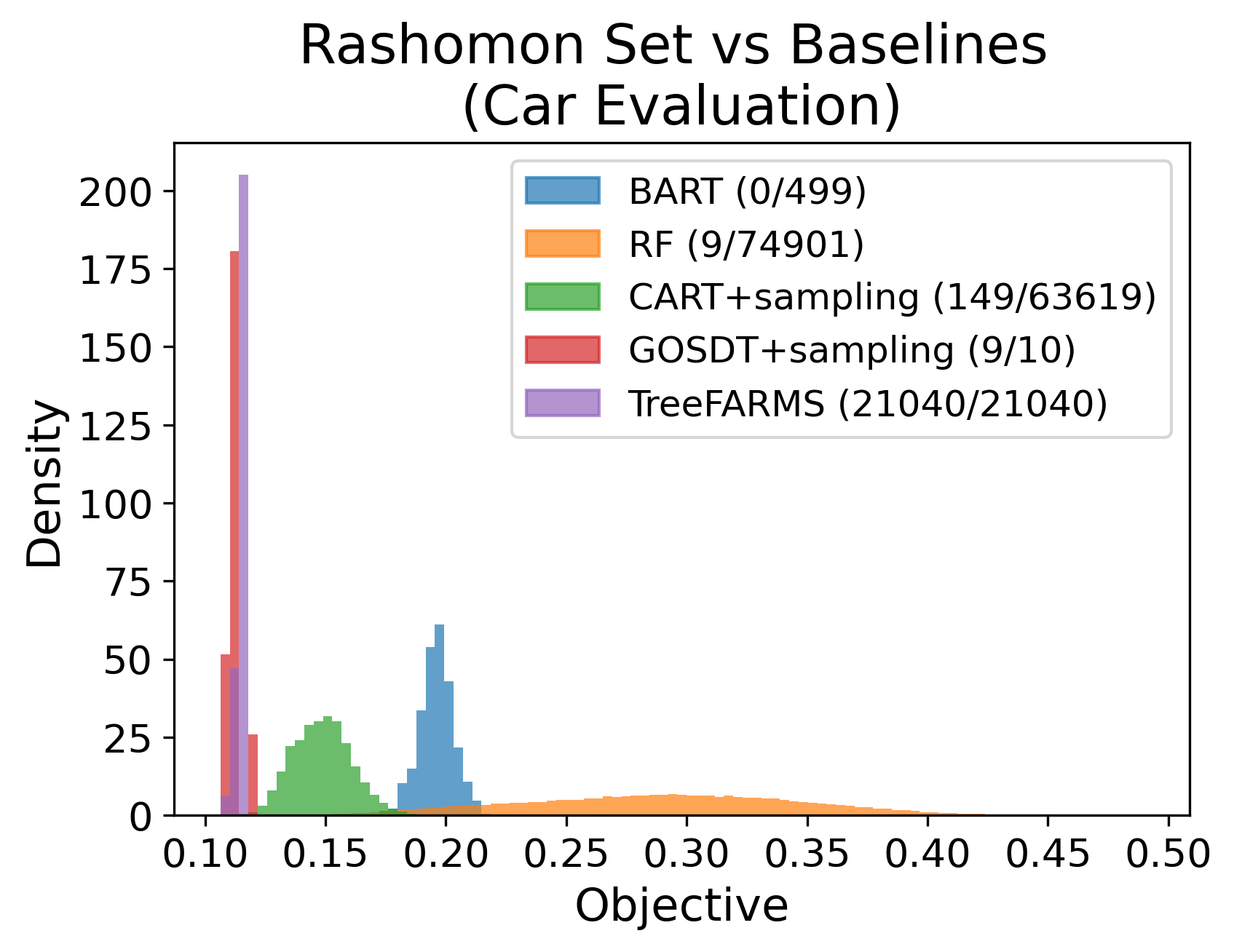}
    \includegraphics[height=0.187\textwidth]{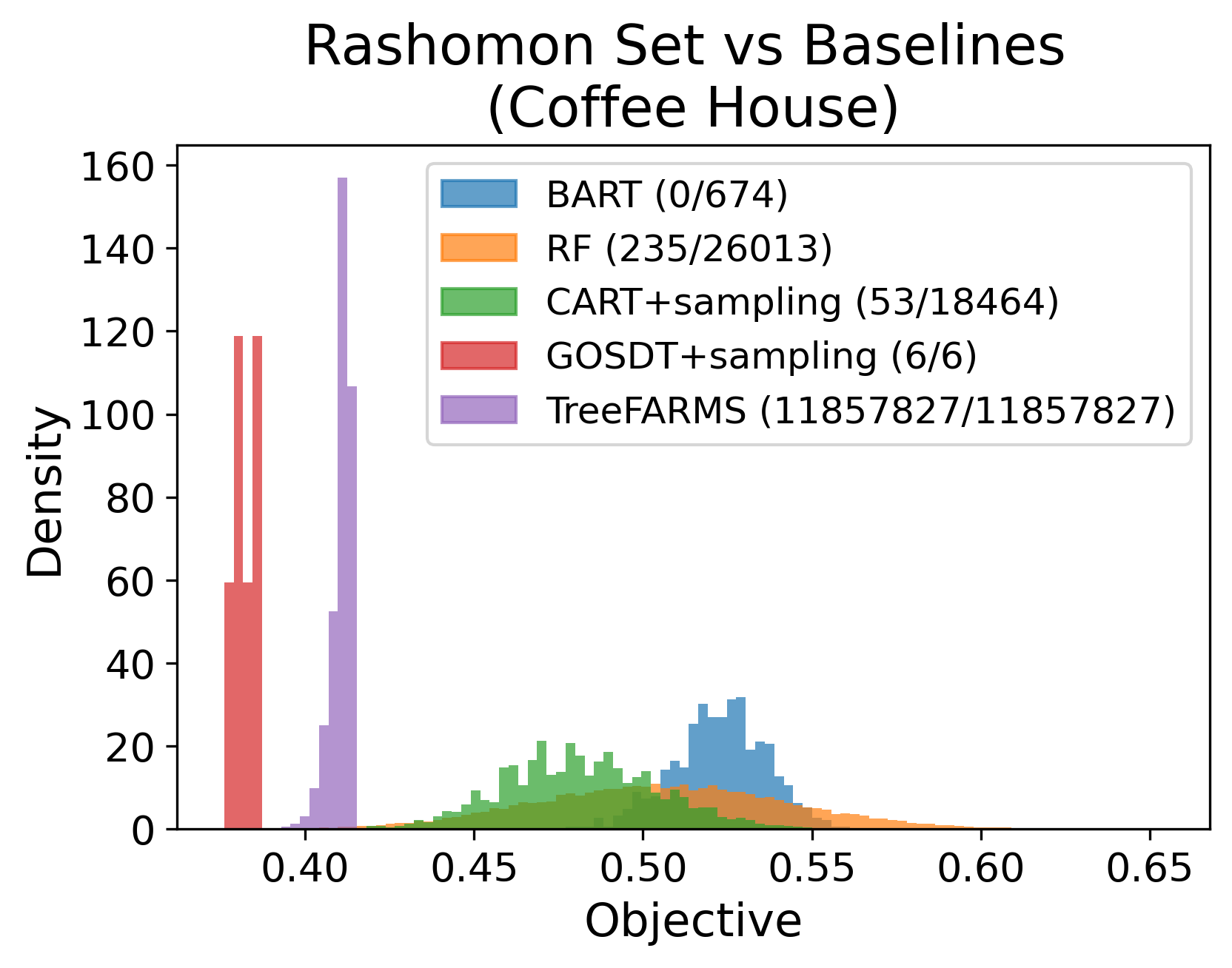}
    \includegraphics[height=0.187\textwidth]{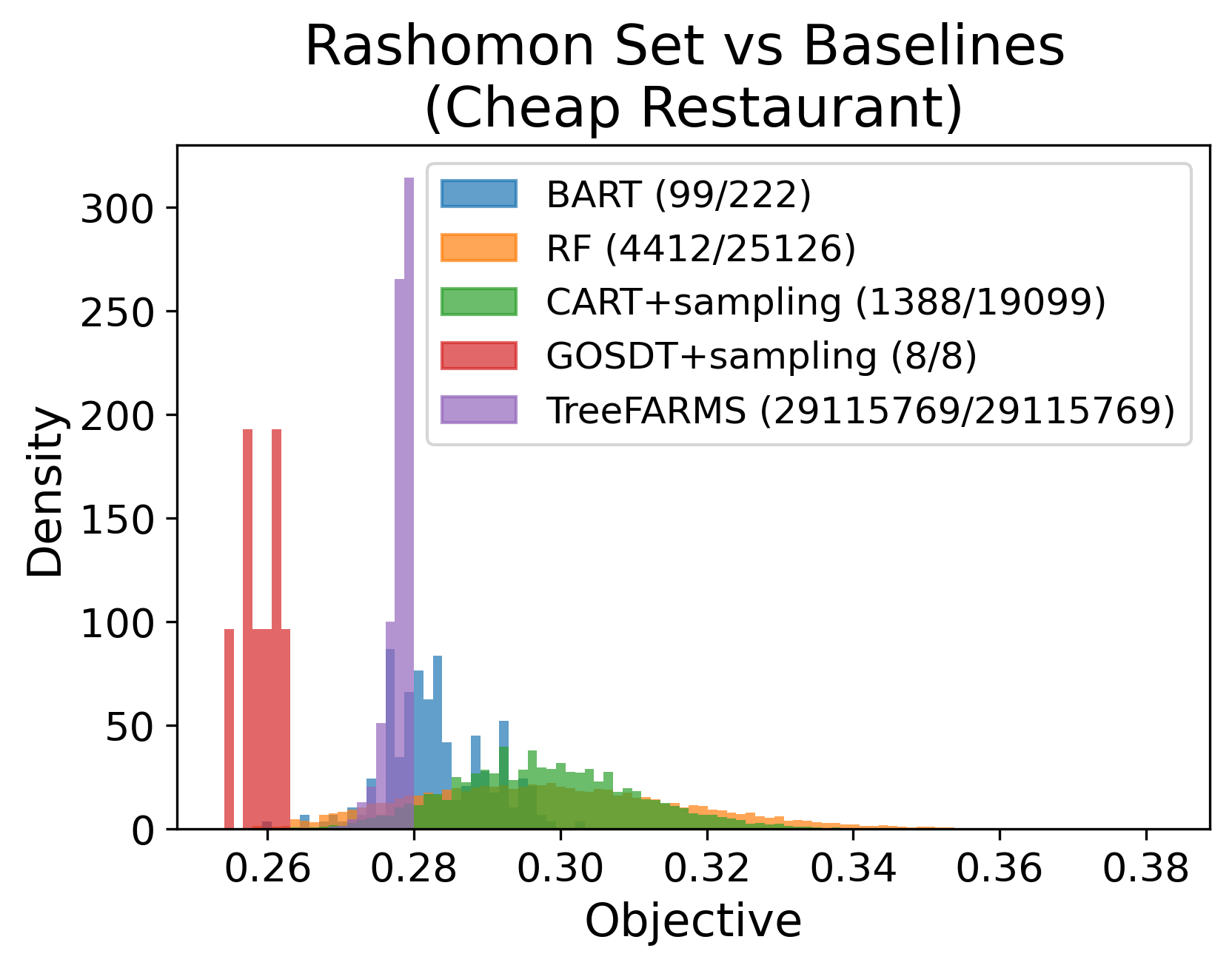}
    \includegraphics[height=0.187\textwidth]{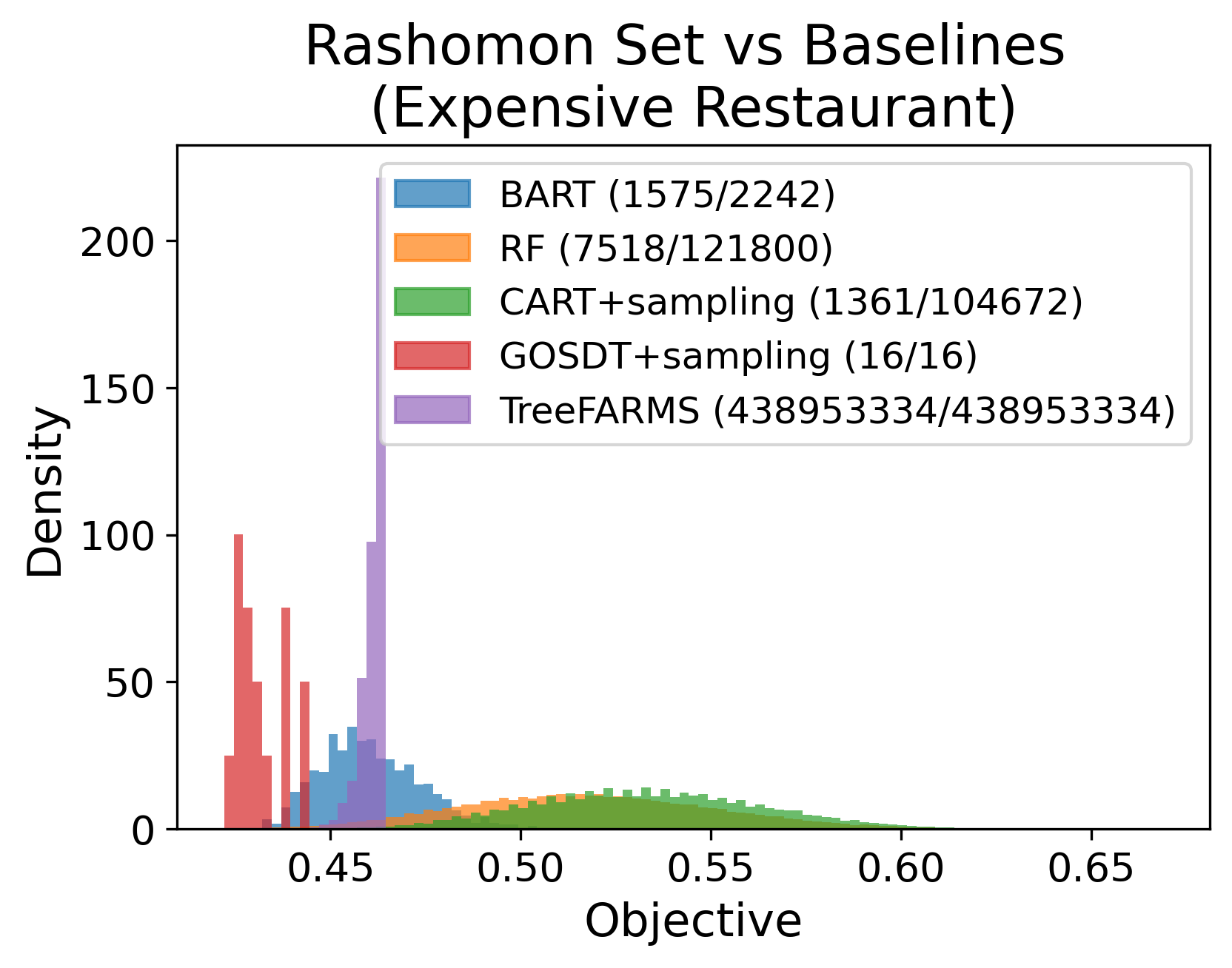}
    \caption{Comparison of trees in the Rashomon set and trees generated by baselines. (A/B) in the legend indicates that A trees of the B trees produced by the baseline method are in the Rashomon set. For example, GOSDT$+$sampling (9/10) means that 9 trees of 10 distinct trees found by GOSDT with sampling are in the Rashomon set.}
    \label{fig:more-vs-baselines}
\end{figure}

\renewcommand\tabularxcolumn[1]{m{#1}}
\newcolumntype{A}{>{\RaggedLeft\arraybackslash}X}
\newcolumntype{M}{>{\Centering\hsize=\dimexpr2\hsize+5\tabcolsep+\arrayrulewidth\relax}X}
\newcolumntype{L}[1]{>{\RaggedRight\let\newline\\\arraybackslash\hspace{0pt}}m{#1}}
\newcolumntype{C}[1]{>{\Centering\let\newline\\\arraybackslash\hspace{0pt}}m{#1}}
\newcolumntype{R}[1]{>{\RaggedLeft\let\newline\\\arraybackslash\hspace{0pt}}m{#1}}
\renewcommand\cellgape{\Gape[2pt]}
\begin{table}[ht]
   \footnotesize
       \centering
       \begin{tabularx}{\textwidth}
       {@{} |L{5.3em}|R{3.3em}|R{2.8em}|*{4}{A|R{2.4em}|} @{}}\hline
           & \multicolumn{2}{c|}{\ourmethod} & \multicolumn{2}{c|}{BART} & \multicolumn{2}{c|}{RF} & \multicolumn{2}{c|}{\makecell{CART + \\ sampling}} & \multicolumn{2}{c|}{\makecell{GOSDT + \\ sampling}}  \\
           \cline{2-11}
           & time (s)
           & \# trees in Rset/s
           & norm. time
           & \# trees in Rset/s
           & norm. time
           & \# trees in Rset/s
           & norm. time
           & \# trees in Rset/s
           & norm. time
           & \# trees in Rset/s\\\hline
 Monk2 & \textcolor{gray}{46.46} & 2.28 $\times 10^{6}$ & \textcolor{gray}{0.99 \tiny{± 0.02}} & 0.03 \tiny{± 0.02} & \textcolor{gray}{1.01 \tiny{± 0.00}} & 0.00 \tiny{± 0.00} & \textcolor{gray}{0.99 \tiny{± 0.02}} & 0.20 \tiny{± 0.11} & \textcolor{gray}{0.95 \tiny{± 0.00}} & 0.31 \tiny{± 0.04}\\\hline
COMPAS & \textcolor{gray}{3.03} & 8.75 $\times 10^{4}$ & \textcolor{gray}{0.98 \tiny{± 0.03}} & 0.14 \tiny{± 0.17} & \textcolor{gray}{0.99 \tiny{± 0.01}} & 68.67 \tiny{± 2.34} & \textcolor{gray}{1.00 \tiny{± 0.00}} & 44.83 \tiny{± 2.17} & \textcolor{gray}{0.96 \tiny{± 0.01}} & 2.20 \tiny{± 0.34}\\\hline
Car Evaluation & \textcolor{gray}{207.76} & 101.27 & \textcolor{gray}{0.95 \tiny{± 0.00}} & 0.00 \tiny{± 0.00} & \textcolor{gray}{1.01 \tiny{± 0.01}} & 0.03 \tiny{± 0.02} & \textcolor{gray}{1.00 \tiny{± 0.01}} & 0.71 \tiny{± 0.04} & \textcolor{gray}{0.97 \tiny{± 0.02}} & 0.03 \tiny{± 0.01}\\\hline
Bar & \textcolor{gray}{266.79} & 1.91 $\times 10^{4}$ & \textcolor{gray}{0.96 \tiny{± 0.00}} & 0.15 \tiny{± 0.29} & \textcolor{gray}{1.00 \tiny{± 0.02}} & 4.78 \tiny{± 0.14} & \textcolor{gray}{1.00 \tiny{± 0.01}} & 3.88 \tiny{± 0.11} & \textcolor{gray}{0.96 \tiny{± 0.01}} & 0.06 \tiny{± 0.01}\\\hline
Coffee House & \textcolor{gray}{81.77} & 1.45 $\times 10^{5}$ & \textcolor{gray}{0.99 \tiny{± 0.01}} & 0.00 \tiny{± 0.00} & \textcolor{gray}{1.00 \tiny{± 0.00}} & 2.60 \tiny{± 0.17} & \textcolor{gray}{1.00 \tiny{± 0.00}} & 0.67 \tiny{± 0.03} & \textcolor{gray}{0.99 \tiny{± 0.03}} & 0.06 \tiny{± 0.01}\\\hline
Expensive Restaurant & \textcolor{gray}{316.40} & 1.39 $\times 10^{6}$ & \textcolor{gray}{0.96 \tiny{± 0.01}} & 3.72 \tiny{± 1.31} & \textcolor{gray}{1.01 \tiny{± 0.00}} & 24.44 \tiny{± 0.76} & \textcolor{gray}{1.01 \tiny{± 0.01}} & 4.51 \tiny{± 0.15} & \textcolor{gray}{0.96 \tiny{± 0.01}} & 0.06 \tiny{± 0.01}\\\hline
Cheap Restaurant & \textcolor{gray}{74.91} & 3.89 $\times 10^{5}$ & \textcolor{gray}{0.99 \tiny{± 0.00}} & 0.98 \tiny{± 0.39} & \textcolor{gray}{1.00 \tiny{± 0.00}} & 59.63 \tiny{± 0.43} & \textcolor{gray}{1.01 \tiny{± 0.01}} & 17.75 \tiny{± 0.55} & \textcolor{gray}{0.96 \tiny{± 0.01}} & 0.15 \tiny{± 0.03}\\\hline
Bar-7 & \textcolor{gray}{13.11} & 1.13 $\times 10^{4}$ & \textcolor{gray}{1.01 \tiny{± 0.01}} & 0.67 \tiny{± 0.83} & \textcolor{gray}{1.00 \tiny{± 0.00}} & 22.78 \tiny{± 1.35} & \textcolor{gray}{1.00 \tiny{± 0.00}} & 8.06 \tiny{± 0.35} & \textcolor{gray}{1.03 \tiny{± 0.08}} & 0.33 \tiny{± 0.07}\\\hline
FICO & \textcolor{gray}{3841.94} & 21.08 & \textcolor{gray}{0.98 \tiny{± 0.03}} & 0.00 \tiny{± 0.00} & \textcolor{gray}{1.02 \tiny{± 0.01}} & 3.17 \tiny{± 0.03} & \textcolor{gray}{1.00 \tiny{± 0.02}} & 0.05 \tiny{± 0.00} & \textcolor{gray}{0.94 \tiny{± 0.00}} & 0.00 \tiny{± 0.00}\\\hline
Congressional Voting Records & \textcolor{gray}{16.67} & 0.06 & \textcolor{gray}{1.00 \tiny{± 0.01}} & 0.00 \tiny{± 0.00} & \textcolor{gray}{1.00 \tiny{± 0.00}} & 0.06 \tiny{± 0.00} & \textcolor{gray}{1.00 \tiny{± 0.00}} & 0.06 \tiny{± 0.00} & \textcolor{gray}{0.96 \tiny{± 0.02}} & 0.06 \tiny{± 0.00}\\\hline
       \end{tabularx}
       \vspace{4pt}
       
    \caption{Runtime and number of trees found for \ourmethod{} and the four baselines. We show the runtime (in seconds) for \ourmethod{} to find the Rashomon set. For the baselines, we show their average runtime and standard deviation normalized to that of \ourmethod{}. For instance, a normalized time of 1.1 means more time than \ourmethod{} by 10\%. We also show the average number of trees in the Rashomon set produced per second and its standard deviation. The Congressional Voting Record dataset has only one tree in the Rashomon set. \emph{The only method guaranteed to find the full Rashomon set is \ourmethod{}.}}
       \label{tab:more_time_baselines}
   \end{table}

\subsection{Landscape of the Rashomon set}\label{sec:rset_visual}
\textbf{Collection and Setup:} We ran this experiment on \textbf{Monk2}. We used dimension reduction techniques to visualize the Rashomon set. 

\begin{figure*}
\centering
\includegraphics[height=0.3\textwidth]{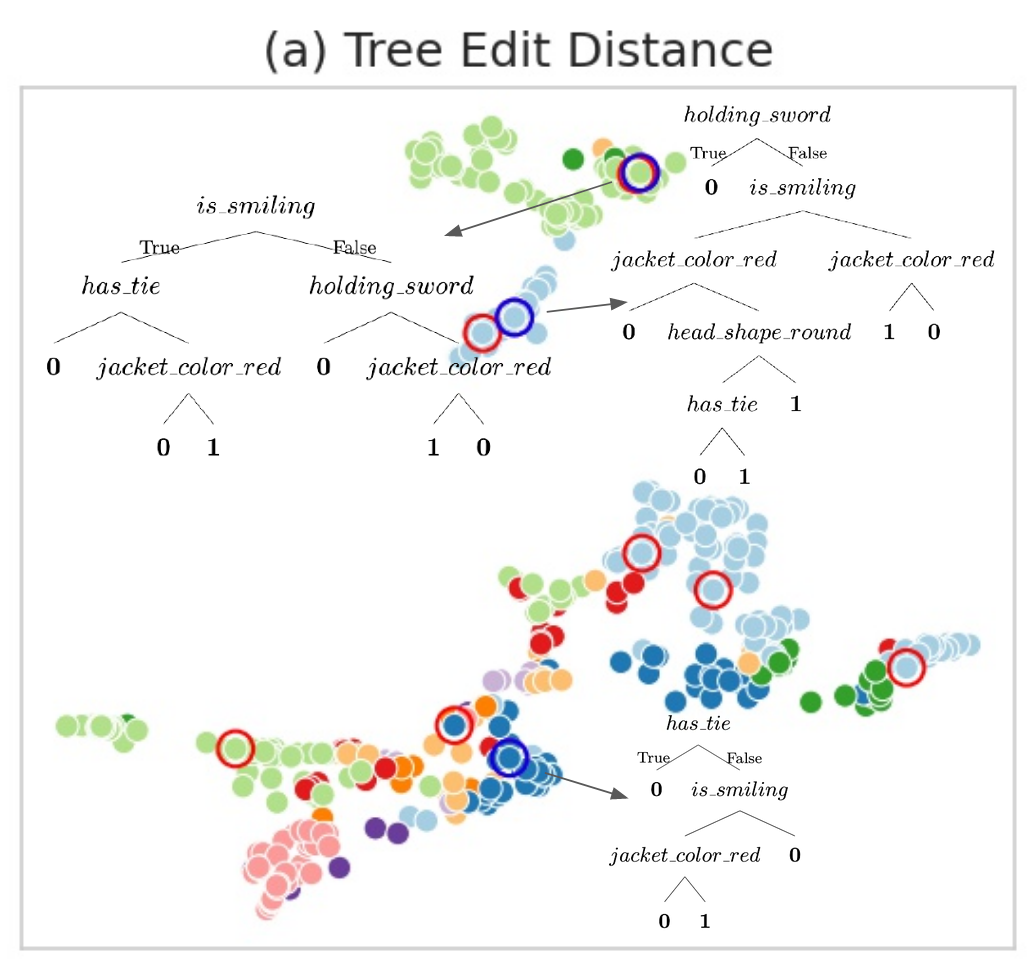}
\includegraphics[height=0.3\textwidth]{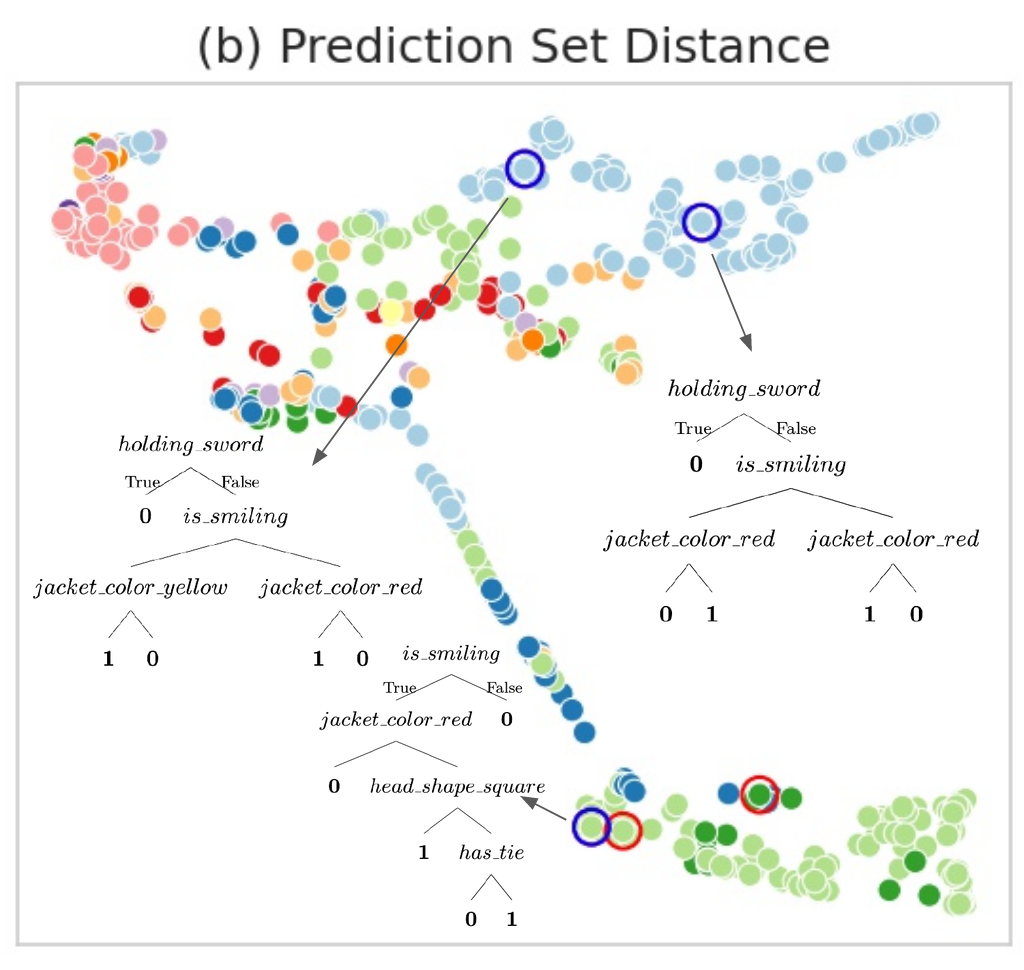}
\includegraphics[height=0.3\textwidth]{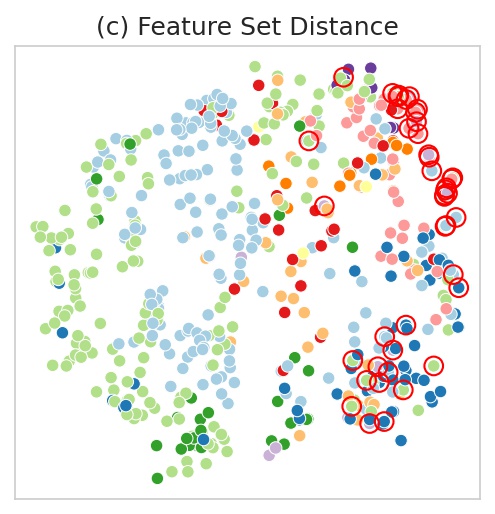}
\includegraphics[width=0.9\textwidth]{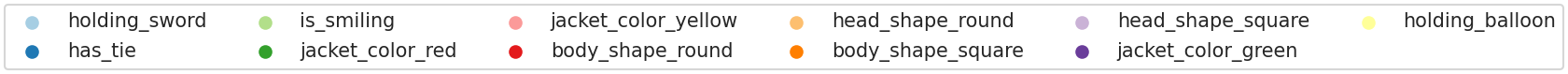}
\caption{Dimension-reduced images of the Rashomon set for Monk2. Colors represent the feature on which we split at the root. Subfigures (a), (b), and (c) show the embedding using tree edit distance, prediction distance, and feature set distance, respectively. In (a), red and blue circles identify the 10 trees with the lowest objective; the blue circles correspond to the trees on the side. The circles in (b) identify the 5 sparsest trees with accuracy above 0.7. The circles in (c) correspond to trees that use feature \textit{head}$\_$\textit{shape}=\textit{square} but not \textit{head}$\_$\textit{shape}=\textit{round} and not \textit{jacket}$\_$\textit{color=red}.} \label{fig:scatter}
\end{figure*}

We first generated the Rashomon set with $\lambda=0.025$ and $\epsilon=0.2$. We then selected the best tree for each unique set of features in the Rashmon set; that is, if two trees used the same features, we kept the tree with the better objective.
Next, we reduced the Rashomon set to $\mathbb{R}^2$ using the PaCMAP dimension reduction technique \nocite{wang2021understanding}\citepAppendix[which handles global structure better than techniques such as t-SNE and UMAP, see][]{appwang2021understanding}. 

\textbf{Results:} Figure \ref{fig:scatter} shows the embedding using three different distance metrics: tree edit distance (a), prediction set distance (b), and feature set distance (c). The \textit{tree edit distance} is the edit distance between two trees using the operations add-node, delete-node, or swap-nodes, all with equal weight. The \textit{prediction set distance} is the Hamming distance between the set of predictions each tree creates for the dataset (up to $n$). 
The \textit{feature set distance} between two trees is the Hamming distance between the feature sets they used (up to $p$). 

The images for the Monk2 dataset, shown in Figure \ref{fig:scatter}, yield interesting results.
When designing decision trees, we usually expect that the root split contains most of the information, and thus we might expect that much of the Rashomon set shares the same root split. All three embeddings in Figure \ref{fig:scatter}, in which color identifies the feature that splits the root node, show that this is not the case. 
Interestingly, Figure \ref{fig:scatter}a exhibits clustering, where each cluster includes models with different root splits, which is not what one might expect. Figure \ref{fig:scatter}c suggests the importance of two prominent features that often appear at the root split, \textit{holding}$\_$\textit{sword} and \textit{is}$\_$\textit{smiling}, while still highlighting the diversity in the root node.

Querying the Rashomon set is now extremely easy.
Here are some examples.

-- \emph{Find the 10 best trees}.
The circles in Figure \ref{fig:scatter}a identify the ten models with the best objective values. Interestingly, these trees appear in different parts of the space and have diverse structure.

-- \emph{Find the five sparsest trees with accuracy above 0.7.} Figure \ref{fig:scatter}b shows the result. These trees are (again) diverse.

-- \emph{Find trees with feature \textit{head}$\_$\textit{shape}=\textit{square} but neither \textit{head}$\_$\textit{shape}=\textit{round} nor \textit{jacket}$\_$\textit{color=red}.}

The red circles in Figure \ref{fig:scatter}c identify these trees. Constraining models in this way produces trees that use features quite different from all the trees on the left, even when they use the same root.

\subsection{Variable Importance: Model Class Reliance}\label{app:exp_mcr}

\textbf{Collection and Set:} We ran this experiments on \textbf{COMPAS, Bar, Coffee House, Expensive Restaurant}. We find the whole Rashomon set of each dataset given $\lambda=0.01$ and $\epsilon=0.05$, and calculate $\textrm{MCR}_-$ and $\textrm{MCR}_+$ of each feature as described in Appendix \ref{app:mcr}. We then study the how sampling can help us estimate the true model class reliance by sampling 1\%, 5\%, and 25\% trees from the Rashomon set and calculating the MCR based on these subsets of the Rashomon set. Figure \ref{fig:more_mcr} shows sampled MCR converges to true MCR. For each sampling proportion, we use 5 different random seeds to compute the average and standard deviation of $\textrm{MCR}_-$ and $\textrm{MCR}_+$. 
We ran this particular experiment on a 2.7Ghz (768GB RAM 48 cores) Intel Xeon Gold 6226 processor. We set a 200-GB memory limit. 

\textbf{Results}: Figure \ref{fig:more_mcr} shows the model class reliance on four different datasets. For the COMPAS dataset (top-left subfigure), features related to prior counts generally have high $\textrm{MCR}_+$, which means these features are very important for some of trees in the Rashomon set. For the Bar dataset (top-right subfigure), features “Bar\_1-3” and “Bar\_4-8” have dominant $\textrm{MCR}_+$ and $\textrm{MCR}_-$ compared with other features, indicating that for all well-performing trees, these features are the most important. Similarly, features "CoffeeHouse\_1~3" and "CoffeeHouse\_4~8" have dominant $\textrm{MCR}_+$ and $\textrm{MCR}_-$. This makes sense, since people who go to bar or coffee house regularly would be likely to accept a coupon for a bar or a coffee house.

Sampling can help us calculate MCR more efficiently by considering only part of the whole Rashomon set. Figure \ref{fig:more_mcr} shows that sampled MCR, in general, converges to true MCR even if only 1\% or 5\% of trees are sampled.  

\begin{figure}
    \centering
    \includegraphics[width=1\linewidth]{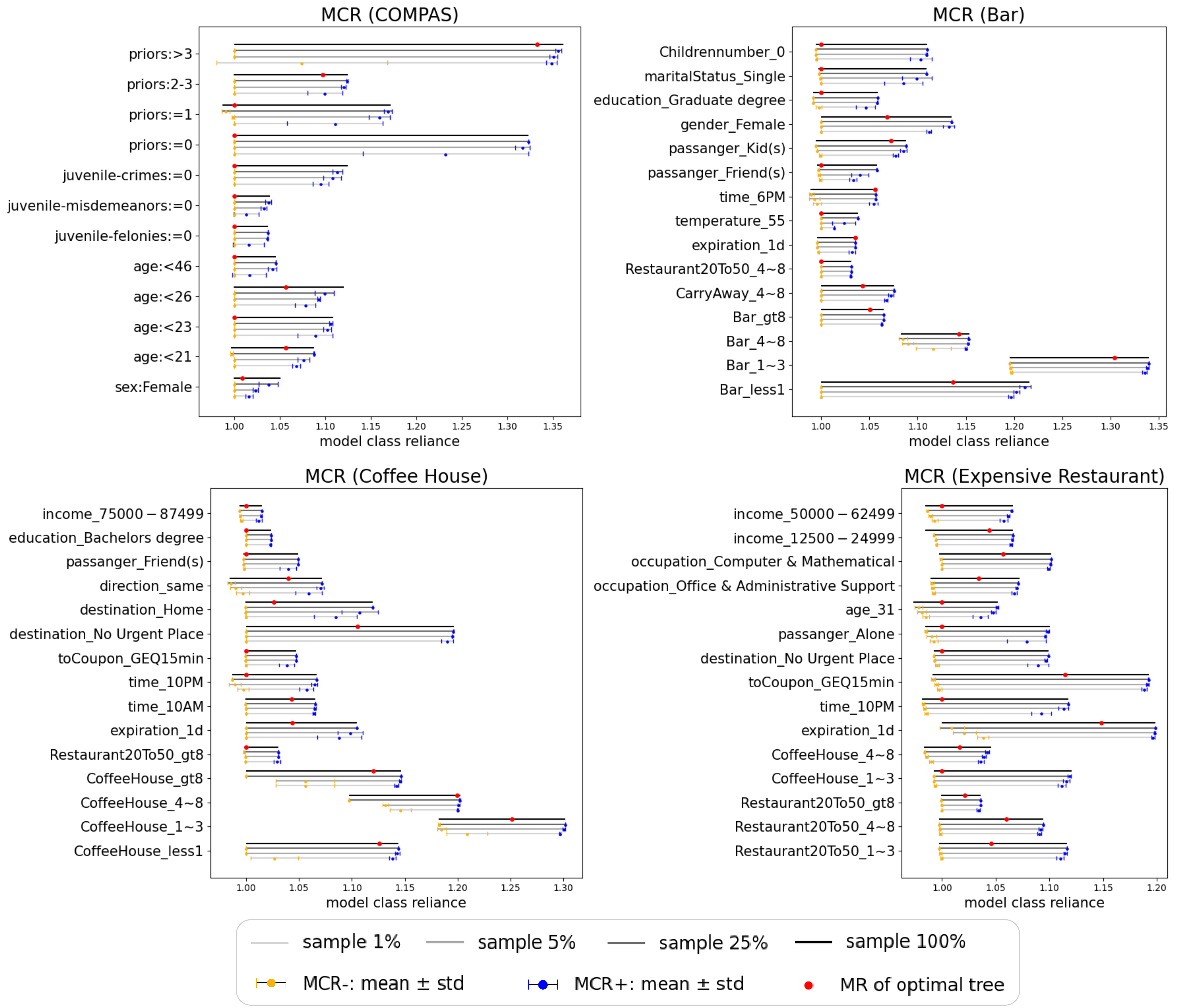}
    \caption{Variable Importance: Model class reliance on the COMPAS, Bar, Coffee House, and Expensive Restaurant ($\lambda$ = 0.01, $\epsilon$ = 0.05). Four different line segments colored in different gray levels are model class reliance calculated by sampling different proportions from the whole Rashomon set. Yellow/blue dots with bars indicate $\textrm{MCR}_-$/$\textrm{MCR}_+$ within one standard deviation of the mean. ``Sample 100\%'' means using the whole Rashomon set and there is no variation. 
    Red dots indicate the model reliance (variable importance) calculated by the optimal tree. }
    \label{fig:more_mcr}
\end{figure}

\subsection{Balanced Accuracy and F1-score Rashomon set from Accuracy Rashomon set}\label{app:exp_other_metrics}

\textbf{Collection and Set:} We ran this experiments on three imbalanced datasets \textbf{Breast Cancer, Cheap Restaurant, Takeaway Food}. We ran this particular experiment on a 2.7Ghz (768GB RAM 48 cores) Intel Xeon Gold 6226 processor. We set a 200-GB memory limit. 

\textbf{Results:} Figure \ref{fig:more_bacc_results}, \ref{fig:more_f1_results} show trees in the Accuracy Rashomon set which covers the Balanced Accuracy Rashomon set and F1-score Rashomon set respectively. The black dashed line indicates the corresponding objective thresholds and blue dots below the dashed line are trees within these Rashomon sets. In some cases, the tree with the minimum misclassification objective is also the tree with the minimum other metric objective (see Figure \ref{fig:more_bacc_results} left and Figure \ref{fig:more_f1_results} right). But this is not always guaranteed. For example, in the right subfigure of Figure \ref{fig:more_bacc_results}, a single split node has the optimal accuracy objective (in yellow), while another three-leaf tree minimizes the balanced accuracy objective (in green). Actually, many trees have better balanced accuracy objective than the tree that minimizes the accuracy objective. A similar pattern holds for the F1-score Rashomon set (see left subfigure Figure \ref{fig:more_f1_results}). 

\begin{figure}
    \centering
    \includegraphics[scale=0.33]{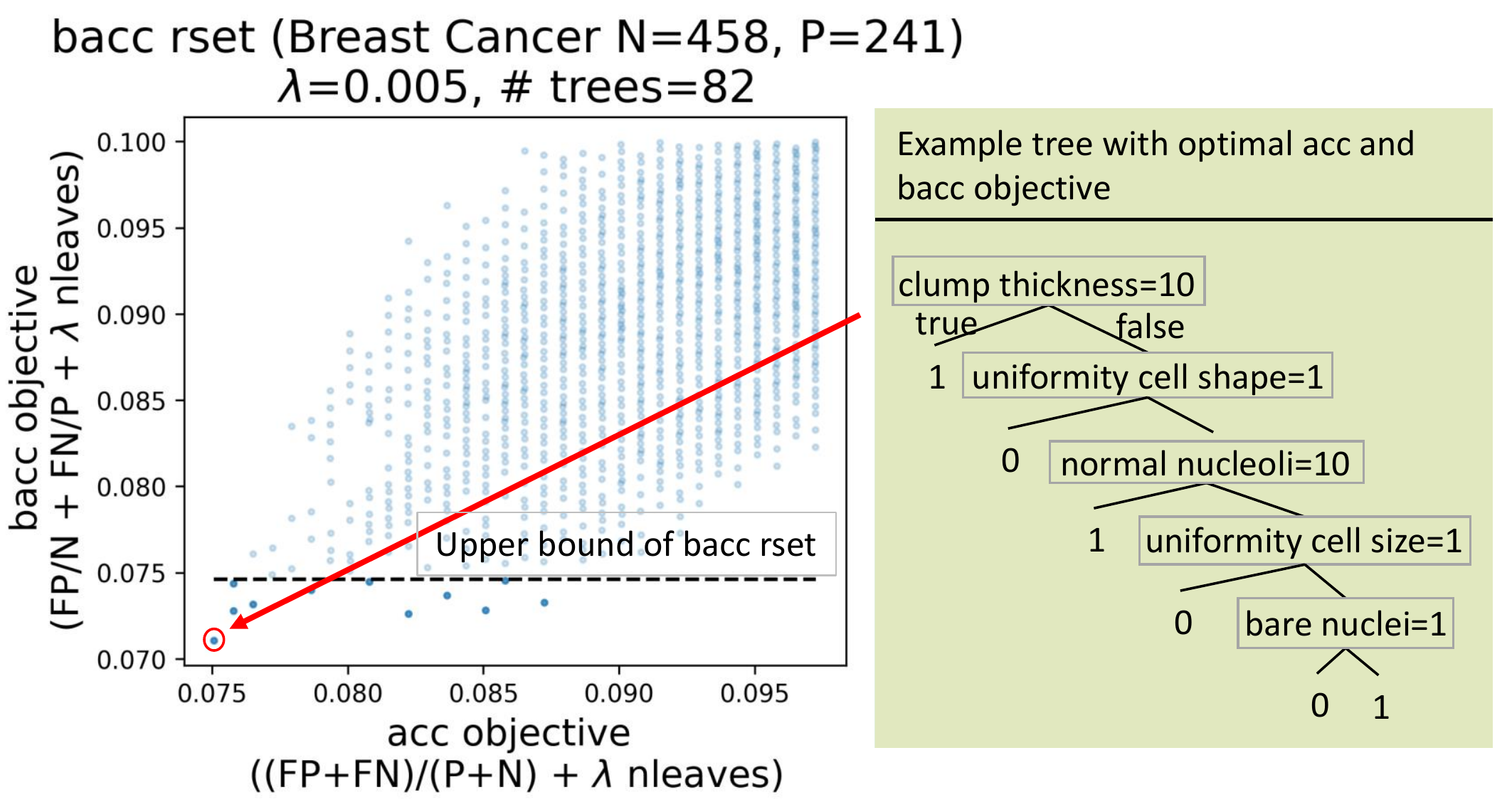}
    \includegraphics[scale=0.33]{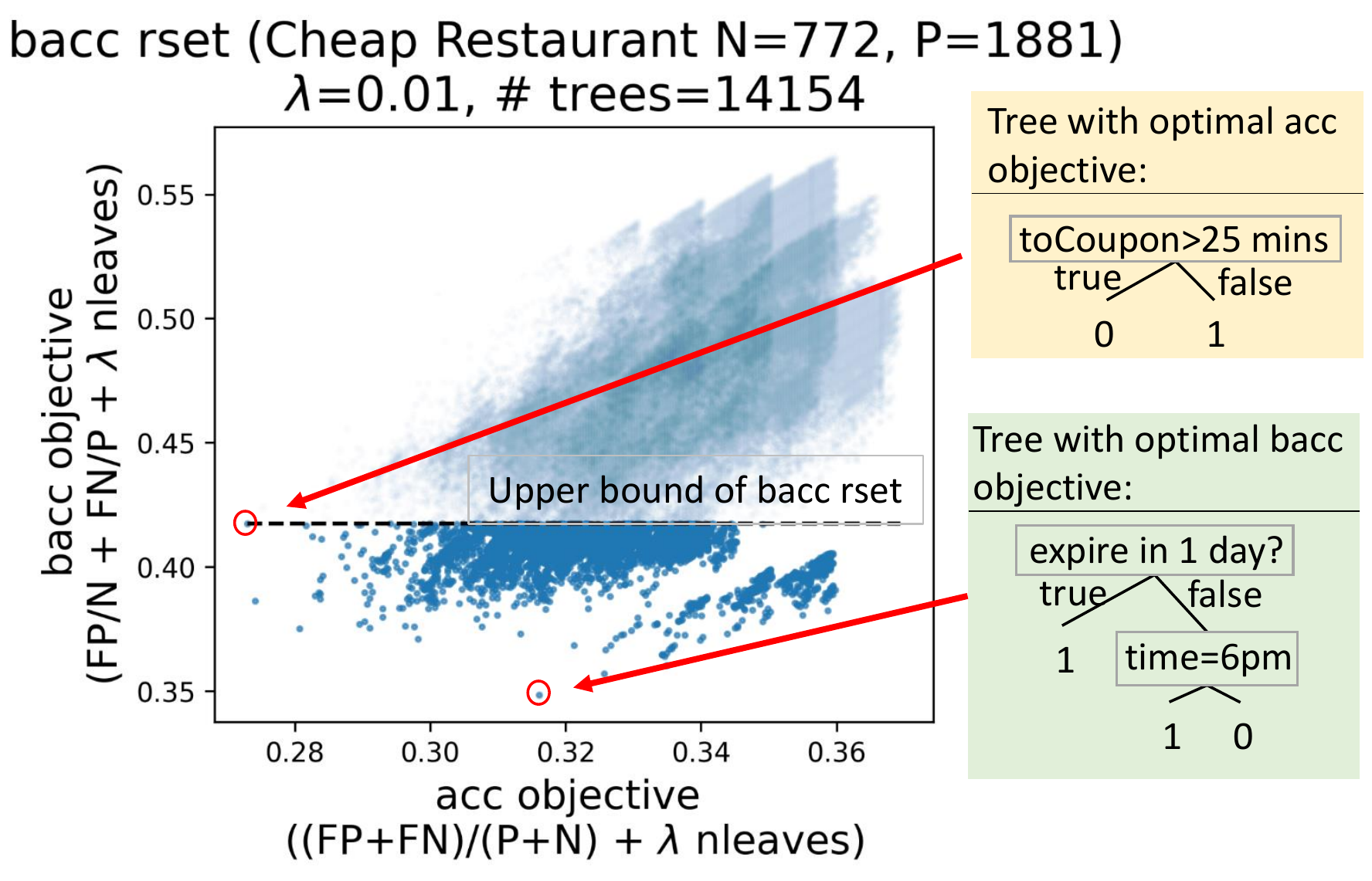}
    \caption{Example of Balanced Accuracy Rashomon sets. \# trees indicates the number of trees within the Balanced Accuracy Rashomon set. Trees in the yellow region have optimal accuracy objective and trees in the green region have optimal balanced accuracy. (Breast Cancer: $\lambda=0.005$, Cheap Restaurant: $\lambda=0.01$)}
    \label{fig:more_bacc_results}
\end{figure}

\begin{figure}
    \centering
    \includegraphics[scale=0.32]{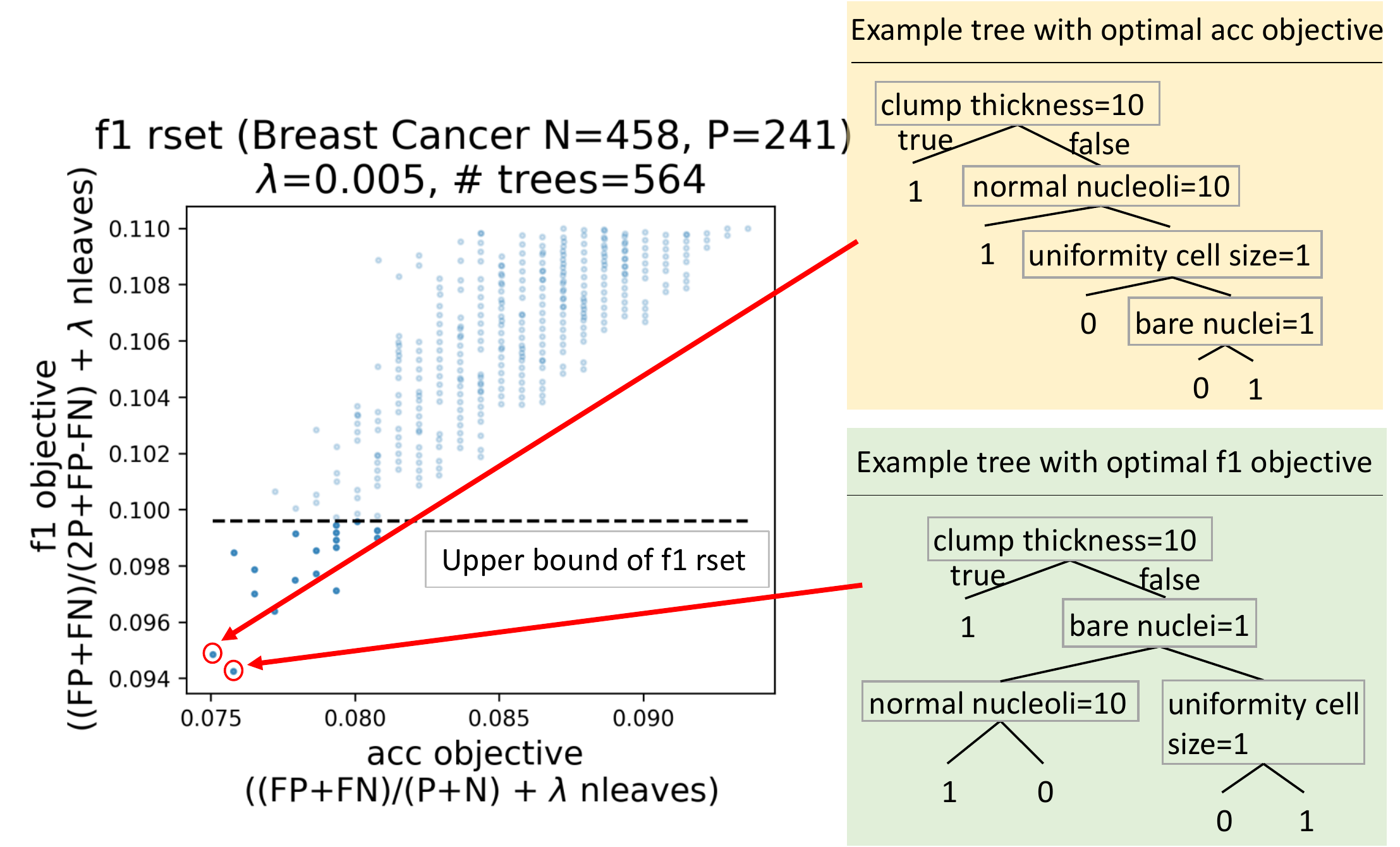}
    \includegraphics[scale=0.32]{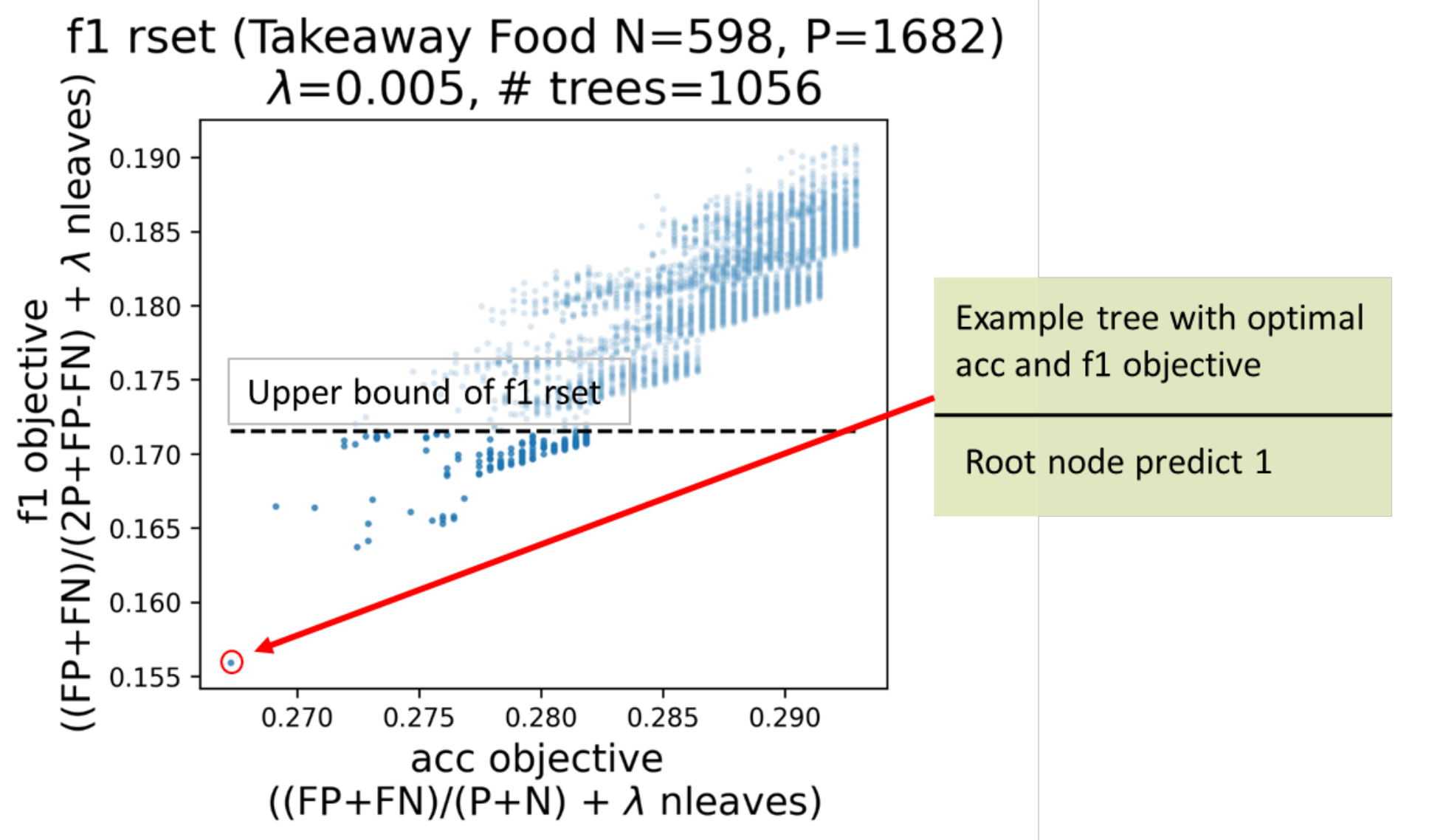}
    \caption{Example of F1-Score Rashomon set. \# trees indicates the number of trees within the F1-score Rashomon set. Trees in the yellow region have optimal accuracy objective and trees in the green region have optimal F1-score objective.(Breast Cancer: $\lambda=0.005$, Takeaway Food: $\lambda=0.005$)}
    \label{fig:more_f1_results}
\end{figure}

\subsection{Rashomon set after removing a group of samples}\label{app:exp_sensitivity}

\textbf{Collection and Set:} We ran this experiments on \textbf{Monk2, COMPAS, Bar, Expensive Restaurant}. We ran this particular experiment on a 2.7Ghz (768GB RAM 48 cores) Intel Xeon Gold 6226 processor. We set a 200-GB memory limit. 

\textbf{Results:} Figure \ref{fig:more_sensitivity_results} show accuracy objective on the full dataset versus objective on the reduced dataset after 1\% of different samples are removed on the Bar and Monk2 datasets. The black dashed line indicates the objective threshold of the reduced Rashomon set and blue dots below the dashed line are trees within the reduced Rashomon set. Similar to Figure \ref{fig:reduced} all scatter plots show a high correlation between the accuracy objective on the full dataset and the reduced dataset, indicating sparse near-optimal trees are robust to the shift in sample distribution. Optimal trees on the reduced dataset might be different, as we see by comparing the trees in the orange region and blue region. 

\begin{figure}
    \centering
    \begin{subfigure}[b]{0.98\linewidth}
        \includegraphics[scale=0.31]{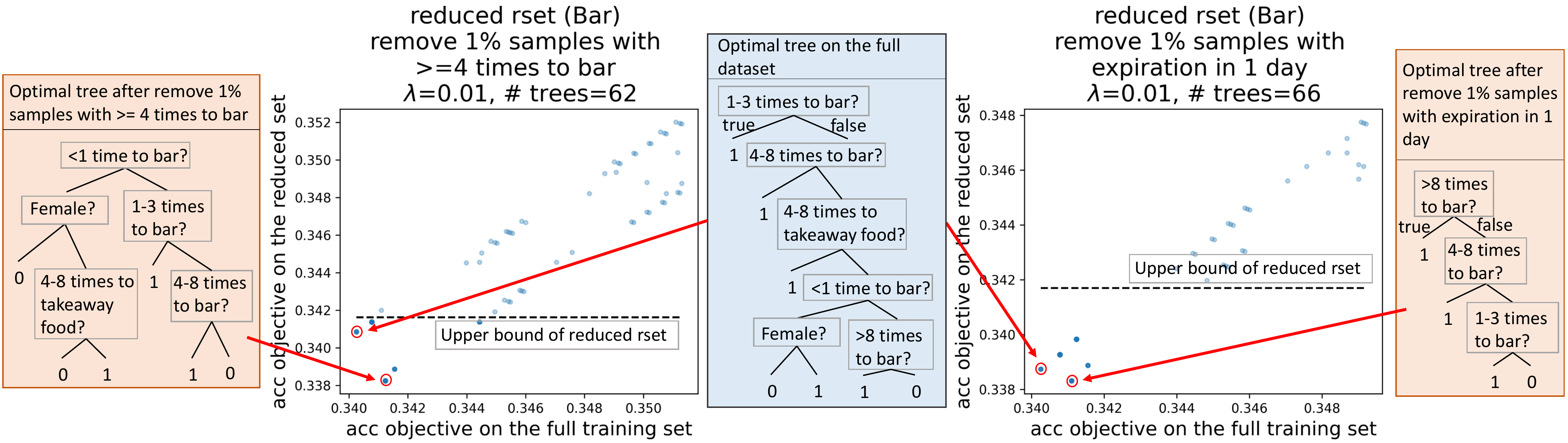}
        \caption{Example Rashomon sets and optimal trees after we remove the 1\% of samples with “number of times to bar $\geq 4$” (left) and “expiration of coupon in 1 day” (right) on the Bar dataset.}
    \end{subfigure}
    \begin{subfigure}[b]{0.98\linewidth}
        \includegraphics[scale=0.36]{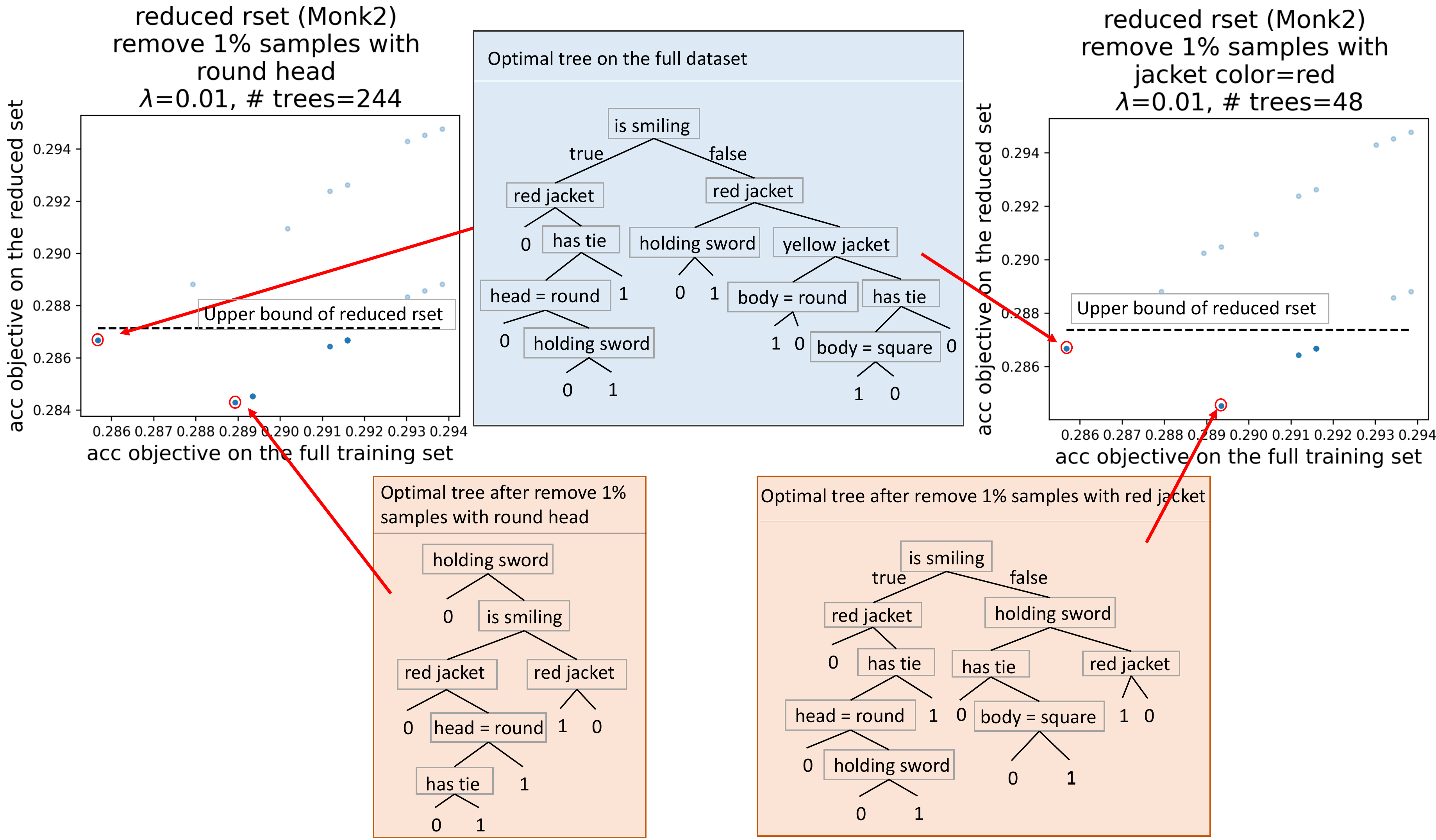}
        \caption{Example Rashomon sets and optimal trees after we remove the 1\% of samples with “round head” (left) and “red jacket" (right) on the Monk2 dataset.}
    \end{subfigure}
    
    \caption{Example Rashomon sets and optimal trees after 1\% of samples are removed. The optimal tree on the full dataset is shown in the blue region and optimal trees on the corresponding reduced datasets are in the orange region.}
    \label{fig:more_sensitivity_results}
\end{figure}

\clearpage
\bibliographystyleAppendix{unsrt}
\bibliographyAppendix{appref}

\end{document}